\documentclass[10pt]{article}
\usepackage[margin=1in]{geometry}
\usepackage[round,compress]{natbib}
\usepackage{parskip}

\usepackage[utf8]{inputenc} 
\usepackage[T1]{fontenc}    

\usepackage{subcaption}

\usepackage{amsmath}
\usepackage{amsfonts,amscd,amssymb,bm,bbm,mathrsfs}
\usepackage{nicefrac}       
\usepackage{amsthm}
\usepackage{mathtools}

\newcommand\labelAndRemember[2]
  {\expandafter\gdef\csname labeled:#1\endcsname{#2}%
   \label{#1}#2}
\newcommand\recallLabel[1]
   {\csname labeled:#1\endcsname\tag{\ref{#1}}}
\newcommand\labelr[2]
  {\expandafter\gdef\csname labeled:#1\endcsname{#2}%
   \label{#1}#2}
\newcommand\recall[1]
   {\csname labeled:#1\endcsname}

\usepackage[algo2e,linesnumbered,vlined,ruled]{algorithm2e}
\usepackage{algorithm}
\usepackage[noend]{algorithmic}

\usepackage{url}
\usepackage{hyperref}
\hypersetup{
  colorlinks,
  breaklinks=true,
  citecolor=blue!70!black,
  linkcolor=blue!70!black,
}
\usepackage[capitalise]{cleveref}
\usepackage[titletoc,title]{appendix}
\usepackage{cancel}
\usepackage{enumerate}
\usepackage[shortlabels]{enumitem}
\usepackage{comment}
\crefformat{equation}{#2(#1)#3}
\Crefformat{equation}{#2(#1)#3}
\crefformat{figure}{Figure #2#1#3}
\Crefformat{figure}{Figure #2#1#3}

\usepackage{booktabs}       
\usepackage{multirow}
\usepackage{multicol}
\usepackage{makecell}
\usepackage{array}
\newcolumntype{H}{>{\setbox0=\hbox\bgroup}c<{\egroup}@{}}
\newcolumntype{Z}{>{\setbox0=\hbox\bgroup}c<{\egroup}@{\hspace*{-\tabcolsep}}}

\usepackage[normalem]{ulem}
\usepackage{exscale}
\usepackage{tikz}
\usepackage{float}
\usepackage{graphicx,graphics}
\graphicspath{ {./fig/} }
\allowdisplaybreaks

\newtheorem{theorem}{Theorem}
\newtheorem{claim}{Claim}
\newtheorem*{theorem*}{Theorem}
\newtheorem{lemma}[theorem]{Lemma}

\newtheorem*{remark*}{Remark}
\newtheorem*{lemma*}{Lemma}

\newtheorem{proposition}[theorem]{Proposition}

\newtheorem{assumption}{Assumption}
\renewcommand{\theassumption}{\Alph{assumption}}

\newenvironment{proof-sketch}{\noindent{\bf Proof Sketch}
  \hspace*{1em}}{\qed\bigskip\\}
\newenvironment{proof-idea}{\noindent{\bf Proof Idea}
  \hspace*{1em}}{\qed\bigskip\\}
\newenvironment{proof-of}[1][{}]{\noindent{\bf Proof of \cref{#1}}
  \hspace*{1em}}{\qed\bigskip\\}
\newenvironment{proof-of-lemma}[1][{}]{\noindent{\bf Proof of Lemma {#1}}
  \hspace*{1em}}{\qed\bigskip\\}
\newenvironment{proof-of-proposition}[1][{}]{\noindent{\bf
    Proof of Proposition {#1}}
  \hspace*{1em}}{\qed\bigskip\\}
\newenvironment{proof-of-theorem}[1][{}]{\noindent{\bf Proof of Theorem {#1}}
  \hspace*{1em}}{\qed\bigskip\\}
\newenvironment{inner-proof}{\noindent{\bf Proof}\hspace{1em}}{
  $\bigtriangledown$\medskip\\}
\newenvironment{proof-attempt}{\noindent{\bf Proof Attempt}
  \hspace*{1em}}{\qed\bigskip\\}

\renewcommand{\bar}{\overline}
\renewcommand{\epsilon}{\varepsilon}

\usepackage{pifont}

\newcommand{\eps}{\varepsilon}

\newcounter{cnt}
\foreach \num in {1,2,...,26}{%
  \stepcounter{cnt}%
  \expandafter\xdef \csname c\Alph{cnt}\endcsname {\noexpand\mathcal{\Alph{cnt}}}%
  \expandafter\xdef \csname b\Alph{cnt}\endcsname {\noexpand\mathbb{\Alph{cnt}}}%
}

\newcommand{\diag}{\operatorname{diag}}

\newcommand{\abs}[1]{\left|#1\right|}

\DeclarePairedDelimiterX{\ddiv}[2]{(}{)}{%
  #1\;\delimsize\|\;#2%
}

\newcommand{\wt}{\widetilde}
\newcommand{\indic}[1]{1\{#1\}}

\renewcommand{\bQ}{\mathbf{Q}}
\renewcommand{\bP}{\mathbf{P}}

\newcommand{\relu}{\mathrm{ReLU}}

\newenvironment{talign}
 {\align}
 {\endalign}
\newenvironment{talign*}
 {\csname align*\endcsname}
 {\endalign}
 
\newcommand{\Attn}{{\rm Attn}}

\newcommand{\TF}{{\rm TF}}

\newcommand{\loss}{\mathsf{loss}}

\newcommand{\val}{{\tt Val}}
\newcommand{\res}{{\tt Res}}

\newcommand{\mlp}{{\tt mlp}}
\newcommand{\attn}{{\tt attn}}

\newcommand{\embd}{{\tt ebd}}
\newcommand{\query}{{\tt Qry}}
\newcommand{\key}{{\tt Key}}
\newcommand{\vall}{{\tt Val}}
\newcommand{\transit}{{\sf P}}

\mathchardef\mhyphen="2D

\newcommand{\softmax}{{\sf SoftMax}}
\newcommand{\mask}{{\sf mask}}
\def\period{{{$\langle${$\mathtt{period}$}$\rangle$}}}

\newcommand{\attnsink}{attention sink}
\newcommand{\valuedrain}{value-state drain}

\newcommand{\activedormant}{active-dormant mechanism}

\newcommand{\vocab}{\mathcal{V}}
\newcommand{\vocabsize}{V}
\newcommand{\tok}{v}
\newcommand{\toki}{i}

\newcommand{\tokk}{k}
\newcommand{\tokj}{j}

\newcommand{\sink}{\alpha}
\newcommand{\vecvalue}{\boldsymbol{\beta}}
\newcommand{\ivalue}{\beta}
\newcommand{\transition}{p}
\newcommand{\Transition}{\bP}
\newcommand{\ppred}{l}
\newcommand{\Ppred}{\mathbf{L}}

\newcommand{\stable}{\pi}

\newcommand{\vecsink}{\bm{\sink}}
\newcommand{\meanvalue}{\bar{B}}
\newcommand{\inv}{{-1}}
\newcommand{\gppred}{\bG}
\newcommand{\Var}{\operatorname{Var}}
\newcommand{\bb}{Bigram-Backcopy}

\newcommand{\lnm}{\text{LayerNorm}}

\newcommand{\mass}{M}

\newcommand\blfootnote[1]{
    \begingroup
    \renewcommand\thefootnote{}\footnote{#1}
    \addtocounter{footnote}{-1}
    \endgroup
}

\newcommand{\paren}[1]{{\left( #1 \right)}}

\newcommand{\set}[1]{{\left\{ #1 \right\}}}
\newcommand{\sets}[1]{{\{ #1 \}}}

\newcommand{\defeq}{\mathrel{\mathop:}=}

\newcommand{\E}{\mathbb{E}}

\renewcommand{\P}{\mathbb{P}}

\newcommand{\Cov}{\mathrm{Cov}}

\newcommand{\R}{\mathbb{R}}

\def\LN{{\rm LN}}

\def\bB{{\mathbf B}}

\def\bE{{\mathbf E}}
\def\bG{{\mathbf G}}
\def\bH{{\mathbf H}}

\def\bK{{\mathbf K}}

\def\bO{{\mathbf O}}
\def\bQ{{\mathbf Q}}

\def\bV{{\mathbf V}}
\def\bW{{\mathbf W}}

\def\bbeta{{\boldsymbol \beta}}

\def\bh{{\mathbf h}}

\def\bz{{\mathbf z}}

\newcommand{\bos}{\texttt{$\langle \texttt{s}\rangle$}}

\title{Active-Dormant Attention Heads: 
Mechanistically Demystifying Extreme-Token Phenomena in LLMs}

\def\shownotes{1}  
\ifnum\shownotes=1
\newcommand{\authnote}[2]{{\scriptsize $\ll$\textsf{#1 notes: #2}$\gg$}}
\else
\newcommand{\authnote}[2]{}
\fi

\begin{document}

\author{
Tianyu Guo\thanks{UC Berkeley. Email: \texttt{\{tianyu\_guo, druvpai, jiantao, michael\_jordan, songmei\}@berkeley.edu}.} \and Druv Pai\footnotemark[1] \and Yu Bai\thanks{Work done at Salesforce AI Research. Email: \texttt{yubai.pku@gmail.com}.} \and Jiantao Jiao\footnotemark[1] \and  Michael I. Jordan\footnotemark[1] \\ \vspace{-1em} \and Song Mei\footnotemark[1]}

\maketitle

\begin{abstract}
Practitioners have consistently observed three puzzling phenomena in transformer-based large language models (LLMs): \textit{attention sinks}, \textit{value-state drains}, and \textit{residual-state peaks}, collectively referred to as \textit{extreme-token phenomena}. These phenomena are characterized by certain so-called ``sink tokens'' receiving disproportionately high attention weights, exhibiting significantly smaller value states, and having much larger residual-state norms than those of other tokens. These extreme tokens give rise to various challenges in LLM inference, quantization, and interpretability.

We elucidate the mechanisms behind extreme-token phenomena. First, we show that these phenomena arise in very simple architectures---transformers with one to three layers---trained on a toy model, the Bigram-Backcopy (BB) task. In this setting, we identify an \textit{active-dormant mechanism}, where attention heads become sinks for specific input domains while remaining non-sinks for others. Our theoretical analysis of the training dynamics reveals that these phenomena are driven by a \textit{mutual reinforcement mechanism}. Building on these insights, we propose strategies to mitigate extreme-token phenomena during pretraining, including replacing softmax with ReLU and Adam with SGD. Next, we extend our analysis to pretrained LLMs, including Llama and OLMo, showing that many attention heads exhibit a similar \textit{active-dormant mechanism} as in the BB task, and that the \textit{mutual reinforcement mechanism} also governs the emergence of extreme-token phenomena during LLM pretraining. Our results reveal that many of the static and dynamic properties of extreme-token phenomena predicted by the BB task align with observations in pretrained LLMs. 

\blfootnote{Code for our experiments is available at~\url{https://github.com/GuoTianYu2000/Active-Dormant-Attention}}.
\end{abstract}

\section{Introduction}

Recent analyses of transformer-based open-source large language models (LLMs), such as GPT-2 \citep{radford2019language}, Llama-2 \citep{touvron2023llama}, Llama-3 \citep{dubey2024llama}, Mixtral \citep{jiang2023mistral}, and Pythia \citep{biderman2023pythia}, have revealed three intriguing phenomena: 
\begin{itemize}[leftmargin=2em]
\setlength\itemsep{-0.3em}
\item[-] \textbf{Attention sinks} \citep{xiao2023efficient}: In many attention heads, the initial token consistently attracts a large portion of the attention weights. Other special tokens, such as the delimiter token, can also draw significant attention weights. These tokens are collectively referred to as \emph{sink tokens}. 
\item[-] \textbf{Value-state drains} \citep{guo2024attention}: For the attention heads that exhibit attention sinks, the value states of sink tokens are consistently much smaller than those of other tokens. 
\item[-] \textbf{Residual-state peaks} \citep{sun2024massive}: The residual states of sink tokens, excluding those from the first and last layers, exhibit significantly larger norms compared to other tokens. 
\end{itemize}
These phenomena often appear together and consistently occur in various pretrained LLMs, which we collectively refer to as the \emph{extreme-token phenomena}. Figure \ref{figure:extreme-token} illustrates these phenomena in Llama-3.1-8B-Base, using a fixed prompt sentence: “\bos Summer is warm\period~Winter is cold\period”. Here, the first token, \bos~(the Beginning-of-Sequence token), serves as the sink token. As shown in the figure, the sink token receives disproportionately high attention weights, exhibits significantly smaller value states, and has much larger residual state norms compared to other tokens. It is important to note that the first token does not have to be \bos~to act as a sink token; other tokens appearing first in the sequence can also serve this role. Additionally, in models such as Llama-2, a delimiter token can also function as the sink token. 

\begin{figure}[t]
  \centering
  \begin{subfigure}[t]{0.32\textwidth}
      \centering 
      \caption{\small Attention weights at L24}
      \includegraphics[width=0.9\textwidth]{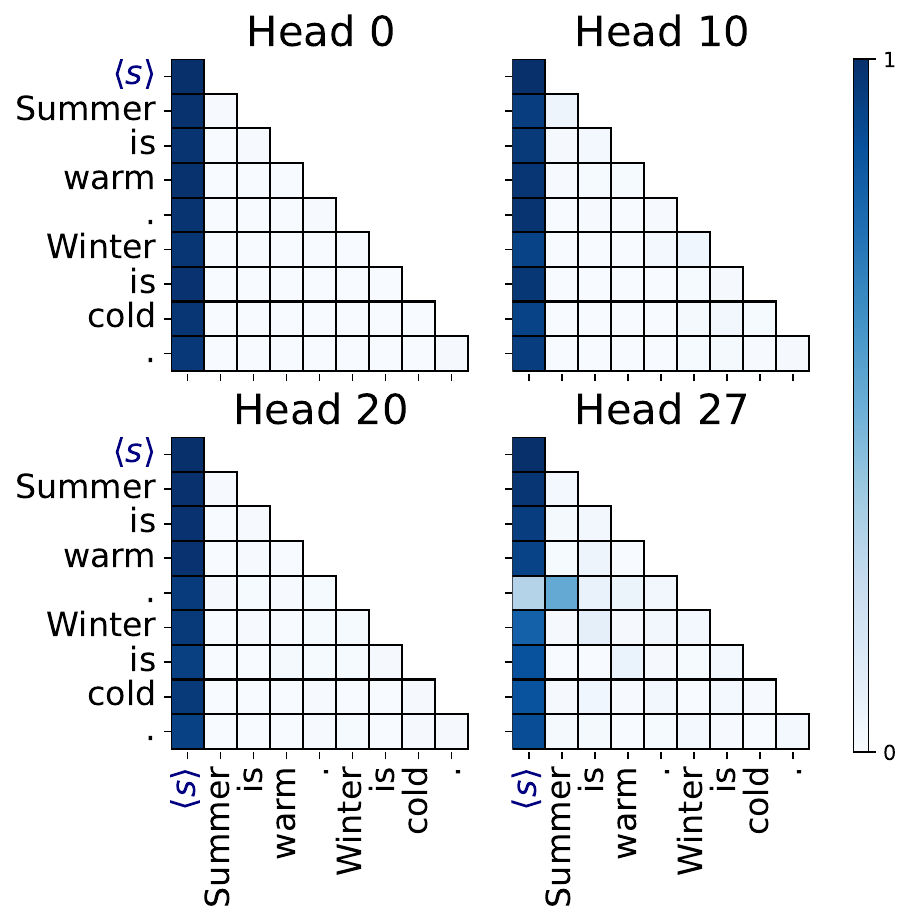}
      \label{fig:attention_sinks_wide_random}
  \end{subfigure}
  \hfill
  \begin{subfigure}[t]{0.3\textwidth}
      \centering 
      \caption{\small Norms of value states}
      \includegraphics[width=0.9\textwidth]{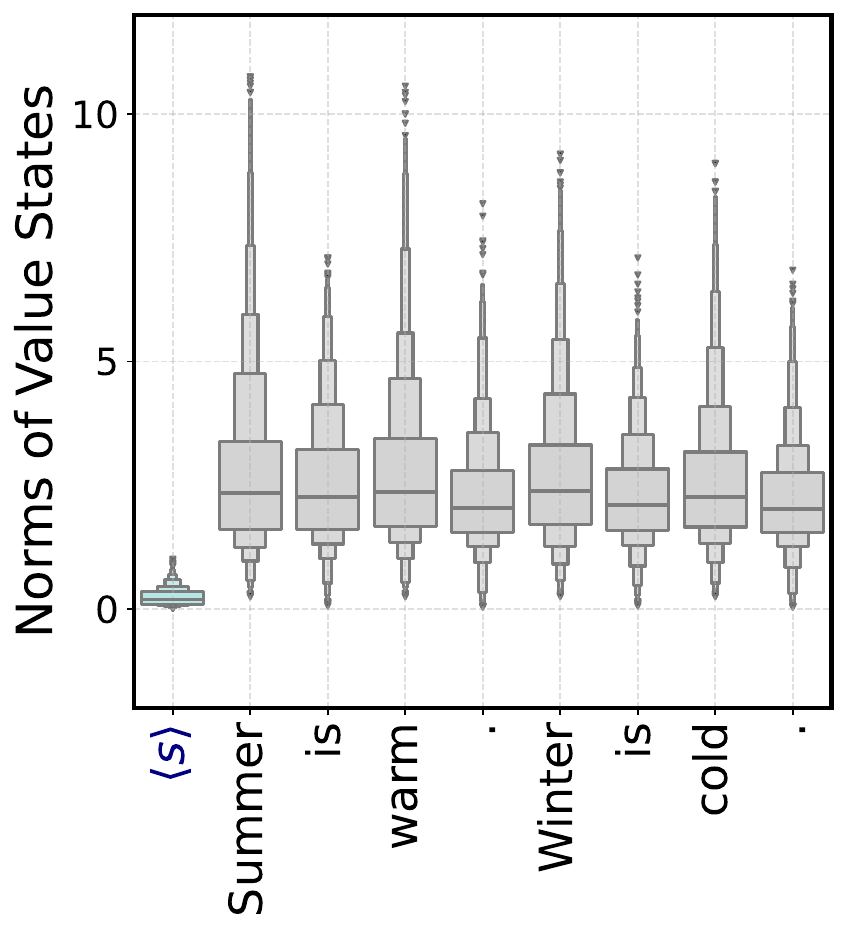}
      \label{fig:value_norms_zeroed}
  \end{subfigure}
  \hfill
  \begin{subfigure}[t]{0.3\textwidth}
      \centering 
      \caption{\small Norms of residual states}
      \includegraphics[width=0.9\textwidth]{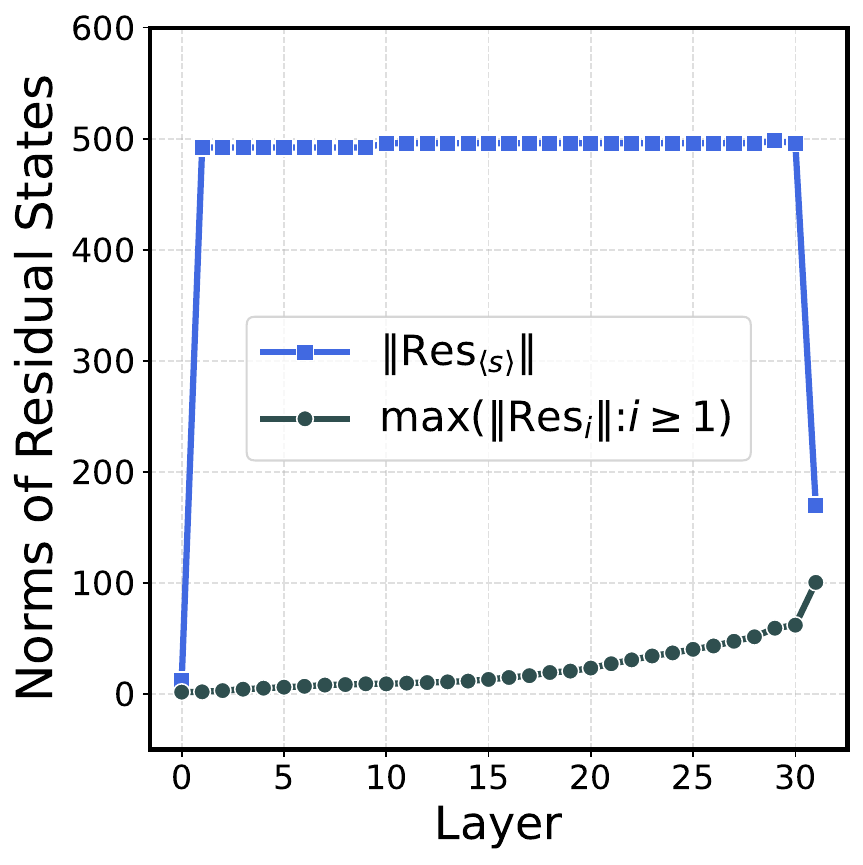}
      \label{fig:token_norms_massive}
  \end{subfigure}
  
  \vspace{-0.5em}
  \caption{\small \textbf{Extreme-token phenomena in Llama 3.1.}
  We evaluate the attention weights, value states norm, and residual states norm on the Llama 3.1-8B-Base model, where the input sentence is ``\bos Summer is warm\period~Winter is cold\period''.  \textit{Left (a)}: The attention weights across multiple heads at Layer 24. We observe the \textit{attention sink} phenomenon: the \bos{} token attracts a significant portion of the overall attention weight.
  \textit{Middle (b)}: The empirical distribution of the norms of value states over all layers 
  and all heads. We exclude 2\% of the outlier values to help visualization. We observe the \textit{value-state drain} phenomenon: the value state of the \bos~token is much smaller than those of other tokens on average.
  \textit{Right (c)}: The norm of the residual stream states, measured at the output of each layer. We observe the \textit{residual-state peak} phenomenon: the \bos~token's residual states have significantly larger norms than those of other tokens from layers 1 to 30. We present the extreme-token phenomena over other input sequences in Appendix~\ref{sec:many_samples}.
  }
  \label{figure:extreme-token}
  \vspace{-1em}
\end{figure}

The extreme-token phenomena have posed several challenges for pretrained transformers in downstream tasks. For instance, sink tokens require special treatment during long-context inference \citep{xiao2023efficient, han2023lm, yu2024unveiling, chen2024image} and model quantization \citep{dettmers2022gpt3, liu2024intactkv, son2024prefixing} to maintain high levels of performance. Additionally, attention sinks have reduced the interpretability of attention maps in vision transformers \citep{darcet2023vision}. To address these issues, \citet{sun2024massive} and \citet{darcet2023vision} propose adding a ``special token'' to transformers to serve as the sink token, preventing other tokens from becoming sinks. However, even this special token still exhibits extreme-token phenomena. Despite these efforts, no prior work has satisfiably explained the mechanisms behind the extreme-token phenomena. \citet{xiao2023efficient} proposes a hypothesis for why they occur, suggesting that models tend to dump unnecessary attention values to specific tokens. 


This work aims to demystify the extreme-token phenomena in LLMs. We demonstrate that these phenomena arise from an \textit{active-dormant mechanism} in attention heads (cf.\ Claim~\ref{claim:active-dormant}), coupled with a \textit{mutual-reinforcement mechanism} during pretraining (cf.\ Claim~\ref{claim:mutual-reinforcement}). We support these statements through studies on simplified transformer architectures and tasks, a dynamical theory for these models, and experiments on pretrained LLMs. The structure of the paper and our key contributions are outlined as follows:
\begin{enumerate}[leftmargin=2em]
\setlength\itemsep{0pt}
\item In \Cref{sec:bb_task},  we train one- to three-layer transformers on a simple task called the \textit{Bigram-Backcopy} (BB) task, which also displays extreme-token phenomena similar to those observed in LLMs. We show that attention sinks and value-state drains are a consequence of the \textit{\activedormant} (cf.\ Claim~\ref{claim:active-dormant}). Both theoretically and empirically, we demonstrate that \textit{mutual reinforcement mechanism} (cf.\ Claim~\ref{claim:mutual-reinforcement}) dynamically drives these phenomena: attention sinks and value-state drains reinforce one another, leading to a stable phase where all query tokens generate near identical attention logits for the keys of extreme tokens. Additionally, empirical results reveal that residual-state peaks arise from the interaction between this mutual reinforcement mechanism and the Adam optimization algorithm. 
\item In \Cref{sec:llm}, we demonstrate the \textit{\activedormant}~in pre-trained LLMs by showing that many attention heads transition between active and dormant phases based on the input domain. Specifically, we identify an interpretable active-dormant head (Layer 16, Head 25 in Llama 2-7B-Base \citep{touvron2023llama}) that activates on GitHub data but remains dormant on Wikipedia data. Moreover, in examining the dynamics of OLMo-7B-0424 \citep{groeneveld2024olmo}, we observe the same mutual reinforcement mechanism and stable phase, consistent with those found in the BB task. This demonstrates that the simple BB model captures both the static and dynamic properties of extreme-token phenomena in LLMs and accurately predicts their behavior. 
\item Importantly, the quantitative properties of extreme-token dynamics show strong consistency among the theoretical and empirical results of the Bigram-Backcopy task and the empirical performance of OLMo. In particular, we consistently observe the \textbf{sink-logits concentration} phenomenon, where the logits corresponding to the key of the extreme token and the queries of all non-extreme tokens ($ \text{logit}_{\cdot,\bos}$) are nearly identical—an observation not previously documented in the literature. We summarize the aligned results between the theoretical and empirical findings of the Bigram-Backcopy task and the empirical performance of LLMs in Table~\ref{tab:consistent_transposed}. 

\item We propose architectural and optimization modifications to mitigate the extreme-token phenomena. Specifically, we demonstrate that replacing SoftMax with ReLU activations in attention heads eliminates extreme-token phenomena in the BB task, while switching from Adam to SGD removes the residual-state peak phenomenon. We discuss the possibility that similar modifications could mitigate extreme-token phenomena in LLMs. 
\end{enumerate}

\begin{table}
    \centering
    \begin{tabular}{|c|c|c|c|}
    \hline
         & BB-task Theory & BB-task Experiments & LLM Experiments \\
    \hline
    \makecell{$\Delta \text{logit}_{\cdot,\bos}$ $\log$-growth} & \checkmark & \checkmark & $\star$ \\ \hline
    \makecell{$\|\vall_{\bos}\|$ monotonic decrease} & \checkmark & \checkmark & \checkmark \\ \hline
    \makecell{$\|\res_{\bos}\|$ linear growth} & $\star$ & \checkmark & \checkmark \\ \hline
    \makecell{$\text{logit}_{\cdot,\bos}$ concentration} & \checkmark & \checkmark & \checkmark \\
    \hline
    \end{tabular}
        \caption{Consistency of the quantitative properties across the theoretical and empirical results of the Bigram-Backcopy task and empirical results of LLMs. A \checkmark~denotes a consistent result, while a $\star$ denotes an inconclusive result. The $\text{logit}_{\cdot,\bos}$ denotes logits corresponding to the key of the extreme token and queries of all non-extreme tokens, i.e., the $\query_{\cdot}^\top \key_{\bos}$. The $\Delta \text{logit}_{\cdot,\bos}=\text{logit}_{\cdot,\bos}-\text{Mean}[\text{logit}_{\cdot,\text{others}}]$ is a progress measure for attention sinks. The $\|\vall_{\bos}\|$ denotes the value state norm of the extreme token, and $\|\res_{\bos}\|$ denotes the residual state norm of the extreme token. See \Cref{sec:prelim_notation} for the definitions of these notations. }
    \label{tab:consistent_transposed}
\end{table}

\subsection{Related work}
Several studies independently identified the ``attention sink'' phenomenon in language models and vision transformers,  where attention weights were found to be concentrated on a few tokens \citep{xiao2023efficient, darcet2023vision, han2023lm, zhai2023stabilizing, elhage2023privileged, dettmers2022gpt3}. Recent research has provided more detailed characterizations of this attention pattern and the attention sink phenomenon \citep{fu2024attentionpattern, sun2024massive}. \citet{sun2024massive} attributed the attention sink to the massive activation of the hidden representations of the corresponding tokens. Both \citet{sun2024massive} and \citet{zhai2023stabilizing} discussed methods for mitigating the attention sink by modifying the model and training recipes. Additionally, recent studies have leveraged the attention sink phenomenon to develop improved quantization and more efficient inference algorithms \citep{liu2024intactkv, chen2024image, yu2024unveiling, son2024prefixing, lin2024duquant, bondarenko2023quantizable, hu2024outlier}. A concurrent work by \citet{gu2024attention} studied how optimization, data distribution, loss function, and model architecture in LM pre-training influence the emergence of attention sink, showing that replacing the softmax function with sigmoid can prevent attention sink emergence in models up to 1B parameters.

The dynamics of transformers are studied under various simplifications, including linear attention structures \citep{zhang2023trained,ahn2024transformers}, reparametrizations \citep{tian2023joma}, NTK \citep{deora2023optimization}, often in the setting of in-context linear regression \citep{ahn2023linear,wu2023many,zhang2024context} and structured sequences \citep{bietti2024birth,nichani2024transformers,tian2023scan}. Notably, \citet{zhang2023trained, huang2023context, kim2024transformers} demonstrate that a one-layer attention head trained via gradient descent converges to a model that effectively performs in-context regression. \cite{bietti2024birth} shows the fast learning of bigram memorization and the slow development of in-context abilities. \cite{tian2023scan} shows the scan and snap dynamics in reparametrized one-layer transformers. \cite{reddy2023mechanistic} simplifies the structure of the induction head, showing the connection between the sharp transitions of in-context learning dynamics and the nested nonlinearities of multi-layer operations.

Mechanistic interpretability is a growing field focused on understanding the internal mechanisms of language models in solving specific tasks \citep{elhage2021mathematical, geva2023dissecting, meng2022locating, nanda2023progress, olsson2022context, bietti2024birth, wang2022interpretability, feng2023language, todd2023function}. This includes mechanisms like the induction head and function vector for in-context learning \citep{elhage2021mathematical, olsson2022context, todd2023function, bietti2024birth}, the binding ID mechanism for binding tasks \citep{feng2023language}, association-storage mechanisms for factual identification tasks \citep{meng2022locating}, and a complete circuit for indirect object identification tasks \citep{wang2022interpretability}. The task addressed in this paper is closely related to \cite{bietti2024birth}, who explored synthetic tasks where tokens are generated from either global or context-specific bigram distributions. Several other studies have also employed synthetic tasks to explore neural network mechanisms \citep{charton2022my, liu2022towards, nanda2023progress, allen2023physics, zhu2023physics, guo2023transformers, zhang2022unveiling, lin2023transformers}. 

A line of work focuses on quantizing neural networks using low-bit fixed-point representations \citep{jacob2018quantization,zafrir2019q8bert,lin2020towards,nagel2021white,gholami2022survey}, such as \texttt{INT8} \citep{lin2020towards,dettmers2022gpt3} or \texttt{INT4} \citep{yao2206efficient,wu2023understanding,dettmers2023case} to save memory usage and computational cost. In LLMs, the extreme-token phenomena lead to substantial performance degradation after quantization \citep{bondarenko2021understanding} and have become a key focus of recent research \citep{fan2020training,yao2022zeroquant,lin2024duquant,hu2024outlier}. \citet{dettmers2022gpt3} and \citet{lin2024awq} propose mixed-precision approaches, using \texttt{FP16} for outlier values and \texttt{INT8} for others, enabling large model quantization without performance loss. \citet{xiao2023smoothquant} rescales the weights and activations to reduce magnitudes of outliers, and \citet{bondarenko2023quantizable} proposes modified attention structures to remove outliers, making language models easier to quantize.

We note that \citet{gurnee2024universal} proposed Attention Deactivation Neurons, \citet{bondarenko2023quantizable} proposed the ``no-op'' hypothesis, and \citet{xiao2023efficient} proposed the ``dump unnecessary attention'' conjecture as mechanisms of attention sinks. In contrast, we explain the extreme-token phenomena through the active-dormant and mutual reinforcement mechanisms, offering the proof of their emergence within training dynamics in a toy model and providing empirical evidence of these mechanisms in LLMs.



\subsection{Preliminaries and notations}\label{sec:prelim_notation}

 %

While different LLMs may use slightly varying transformer architectures, most use the structure proposed by \cite{vaswani2017attention}, with the key modification being the shift from post-norm to pre-norm. We represent the tokenized input sequence of length $n$, with positional embeddings included, as $\bH = [\bh_1, \dots, \bh_n]\in \R^{d \times n}$, where $\bh_i$ denotes the $i$th input token, and $d$ is the embedding dimension. We denote the layer-normalization operation as $\LN$, the column-wise SoftMax operation as $\softmax$, the causal-mask as $\mask$, and the pointwise ReLU function as $\relu$. 

The transformer architecture applies causal-attention and MLP layers iteratively to the input sequence $\bH$. A causal-attention layer with $M$ heads is represented as $\Attn(\cdot)$, parameterized by $\sets{ (\bQ_m,\bK_m, \bV_m, \bO_m)}_{m}$: 
\begin{talign}
\label{eqn:attention}
\Attn(\bH) \defeq \sum_{m = 0}^{M-1}  \attn_{m}(\bH)\in \R^{d \times n},
\end{talign}
where each attention head $\attn_m(\cdot)$ is given by
\begin{talign}\label{eqn:attention_head_prelim}
\attn_m(\bH) \defeq  \bO_m \bV_m \LN(\bH) \softmax\paren{\mask\paren{ \LN(\bH)^\top \bK_m^\top \bQ_m \LN(\bH) }}. 
\end{talign}
We denote the attention map as ${\sf Map}=\softmax\paren{\mask\paren{ \LN(\bH)^\top \bK_m^\top \bQ_m \LN(\bH) }}$, and typically plot its transpose, ${\sf Map}^\top$, in figures.

An MLP layer, denoted $\mlp(\cdot)$, has parameters $(\bW_1,\bW_2)$:
\begin{talign}\label{eqn:MLP_layer_prelim}
\mlp(\bH) \defeq \bW_2 \relu(\bW_1 \LN(\bH) ) \in \R^{d \times n}.
\end{talign}
An $L$-layer transformer consists of a composition of $L$ self-attention and MLP layers with residual connection structure. Given an input $\bH^{(0)} \in \R^{d \times n}$, the output of the $L$-layer transformer, $\bH^{(L)}$, is computed as follows: 
\begin{talign}
\bH^{(\ell + 1)} =\bH^{(\ell + 1/2)} +  \mlp^{(\ell)}\paren{ \bH^{(\ell + 1/2)}},~~~ \bH^{(\ell + 1/2)} = \bH^{(\ell)} + \Attn^{(\ell)}\paren{\bH^{(\ell)}},~~~ \ell\in\set{0,\dots,L-1}.
\end{talign}
For consistency between the code and the text, we adopt zero-indexing throughout this paper, meaning that attention head and layer indices begin at $0$ instead of $1$. 

For the output $\bH^{(\ell + 1)}$ of layer $\ell$, we define the residual state $\res_{\tok}$ of a token $\tok \in \{0, 1, \ldots, n-1\}$ as the $\tok$th column of $\bH^{(\ell + 1)}$. For a specific layer $\ell$ with input $\bH^{(\ell)} \in \R^{d \times n}$, and for a specific attention head $m$ with query, key, and value matrices $(\bQ, \bK, \bV, \bO)$, we define the \textit{query}, \textit{key}, and \textit{value states} $(\query_\tok, \key_\tok, \val_\tok)$ of a token $\tok \in [n]$ as the $\tok$th columns of $\bQ \bH^{(\ell)}$, $\bK \bH^{(\ell)}$, and $\bO \bV \bH^{(\ell)}$, respectively\footnote{We define the value state as $\bO \bV \bH$ rather than $\bV \bH$ since what is added to the residual state is essentially $\bO \bV \bH$. }. The \textit{attention logit} $\text{logit}_{\tok^\prime, \tok}$ is defined as the $(\tok^\prime, \tok)$th element of $(\bH^{(\ell)})^\top \bQ^\top \bK \bH^{(\ell)}$. For notation simplicity, we omit the dependence on $\ell$ and $m$ in $(\query_\tok, \key_\tok, \val_\tok, \text{logit}_{\tok^\prime, \tok})$, as these will be clear from context.  Additionally, for a fixed token $\tok$, we use the shorthand $\text{logit}_{\cdot,\tok}$ for the set $\{\text{logit}_{\tok^\prime, \tok} \mid \tok^\prime\in\vocab\}$.

We use \bos~to refer to the "Beginning-of-Sequence" token. Since the \bos~token consistently behaves as an extreme token in LLMs, we often refer to \bos~and extreme tokens interchangeably. We also abuse the notation by writing $(\query_{\bos}, \key_{\bos}, \val_{\bos})$ to represent the query, key, and value states of the \bos~token.



\section{Extreme-token Phenomena in the Bigram-Backcopy Task}\label{sec:bb_task}

In this section, we analyze simple transformers trained on the Bigram-Backcopy (BB) task, a simple model that exhibits extreme-token phenomena. We demonstrate the \textit{active-dormant mechanism} (cf.\ Claim~\ref{claim:active-dormant}) and \textit{mutual reinforcement mechanism} (cf.\ Claim~\ref{claim:mutual-reinforcement}) within the BB task and provide predictions for the behavior of sink tokens, which will be validated through LLM experiments in the following section. 

The Bigram-Backcopy task is a data-generation model that consists of two sub-tasks: \textit{Bigram-transition} and \textit{Backcopy}. In this model, each sequence begins with the \bos~token, followed by tokens sampled according to a pre-determined bigram transition probability $\transit$ (in other words, a Markov chain). When specific trigger tokens are encountered, instead of sampling according to the transition $\transit$, the preceding token is copied to the next position. An illustration of the Bigram-Backcopy task is provided in \Cref{fig:bbm-dgp}. Following \citet{bietti2024birth}, we select the transition $\transit$ and the vocabulary $\vocab$ with $| \vocab | = V = 64$ based on the estimated character-level bigram distribution from the \textit{tiny Shakespeare} dataset. In all experiments, the set of trigger tokens, $\cT$, is fixed and consists of the $\abs{\cT}=3$ most frequent tokens from the unigram distribution. Consequently, the non-trigger token set, $\vocab\setminus \cT$, comprises $61$ tokens.

\subsection{One-layer transformer exhibits \attnsink s and \valuedrain s}

On the \bb~task, we pre-train a standard one-layer transformer with a single SoftMax \attn~head and one \mlp~layer. Unless otherwise specified, the model is trained using Adam for $10,000$ steps, achieving near-optimal prediction accuracy. Detailed training procedures are provided in Appendix~\ref{appsec:experiment-detail}. Figure~\ref{fig:bbm-attn} shows that the trained transformer exhibits the \attnsink~phenomenon, where the \bos~token captures a significant proportion of the attention weights. More importantly, the attention weights display interpretable patterns: all non-trigger tokens exhibit \attnsink s, while the attention for trigger tokens is concentrated on their preceding positions. Additionally, Figure~\ref{fig:bbm-value} reveals a value-state drain phenomenon similar to that observed in LLMs, suggesting that, for non-trigger tokens, the \attn~head contributes minimal value to the residual stream. We provide additional attention patterns on different input sequences in \Cref{appsec:additional-attn-maps}.

\begin{figure}[t]
  \centering
  \begin{minipage}{0.38\textwidth}
      \centering
      \subcaption{\small The Bigram-Backcopy task}
      \label{fig:bbm-dgp}
      \vspace{-.2em}
      \includegraphics[width=\linewidth]{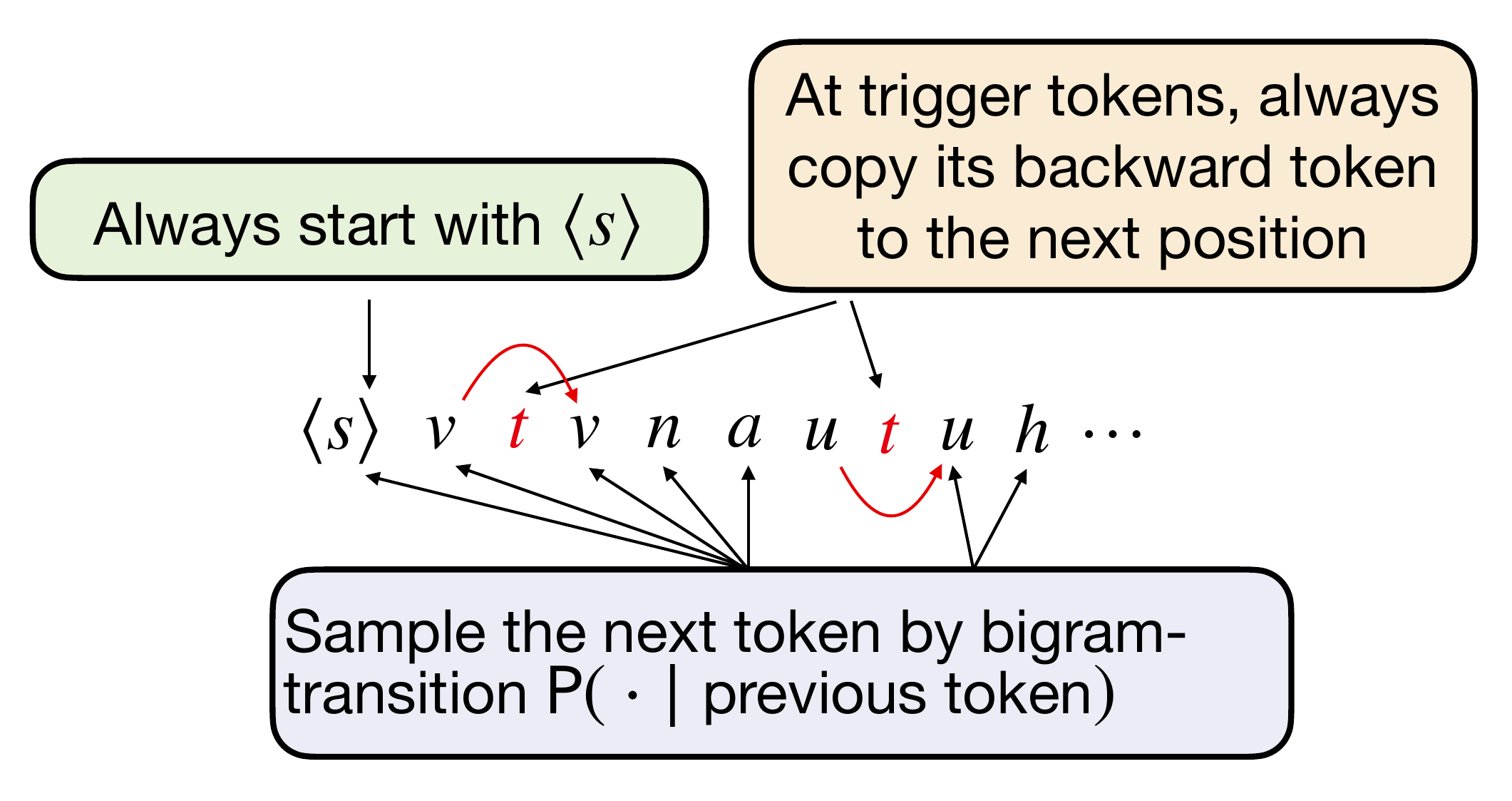}
  \end{minipage}
  \begin{minipage}{0.26\textwidth}
      \centering
      \subcaption{\small Attention pattern}
      \label{fig:bbm-attn}
      \vspace{-.2em}
      \includegraphics[width=\linewidth]{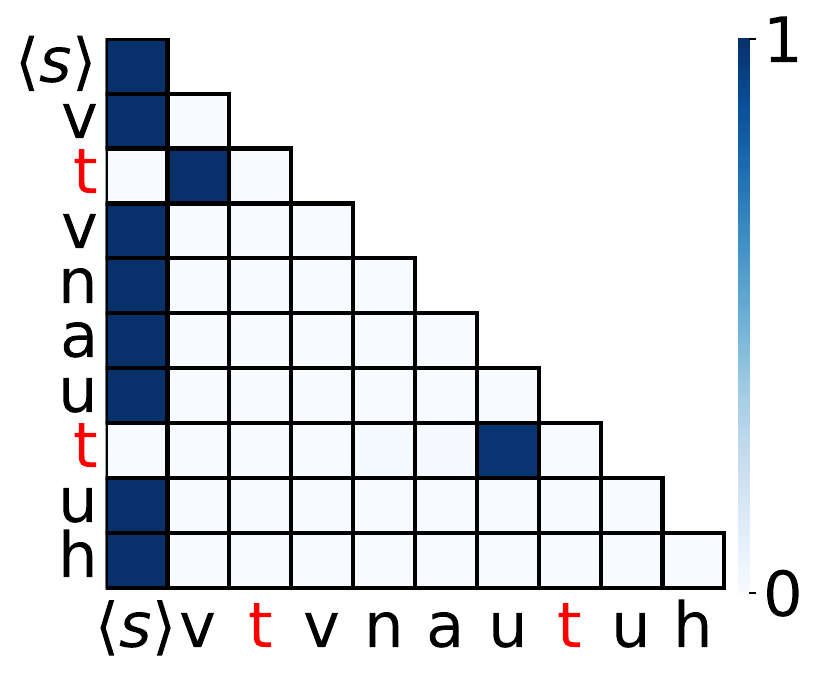}
  \end{minipage}
  \begin{minipage}{0.27\textwidth}
      \centering
      \subcaption{\small Small value states}
      \vspace{0pt}
      \label{fig:bbm-value}
      \vspace{-.2em}
      \includegraphics[width=\linewidth]{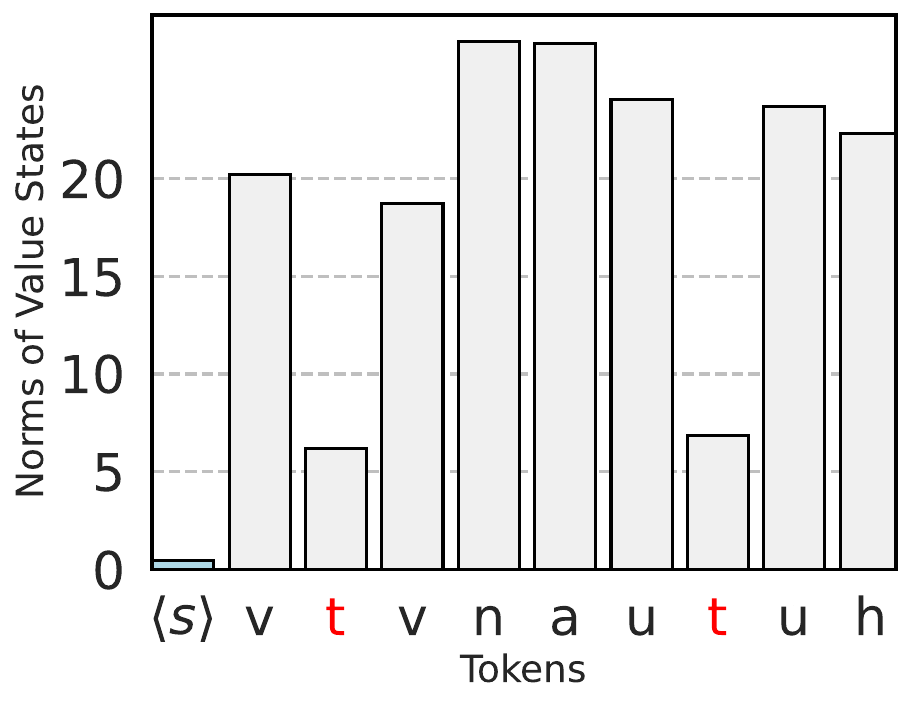}
  \end{minipage}
  \caption{\small \textbf{Experiments on the Bigram-Backcopy task.} 
  \textit{Left (a)}: The data generation procedure for the Bigram-Backcopy task. Here we fix `t', `e', and the space character (` ') as trigger tokens. The BB task samples bigram transitions for non-trigger tokens and backcopies for trigger tokens.  \textit{Middle (b)}: The attention map of a given prompt. Trigger tokens are marked in red. The attention head at non-trigger tokens is dormant and displays attention sinks.  \textit{Right (c)}: The value state norms for the prompt. The \bos~token has the smallest norm. 
  }
  \label{figure:pretraining-findings}
  \vspace{-1em}
\end{figure}


\paragraph{The \activedormant~of the attention head.} Inspired by the interpretable attention weight patterns observed, we propose the \textit{\activedormant}. For any given token, an attention head is considered \textit{active} if it makes a significant contribution to the residual state, and \textit{dormant} if its contribution is minimal. As illustrated in Figure~\ref{fig:bbm-attn}, when trained on the BB task, the attention head is active for trigger tokens and dormant for non-trigger tokens. 

Figure~\ref{fig:interventions} demonstrates that the \mlp~layer is responsible for the Bigram task whereas the \attn~head takes care of the Backcopy task. When the \mlp~layer is zeroed out, the backcopy loss remains significantly better than a random guess, but the bigram loss degrades to near-random levels. Conversely, when the \attn~layer is zeroed out, the backcopy loss becomes worse than a random guess, while the bigram loss remains unaffected. This indicates that on trigger tokens, the \attn~head is active and handles the backcopy task, whereas on non-trigger tokens, the \attn~head is dormant, allowing the \mlp~layer to handle the Bigram task. We summarize the \activedormant~of the \attn~head in Claim~\ref{claim:active-dormant}.

\begin{figure}[h]
    \centering
    \begin{minipage}{0.65\textwidth}
\begin{claim}[Active-dormant mechanism]
\label{claim:active-dormant}
Attention heads of pre-trained models are often governed by the \activedormant, exhibiting two phases:
\vskip5pt
\begin{itemize}[leftmargin=2em]
\setlength\itemsep{5pt}
\item[\textup{(1)}] \textbf{Dormant phase}: On non-trigger tokens, the \attn~head assigns dominant weights to the \bos~token, adding minimal value to the residual stream and having little impact on the model’s output.
\item[\textup{(2)}] \textbf{Active phase}: On trigger tokens, the \attn~head assigns dominant attention weights to relevant context tokens, adding substantial value to the residual stream and significantly impacting the model’s output. 
\end{itemize}
\end{claim}
    \end{minipage}
    \hfill
    \begin{minipage}{0.33\textwidth}
        \includegraphics[width=0.89\linewidth]{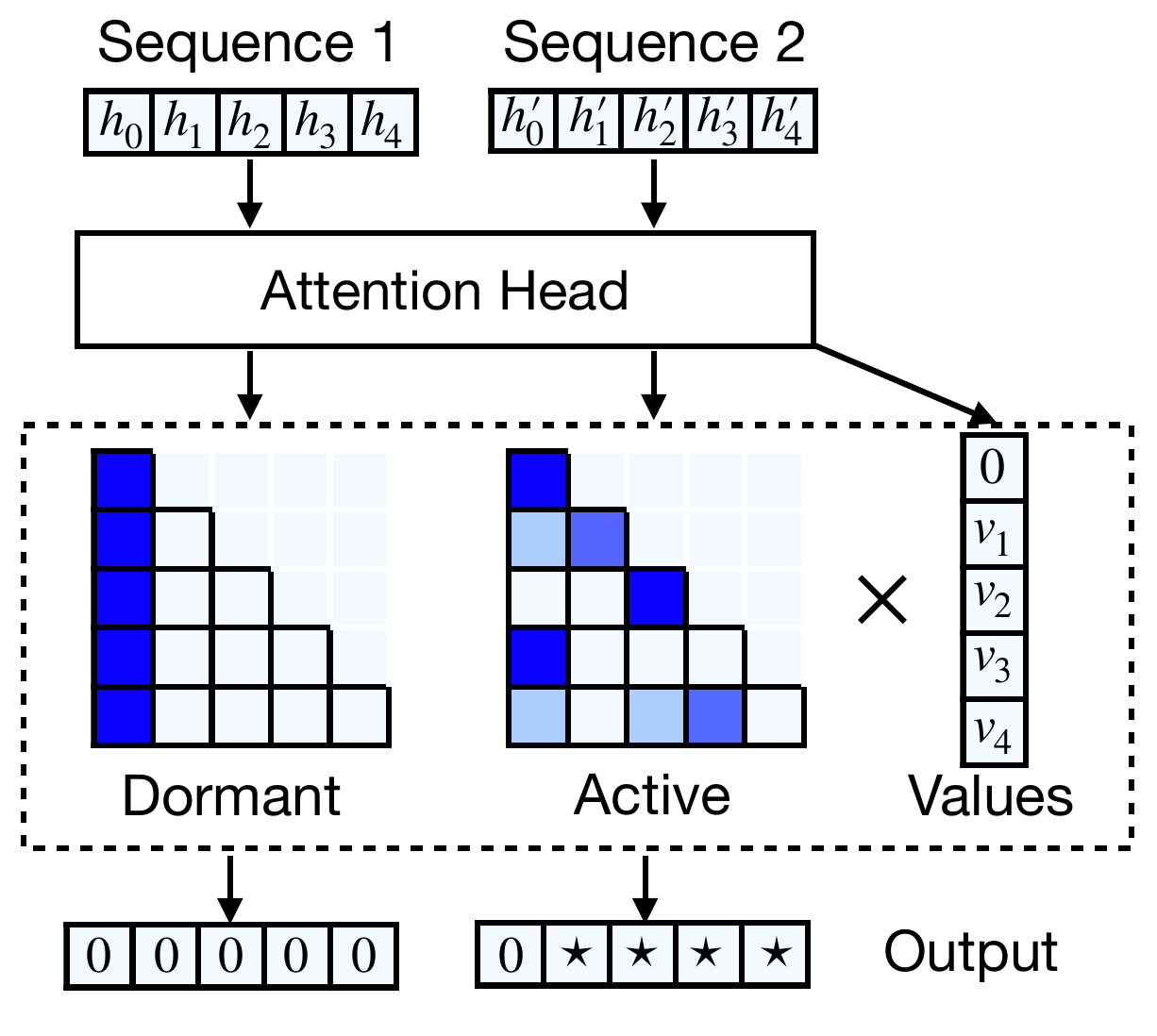}
        \vskip-8pt
        \caption{\small Active-dormant mechanism}
        \label{figure:illustrate-active-dormant}
    \end{minipage}%
\end{figure}


\begin{figure}
  \centering
  \begin{minipage}{0.37\textwidth}
      \centering
      \subcaption{\small Excess risk after interventions}
      \label{fig:interventions}
      \includegraphics[width=0.94\textwidth]{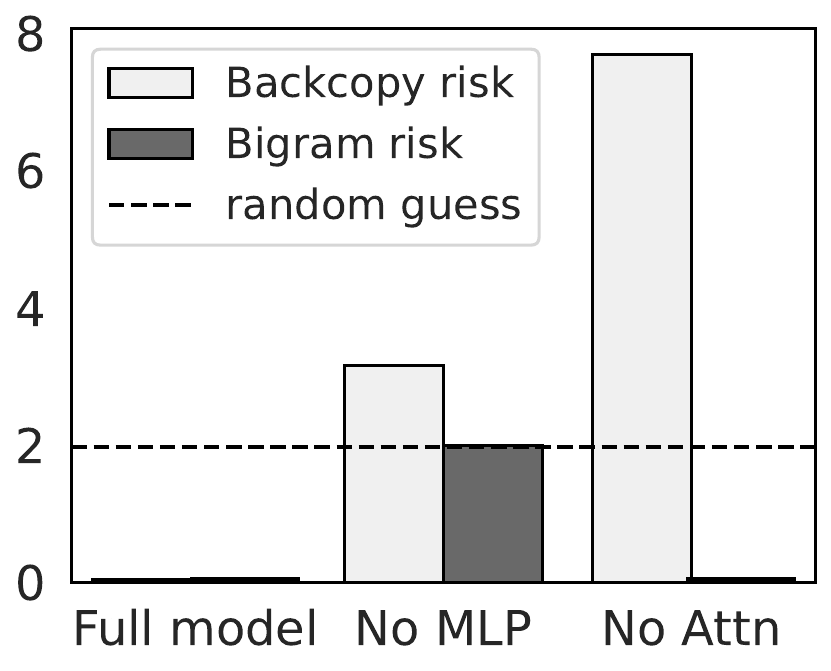}
  \end{minipage}
  \begin{minipage}{0.6\textwidth}
      \centering
      \subcaption{\small Training dynamics }
      \label{fig:dynamics}
    \includegraphics[width=0.87\textwidth]{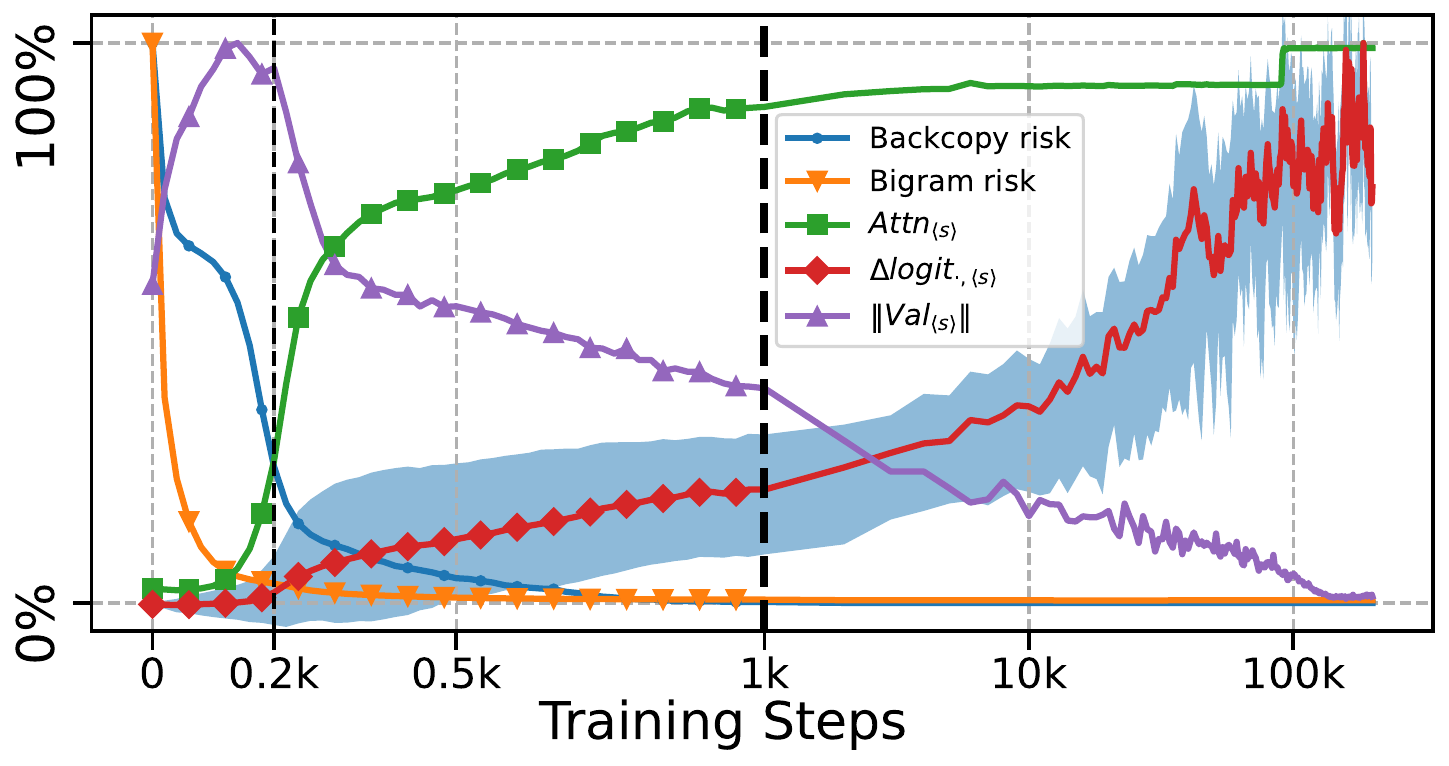}
  \end{minipage}
  \hspace{-1em}
    \caption{\small \textbf{Interventions and dynamics of one-layer transformer on the Bigram-Backcopy task.}  \textit{Left (a)}: Excess risks for a one-layer model trained on the Bigram-Backcopy (BB) task under various interventions. \textit{Right (b)}: The excess risks, attention weights, attention logits, and value state norms for the \bos~token throughout the training dynamics. Each curve is rescaled to fall within a 0 to 1 range. On the right side of \textit{(b)}, the horizontal axis is logarithmically scaled. The $\Delta\text{logit}_{\cdot,\bos}$ curve represents the mean of attention logits from all given non-trigger query tokens $\tok$ on the \bos~token, normalized by the mean of attention logits for other tokens. The shaded area represents the 90\% uncertainty interval on the distribution over all non-trigger tokens. 
    }
    \label{figure:verify-assumptions}
\end{figure}

\paragraph{The growth of attention logits on the \bos~token and the decrease in its value state norms.} Figure~\ref{fig:dynamics} illustrates the training dynamics of excess risks, attention weights, attention logits (for each token $\tok_n$ at position $n$ in the prompt, we compute $\Delta\text{logit}_{\cdot,\bos} \equiv \mathtt{mean}_{n}[\langle \query_{\tok_n}, \key_\bos \rangle - \mathtt{mean}_{i}(\langle \query_{\tok_n}, \key_{\tok_\toki}) \rangle]$, which serves as a progress measure for attention sinks), and value state norms for the \bos~token. All values are rescaled to the $0$ to $1$ range to highlight trends rather than absolute values. Both the Bigram and Backcopy excess risks decrease to nearly zero within the first 1000 steps, with the Bigram excess risk approaching zero faster than the Backcopy risk. As the Backcopy risk decreases, the attention weights on the \bos~token begin to increase, suggesting a connection between the formation of attention sinks and the backcopy function in the attention heads. After the first $1000$ steps, although both Bigram and Backcopy excess risks have nearly reached zero, the attention logits and weights on the \bos~token continue to increase, while the value state norm of the \bos~token continues to decrease. While this is an intriguing phenomenon, our next goal is to understand why the attention logits and value state norms continue to evolve toward extreme values.

\subsection{Analysis of a minimally-sufficient transformer architecture}
\label{sec:simple-model}

\begin{figure}[t]
    \centering
    \includegraphics[width=0.7\linewidth]{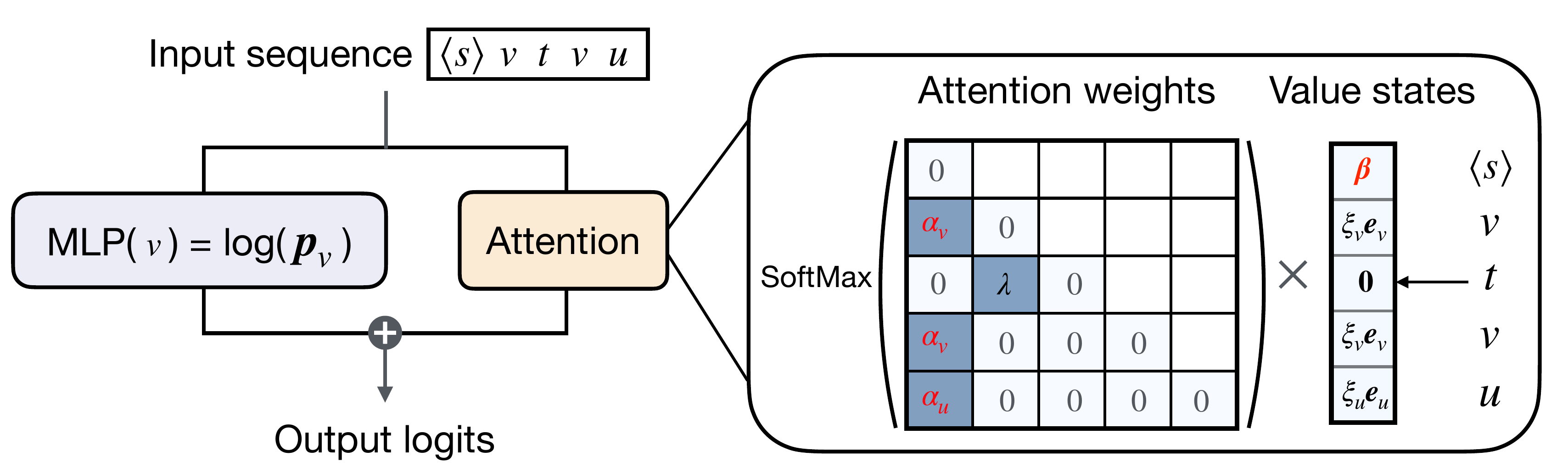}
    \caption{\small \textbf{Simplified transformer architecture.} The output logits are computed by summing the contributions from both the \mlp~layer and the \attn~head. The predicted probabilities are obtained by applying the SoftMax function to these output logits. The \mlp~layer is assumed to provide the Markov transition probabilities for non-trigger tokens, while the \attn~head is parameterized by attention logits and value states, as described in Eq.~(\ref{eqn:simplification_TF_1}), (\ref{eqn:simplification_TF_2}), and (\ref{eqn:simplification_TF_3}). Additionally, the trainable variables, denoted by $(\vecsink,\vecvalue) \in \R^V \times \R^V$, represent the attention logits and value states of the \bos~token.}
    \label{figure:simple-model}
\end{figure}

In this section, we analyze the training dynamics of transformers on the BB task, focusing on a simplified architecture that retains the attention sinks and value-state-drains phenomena. We analyze the regime when the Bigram transition probability is fully learned, and the Backcopy task is partially learned (i.e., after step $200$ in Figure~\ref{fig:dynamics}), and we focus on the dynamics of the attention logits and value states. Readers who are more interested in the results than the theoretical analysis can skip the detailed analysis and proceed directly to the statement of the mutual reinforcement mechanism in Claim~\ref{claim:mutual-reinforcement}. 

Let $\vocab$ (of size $V$) denote the set of all tokens excluding the $\bos$ token, and let $\cT$ represent the set of all trigger tokens. For any  $\tok \in \vocab$, we define $\transition_{\tok\tokk}=\transit(\tokk|\tok)$ as the next-token Markov transition probability, and $\bm{\transition}_v = (p_{v1}, \ldots, p_{vV})^\top \in \Delta(\cV)$ as the transition vector in the simplex. The embedding map is denoted by $\embd: [n] \times \cV \to \R^D$, where for a token $v \in \cV$ at position $i \in [n]$, the embedded vector is $\embd_i(\tok)$. The \bos~token always appears at position $0$, and we denote its embedding vector by $\embd(\bos)$. For simplicity, we abuse the notation and use the sequence itself, $[\bos, \tok_1, \ldots, \tok_n]$ where $\{ \tok_{k} \}_{k \in [n]} \subseteq \cV$, to represent the embedding of the sequence. 

Given an input sequence $\bH = [\bos, v_{1:n}] \in \R^{D \times (n+1)}$ with \bos~as the zeroth token, we define the predicted probability of the next token as $\softmax(\TF(\bH)_n)$, where $\TF(\bH)_n \in \R^D$ is the last column of $\TF(\bH) \in \R^{D \times (n+1)}$, defined as
\begin{equation}\label{eqn:simplified_transformer}
\TF(\cdot) = \attn(\cdot) + \mlp(\cdot),~~\attn(\bH) = \bV \bH \softmax(\mask(\bH^\top \bK^\top \bQ \bH ) ),~~ \mlp(\bH) = \bW_2 \relu( \bW_1 \bH).
\end{equation}
The simplified transformer architecture $\TF$ is a parallel summation of the $\attn$ head and the $\mlp$ layer, with no layer normalization. This parallel summation is a reasonable simplification, as sequential $\attn$ and $\mlp$ layers can effectively simulate parallel $\attn$ and $\mlp$ operations. Notice that we have redefined the notations of $\attn$ and $\mlp$ in this section, which are simplified versions of Eq.~(\ref{eqn:attention_head_prelim}) and (\ref{eqn:MLP_layer_prelim}). 

\paragraph{Simplification and reparameterization of the model.} To simplify the analysis of the training dynamics, we further reduce the model by restricting the $(\bK, \bQ, \bV, \bW_1, \bW_2)$ matrices to follow the patterns observed in the later training stages (i.e., after step 200 of the training in \Cref{fig:dynamics}). 
\begin{itemize}[leftmargin=2em]
\setlength\itemsep{0pt}
\item \textit{Restricted Attention Pattern}. Based on the intuition from \Cref{fig:bbm-attn}, we know that eventually only a few attention logits are non-trivial. Thus, we assume that the model has learned the attention pattern by this stage (which is reasonable given that the Backcopy risk is already small after step 200 in \Cref{fig:dynamics}). We parameterize the attention logits on the $\bos$ key-token as $(\sink_{\bos}; \sink_{v_1}; \ldots; \sink_{v_n})$, restrict the attention logits for any trigger query-token to $(0, \ldots, \lambda, 0)$ (where the second-to-last coordinate is $\lambda$), and set all other logits to zero. Specifically, we restrict:
\begin{equation}\label{eqn:simplification_TF_1}
\begin{aligned}
&~\embd(\bos)^\top \bK^\top \bQ \cdot \embd_i(\tok) =\sink_\tok \cdot 1 \{ v \not\in \cT \}~~~\text{for } \tok\in\vocab, i \in [n], \\
&~ \embd_i(\bar \tok)^\top \bK^\top \bQ \cdot \embd_j(\tok) = \lambda \cdot 1\{ \tok \in \cT, i = j-1 \} ~~~\text{for } \tok, \bar \tok \in \cV, i, j \in [n]. 
\end{aligned}
\end{equation}
Notice that this naturally implies $\sink_v = 0$ for $v \in \cT$. 
\item \textit{Restricted Value Pattern}. At later stages of the training dynamics, we observe that the value states for each token are nearly a scaled version of the one-hot encoding vector. We assume this observed pattern and parameterize the value state of $\tok$ by $\xi_\tok \bm{e}_{\tok} \in \R^V$. For the \bos~token, we parameterize its value state by $\vecvalue \in \R^{V}$. Specifically, we restrict
\begin{equation}\label{eqn:simplification_TF_2}
\begin{aligned}
&~ \bV \cdot \embd(\bos)  = \vecvalue \in \R^{V}, \\
&~ \bV \cdot \embd_i(\tok)  = \xi_\tok \bm{e}_\tok \in \R^{V},  ~~~\text{with $\xi_\tok=0$ for $\tok\in\cT$, and $\xi_\tok\geq 0$ for $\tok\in\vocab\setminus\cT$.} 
\end{aligned}
\end{equation}
\item \textit{MLP Layer Perfectly Predicts the Transition Probability}. Notice that the \mlp~layer handles the Bigram task. By step 200 in \Cref{fig:dynamics}, the Bigram risk has nearly vanished. Therefore, we assume that the $\mlp$ layer outputs the Markov transition probabilities $\bm{\transition}_{\tok}$ for non-trigger tokens $\tok$, and zero for trigger tokens. Specifically, we restrict: 
\begin{equation}\label{eqn:simplification_TF_3}
\mlp(\embd_i(\tok)) = \log \bm{\transition}_v \cdot 1\{ \tok \not\in \cT  \} ~~~ \text{for } \tok \in \vocab. 
\end{equation}
\end{itemize}

These reparameterizations are illustrated in Figure~\ref{figure:simple-model}. Theorem~\ref{thm:construction} establishes the existence of a transformer architecture that satisfies the restrictions and reparameterizations outlined above. Furthermore, this restricted transformer can generate the ground-truth transitions of the BB model when certain parameters diverge. 
\begin{theorem}[Existence of reparameterization that solves the BB task; informal]\label{thm:construction}
For any parameters $(\vecsink \in \R^{V}, \vecvalue \in \R^V, \bm{\xi} \in \R^V, \lambda \in \R)$, there exists a one-layer transformer as described in (\ref{eqn:simplified_transformer}) with weight matrices $(\bQ, \bK, \bV, \bW_1, \bW_2)$ such that Eq. (\ref{eqn:simplification_TF_1}), (\ref{eqn:simplification_TF_2}), and (\ref{eqn:simplification_TF_3}) hold. Furthermore, there exists a sequence of parameters where $\min_{v \in \vocab} \sink_v \to\infty$, $\min_{v \in \vocab} \xi_v \to\infty$, $\lambda\to \infty$, and $\vecvalue=0$, such that this transformer generates the ground-truth transitions of the BB model in the limit. 
\end{theorem}
The formal statement and proof of \Cref{thm:construction} are provided in Appendix~\ref{app:proof-construction}. 

\paragraph{Dynamic analyses of the reparameterized model.} To analyze the later stage training dynamics, we adopt the reparameterization given in Eq.~(\ref{eqn:simplification_TF_1}), (\ref{eqn:simplification_TF_2}), and (\ref{eqn:simplification_TF_3}) as our assumption. We further define $\mass_{\tokk} = \sum_{i = 1}^n 1\{ \tok_i = \tokk \}$, $\bm{\mass} = (\mass_1, \ldots, \mass_V)$, and $\mass = \sum_{\tokk \in \vocab} \mass_{\tokk} = n$. Substituting these into Eq.~\eqref{eqn:simplified_transformer}, for a non-trigger token $v \in \cV \setminus \cT$, the output of the attention layer with input sequence $\bH = [\bos, v_{1:n-1}, v]$ is given by 
\begin{equation}\label{eqn:q}
\TF(\bH)_n = \log \bm{\transition}_v + \frac{e^{\sink_\tok}}{e^{\sink_\tok} + \mass} \vecvalue + \sum_{\tokk=1}^{\vocabsize} \frac{\mass_\tokk \xi_\tokk}{e^{\sink_\tok} + \mass} \cdot \bm{e}_\tokk.
\end{equation}
Therefore, for the non-trigger token $\tok$, the cross-entropy loss between the true Markov transition $\bm{\transition}_\tok$ and the predicted transition $\softmax(\TF(\bH)_n)$ is given by
\begin{equation}\label{eqn:loss_single}
\loss_\tok(\sink_\tok, \vecvalue) = \sum_{\tokk=1}^\vocabsize \transition_{\tok\tokk}\Big\{ \log \Big[ \sum_{\toki=1}^\vocabsize \transition_{\tok\toki} \exp\Big(\frac{e^{\sink_\tok}\ivalue_\toki+\mass_\toki\xi_\toki}{e^{\sink_\tok}+\mass}\Big) \Big] - \frac{e^{\sink_\tok}\ivalue_\tokk+\mass_\tokk\xi_\tokk}{e^{\sink_\tok}+\mass} - \log \transition_{\tok\tokk} \Big\}.
\end{equation}
For simplicity, we neglect the loss on trigger tokens and assume that $(\{ \mass_i \}_{i \in [V]}, \mass)$ remain fixed across different positions in the input sequences.\footnote{We note that \cite{reddy2023mechanistic} makes a similar simplification in analyzing induction heads.} We then consider the total loss as the average of the losses on each non-trigger token, weighted by its proportion in the stable distribution $\{\pi_v\}_{v \in \vocab}$, given by
\begin{equation}\label{eqn:total_loss}
\textstyle \loss(\vecsink,\vecvalue) = \sum_{\tok \in \vocab \setminus \cT} \stable_\tok \cdot \loss_\tok(\sink_\tok, \vecvalue).
\end{equation}
We assume that $\bm{\xi}$ and $\lambda$ are fixed, and that $\vecsink$ (the attention logits of the \bos~token) and $\vecvalue$ (the value state norms of the \bos~token) are trainable variables, as we are interested in the dynamics of the attention logits and value state norm for the \bos~token. The following theorem illustrates the logarithmic growth of the attention logits $\vecsink$, the shrinkage of value states $\vecvalue$, and the stable phase of these two variables.
\begin{theorem}\label{thm:main}
Consider the gradient flow of the loss function $\loss(\vecsink, \vecvalue)$. Assume $\xi_\tok \ge 0$ for any $\tok$ and $\stable_\tok > 0$ for any $\tok\in\vocab$, and $\{ \mass_i \cdot \xi_i \}_{i \in \vocab}$ are not all equal. 
\begin{itemize}[leftmargin=2em]
\setlength\itemsep{0pt}
    \item[\textup{(a)}] (Attention logits grow logarithmically, reinforced by small value states) Fix $\vecvalue= \beta \cdot \bm{1}$ for a constant $\beta$, and consider the gradient flow over $\vecsink$. With any initial value $\vecsink(0)$, there exists $\bm{r}(t)$ with norm uniformly bounded in time, such that 
    \begin{equation}
    \textstyle \vecsink(t) = \frac{1}{2} \log t \cdot \bm{1} + \bm{r}(t).\end{equation}
    \item[\textup{(b)}] (Value state shrinks to a small constant vector, reinforced by large attention logits) Fix $\vecsink = \sink \cdot \bm{1}$ for a constant $\sink$, define $\bar{\ivalue}(0) = V^\inv[\sum_{\tok} \ivalue_\tok(0)]$ and $\meanvalue=\vocabsize^\inv[\sum_{\tok} \mass_\tok \xi_\tok]$. Consider the gradient flow over $\vecvalue$. As $t \to \infty$, we have
    \begin{equation}\vecvalue(t) \to \vecvalue^\star = [\bar\ivalue(0)+e^{-\sink} \meanvalue] \cdot \bm{1} - e^{-\sink} \cdot \bm{\mass}\circ \bm{\xi}.\end{equation}
    \item[\textup{(c)}] (Stable phase: Sink-logits concentration) Consider the gradient flow over the variables $(\vecsink, \vecvalue)$. Any vector of the following form
    \begin{equation}\vecsink = \sink \cdot \bm{1}, \quad \vecvalue = c \cdot \bm{1} - e^{-\sink} \cdot \bm{\mass} \circ \bm{\xi},  ~~~ \sink, c\in\R \end{equation}
     is a stationary point. These are all global minimizers of $\loss(\vecsink,\vecvalue)$.
\end{itemize}
\end{theorem}
The proof of Theorem~\ref{thm:main} is provided in Appendix~\ref{app:proof-main-3}, \ref{appsec:proof-main-1}, and \ref{appsec:proof-main-2}. We offer two key remarks: (1) As $\sink_\tok\to\infty$, a Taylor expansion of the gradient $\partial \loss/\partial \sink_\tok$ suggests that $\mathrm{d}\sink_\tok/\mathrm{d}t \propto \exp(-2 \sink_\tok)$, which leads to the logarithmic growth of $\sink_\tok$. Similar logarithmic growth has been reported in the literature under different setups \citep{tian2023scan,zhu2024towards};
(2) The stable phase described in Theorem~\ref{thm:main}(c) seems to imply that the system can remain stable without attention sinks, as it does not require $\sink$ to be large. However, in practice, models trained on the BB task tend to converge to a stable phase where $\sink$ is relatively large. 


\paragraph{The formation of attention sinks and value-state drains.} Below, we explain how Theorem~\ref{thm:main} reveals the \textit{mutual reinforcement mechanism} behind the formation of attention sinks and value-state drains.  
\begin{itemize}[leftmargin=2em]
\item[(a)] When the value states of the \bos~token are small and constant, $\vecvalue = \ivalue \cdot \bm{1}$, Theorem~\ref{thm:main}(a) shows that the attention logits on the \bos~token $\vecsink(t) \approx \sink(t) \bm{1}$  for $\sink(t) = (1/2) \log t$, grow logarithmically. This demonstrates that the presence of a small constant value state ($\vecvalue = \ivalue \cdot \bm{1}$) reinforces the formation of attention sinks ($\vecsink(t) \approx \sink(t) \cdot \bm{1}$ for $\sink(t)$ increases logarithmically). 
\item[(b)] When the attention logits of the \bos~token are large and constant, $\vecsink = \sink \cdot \bm{1}$ for $\sink \to \infty$, Theorem~\ref{thm:main}(b) shows that the value states of the \bos~token $\vecvalue(t) \to \bar{\ivalue}(0) \cdot \bm{1}$. Starting with a random Gaussian initialization for $\vecvalue(0)$, we have $\|\vecvalue(t)\|_2 \approx \|\bar{\ivalue}(0) \cdot \bm{1}\|_2 \approx \|\vecvalue(0)\|_2 / \sqrt{V}$, where $V$ is the vocabulary size, typically large. This indicates that attention sinks ($\vecsink = \sink \cdot \bm{1}$ for large $\sink$) reinforces the formation of value-state drains ($\vecvalue(t) \to \ivalue \cdot \bm{1}$ for small $\ivalue$). 
\item[(c)] In the later stages of the dynamics, both the attention logits and value states of the \bos~token stabilize, as described in~\ref{thm:main}(c). The attention logits remain constant at $\vecsink = \sink \cdot \bm{1}$ with large $\sink$, while the value states become small, $\vecvalue = [\bar\ivalue(0)+e^{-\sink} \meanvalue] \cdot \bm{1} - e^{-\sink} \cdot \bm{\mass}\circ \bm{\xi}$. 
\end{itemize}


Based on these theoretical insights, we summarize the dynamical mechanism underlying attention sinks and value-state drains: For any attention head given a specific prompt, if the model can accurately predict the next token without using the attention head, but adding any value state from previous tokens—except for certain special tokens—worsens the prediction, the attention head will become dormant, forming an attention sink at those special tokens. This phenomenon is induced by the mutual reinforcement mechanism, as described below: 
\begin{figure}[h]
    \centering
    \begin{minipage}{0.68\textwidth}
    \begin{claim}[Mutual reinforcement mechanism]\label{claim:mutual-reinforcement}
Dynamically, attention sinks and value-state drains arise through mutual reinforcement:
\begin{itemize}[leftmargin=2em]
\setlength\itemsep{0pt}
    \item[\textup{(a)}] The SoftMax mechanism shifts attention weights towards tokens that exhibit value-state drains, reinforcing these tokens as attention sinks. 
    \item[\textup{(b)}] Attention sinks on these extreme tokens further suppress their value states, reinforcing their role as value-state drains. 
    \item[\textup{(c)}] The mutual reinforcement stabilizes when all non-trigger tokens have large, nearly identical attention logits on the extreme token. 
\end{itemize} 
Due to the causal mask, the training dynamics favor the \bos~token as the extreme token. 
\end{claim}
    \end{minipage}
    ~~~~
    \begin{minipage}{0.28\textwidth}
        \includegraphics[width=\linewidth]{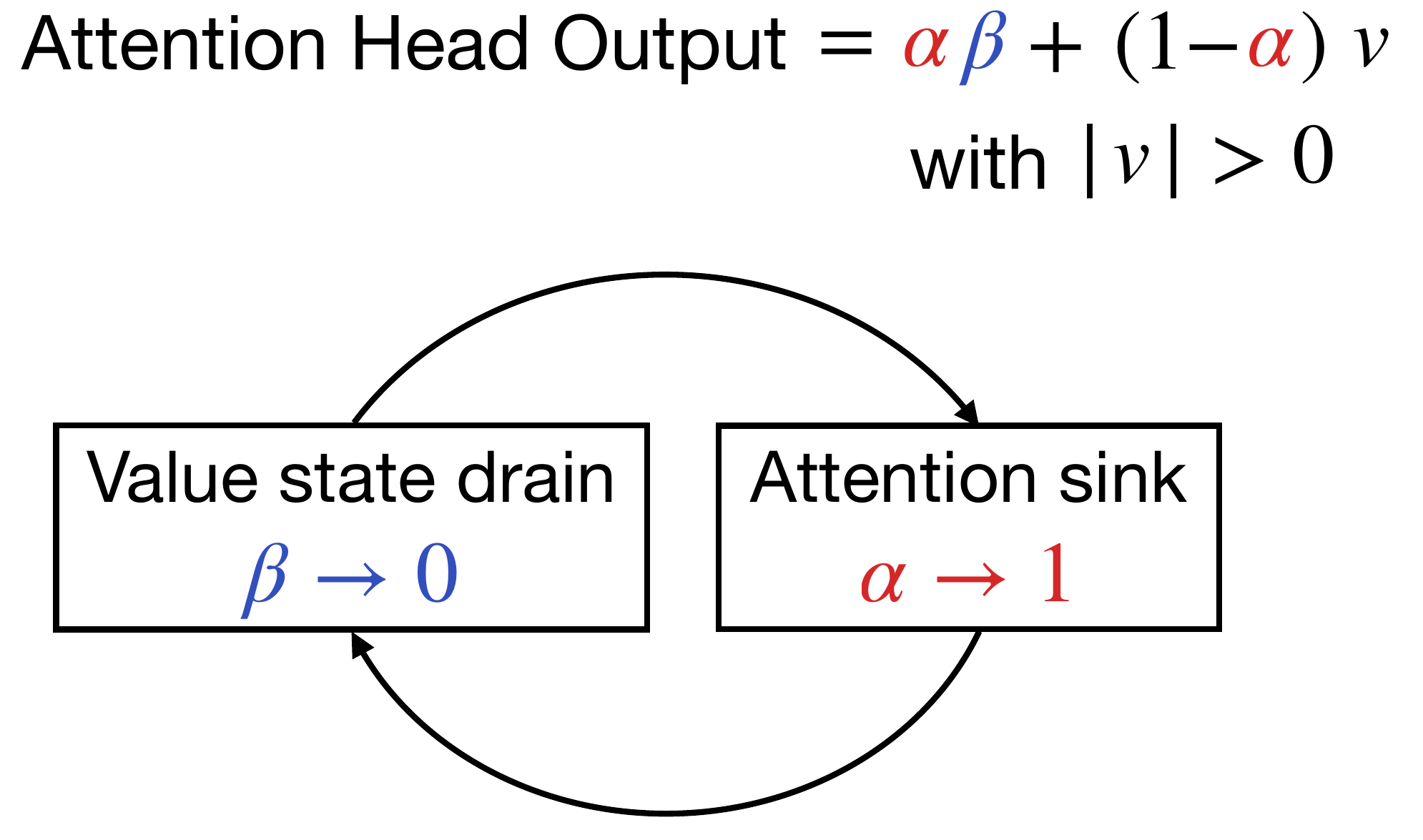}
        \caption{Mutual reinforcement mechanism}
        \label{figure:illustrate-mutual-reinforcement}
    \end{minipage}%
\end{figure}




\paragraph{Experimental verification of the quantitative prediction.} Revisiting Figure~\ref{fig:dynamics}, which illustrates the dynamics of a single-layer transformer model trained with Adam on the BB task, we observe that $\Delta\text{logit}_{\cdot,\bos}$ exhibits growth rates consistent with Theorem~\ref{thm:main}. In this context, $\Delta\text{logit}_{\cdot,\bos}$  corresponds to $\sink$, as all other attention logits are assumed to be zero under the assumptions of Theorem~\ref{thm:main}. When plotted on a logarithmic scale, the $\Delta\text{logit}_{\cdot,\bos}$ curve grows approximately linearly between 1,000 and 10,000 steps, then accelerates before stabilizing around 100,000 steps.  Meanwhile, the norm of the value state $\|\vall_{\bos}\|_2$ decreases monotonically. The simultaneous increase in attention weights and decrease in value-state norms demonstrate the mutual reinforcement mechanism during the training process. 

To further validate that Theorem~\ref{thm:main} accurately captures the dynamics of the original model, we constructed a simplified model based on Eq.~\eqref{eqn:simplification_TF_1}, \eqref{eqn:simplification_TF_2}, and \eqref{eqn:simplification_TF_3}, and trained the parameters $(\vecsink \in \R^{V}, \vecvalue \in \R^V, \bm{\xi} \in \R^V, \lambda \in \R)$ using Adam. The resulting training curves closely resemble those of the one-layer transformer, also displaying the mutual reinforcement mechanism. A detailed description of the experiment can be found in Appendix~\ref{appsec:train-simple}.


\paragraph{Generality of the theoretical prediction.} Although Theorem~\ref{thm:main} focuses on a specific BB task with a simplified architecture and loss function, the underlying principles are broadly applicable to more general settings. In particular, we expect that the formation of extreme tokens in LLMs follows a similar mutual reinforcement mechanism. Indeed, Theorem~\ref{thm:main} is essentially based on the following two key assumptions: (1) even with a specific attention head $\attn$ zeroed out, the LLM can still accurately predict the next token, implying that the attention head is better off dormant; and (2) for the attention head $\attn$, value states of previous tokens—except for certain special tokens—remain relevant for specific tasks and therefore do not vanish. Under these assumptions, we anticipate the formation of attention sinks and value-state drains for the attention head $\attn$ and such special tokens. In Section~\ref{sec:llm}, we explore how these phenomena are formed during the training dynamics of LLMs, finding that the empirical results align with the theory. 


\begin{figure}
  \centering
    \begin{minipage}{0.3\textwidth}
      \centering
      \subcaption{\small ReLU attention}
      \label{fig:relu-attn}
      \includegraphics[width=\textwidth]{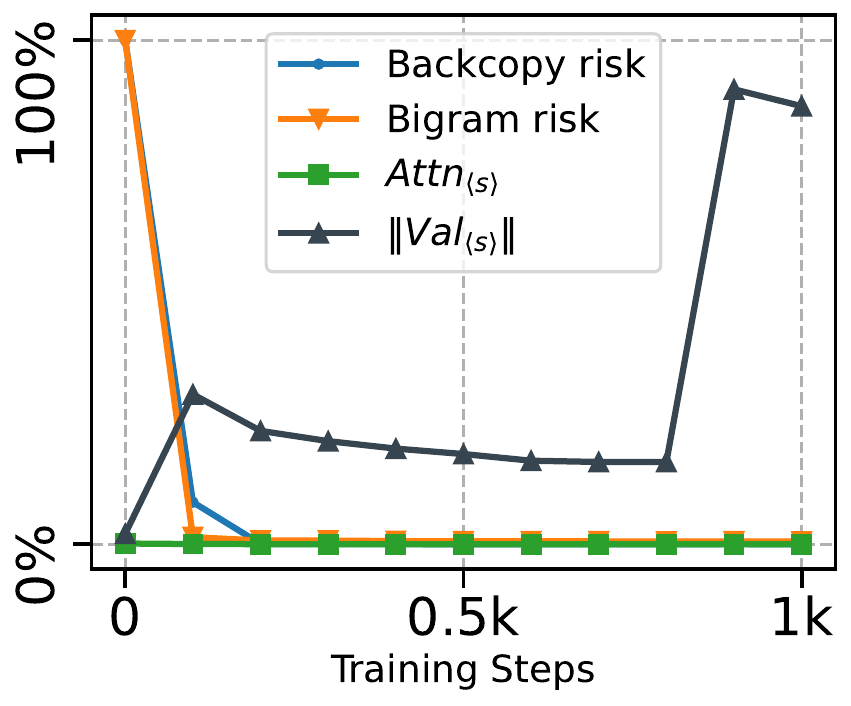}
  \end{minipage}
    \begin{minipage}{0.33\textwidth}
      \centering
      \subcaption{\small Interventions on a 3-layer TF}
      \label{fig:massive-interventions}
      \includegraphics[width=\textwidth]{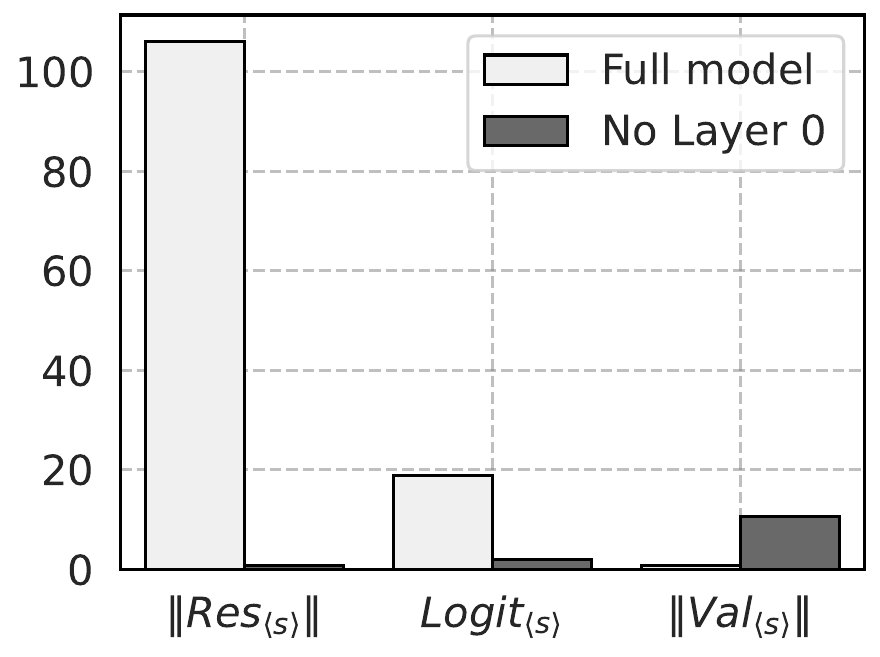}
  \end{minipage}
  \begin{minipage}{0.33\textwidth}
      \centering
      \subcaption{\small Eliminating
      residual-state peaks}
      \label{fig:sgd}
      \includegraphics[width=\textwidth]{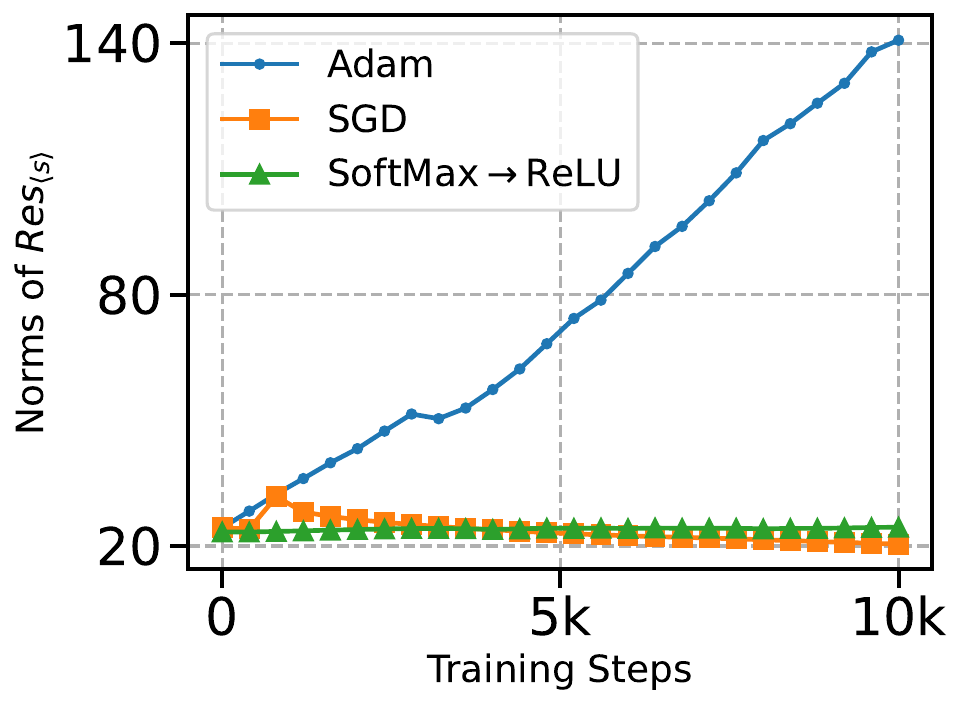}
  \end{minipage}
    \caption{\small 
    \textit{Left (a)}: The training dynamics of the single-layer ReLU attention transformer on the BB task.
    \textit{Middle (b)}: The intervention results on the \attn+\mlp+\attn+\mlp+\mlp~architecture. The attention sink and value-state peak of the middle $\attn$ layer disappear after zeroing out $\attn+\mlp$ of layer 0. 
    \textit{Right (c)}: The evolution of massive norms in a three-layer transformer trained with Adam, SGD, and using a ReLU attention transformer. Notably, only the three-layer model with Softax attention trained using Adam results in the formation of residual-state peaks.}
\end{figure}

\paragraph{Replacing SoftMax by ReLU attention removes attention sinks and value-state drains.} As a consequence of our theory, we predict that training using ReLU attention in place of SoftMax attention will prevent the mutual reinforcement mechanism. Without SoftMax, the training dynamics no longer push the attention weights toward the \bos~token, which remains zero throughout training. In the absence of attention sinks, the dynamics no longer push down the value state norm, and the mutual reinforcement mechanism breaks. Figure~\ref{fig:relu-attn} presents the training dynamics on the BB task using ReLU instead of SoftMax attention, showing that both the Bigram and Backcopy risk converge to the Bayes risk after 200 training steps, but the attention logits of \bos~do not increase, and the value state does not shrink, confirming our prediction.

\subsection{The emergence of residual-state peaks}
\label{sec:res-peak}
In this section, we experimentally investigate the residual-state peaks phenomenon. We observe that no residual-state peaks occur in the single-layer transformer trained on the BB task. To explore this further, we train slightly deeper transformers on the BB task and track the residual state norm after layer $0$. We observe that two-layer models do not exhibit residual-state peaks, while models with three or more layers do.
Additional experimental results are provided in Appendix~\ref{appsec:mini-res-peak} and \ref{appsec:three-layer-tf}. 

\paragraph{Massive residual state at layer 0 output induces attention sinks and value-state drains in the middle layer.} To investigate the relationship between massive residual states and attention sinks, we train on the BB task using the ``\attn+\mlp+\attn+\mlp+\mlp'' model, which is the minimal structure that shows the massive residual states phenomena. We perform intervention by analyzing how the model's behavior changes after zeroing out layer 0 (the first ``\attn+\mlp'' block). Before and after zeroing, we compute the difference in $\|\res_\bos\|$ and $\text{Mean}_\tok[\|\res_\tok\|]$ at the layer 0 output, and compute $\text{logit}_{\cdot,\bos}$ and $\|\val_{\bos}\|$ in the middle layer. After zeroing out, the residual state norm becomes non-massive, and attention logits and the value state norm return to a normal level. This confirms that the residual-state peak contributes to the attention sink and value-state-drain phenomena in the middle layer of pre-trained transformers. 



\paragraph{Linear growth of residual-state norm with Adam training.} Figure~\ref{fig:sgd} shows the residual-state norms of the \bos~token at the layer 0 output of three-layer transformers during pre-training on the BB task. The results indicate that training the transformer with Adam leads to a linear increase in residual state norms.

\paragraph{Switching from Adam to SGD and switching from SoftMax to ReLU attention eliminates the residual-state peaks.} Figure~\ref{fig:sgd} also illustrates the dynamics of residual-state norms in other training setups. When switching the training algorithm from Adam to SGD, attention sinks remain, but residual-state peaks disappear. Similarly, switching to ReLU attention, which lacks the mutual reinforcement mechanism, also eliminates residual-state peaks. These findings highlight the dependence of residual-state peaks on SoftMax attention and the Adam optimization algorithm. We propose a potential explanation of this phenomenon in Appendix~\ref{appsec:theory-for-res}.

\section{Extreme-token Phenomena in pretrained LLMs} \label{sec:llm}


In this section, we investigate extreme-token phenomena in open-source pretrained LLMs. In \Cref{sub:active_dormant}, we analyze the static behavior of these phenomena in Llama 2-7B-Base \citep{touvron2023llama}, confirming the existence of the \textit{active-dormant mechanism} in LLMs. Notably, we identify a specific head that is active on GitHub samples but dormant on Wikipedia samples. In \Cref{sub:olmo_dynamics}, we examine the dynamic behavior of extreme-token phenomena during the pretraining of OLMo-7B \citep{groeneveld2024olmo}. We show that the attention logits, value states norm, and residual states norm of the sink token(s) in OLMo reflect behavior similar to that of the simpler BB model. Specifically, the simultaneous formation of attention sinks and value-state drains gives evidence for the \textit{mutual reinforcement mechanism}.

\subsection{Active-dormant mechanism in LLMs}\label{sub:active_dormant}

\begin{figure}
    \centering
    \begin{subfigure}[t]{0.58\textwidth}
        \centering
        \caption{\small Attention weights for GitHub/Wikipedia data}
        \includegraphics[width=0.9\textwidth]{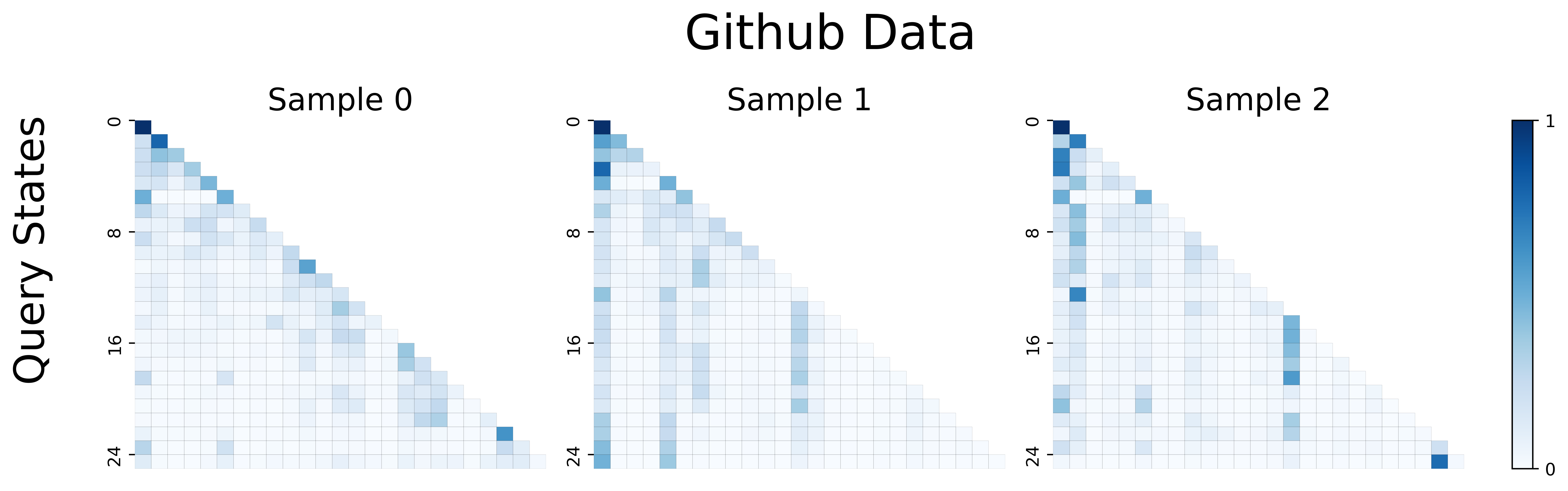}
        
        \includegraphics[width=0.9\textwidth]{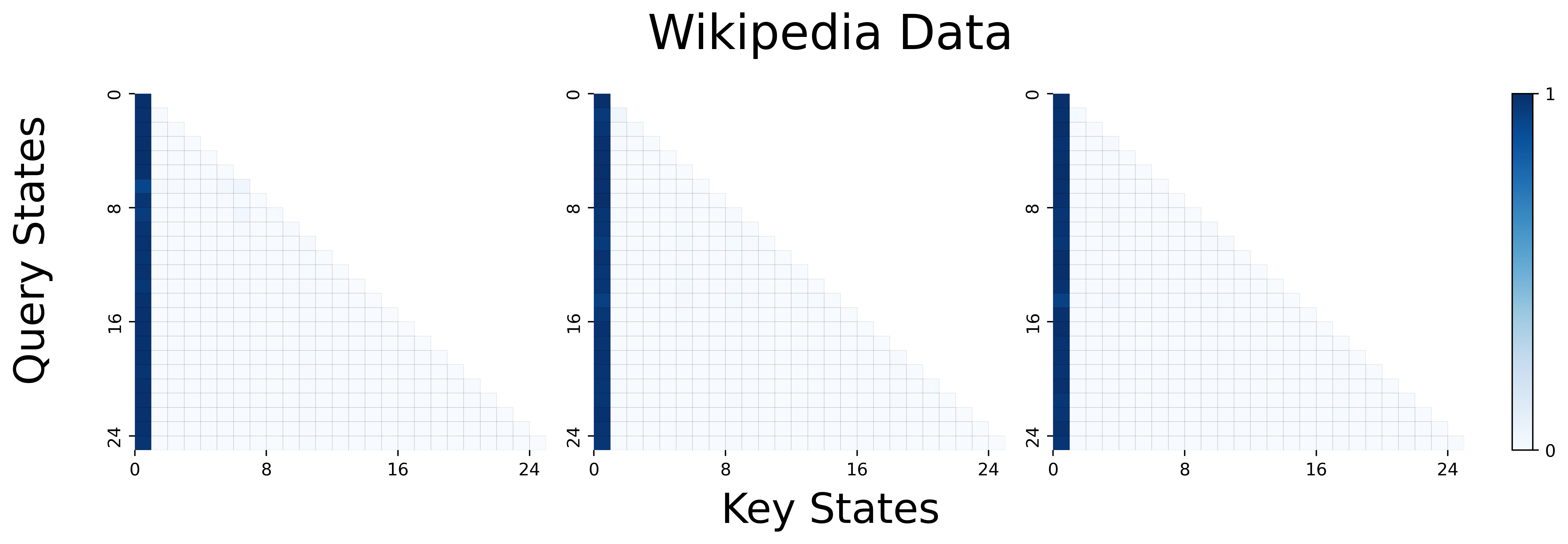}
        \label{fig:github_wikipedia_weights}
    \end{subfigure}
    \hfill
    \begin{subfigure}[t]{0.38\textwidth}
        \caption{\small Zero-out-head intervention outcomes}
        \label{fig:github_wikipedia_zero_out}
        \vskip1.5em
        \includegraphics[width=0.9\textwidth]{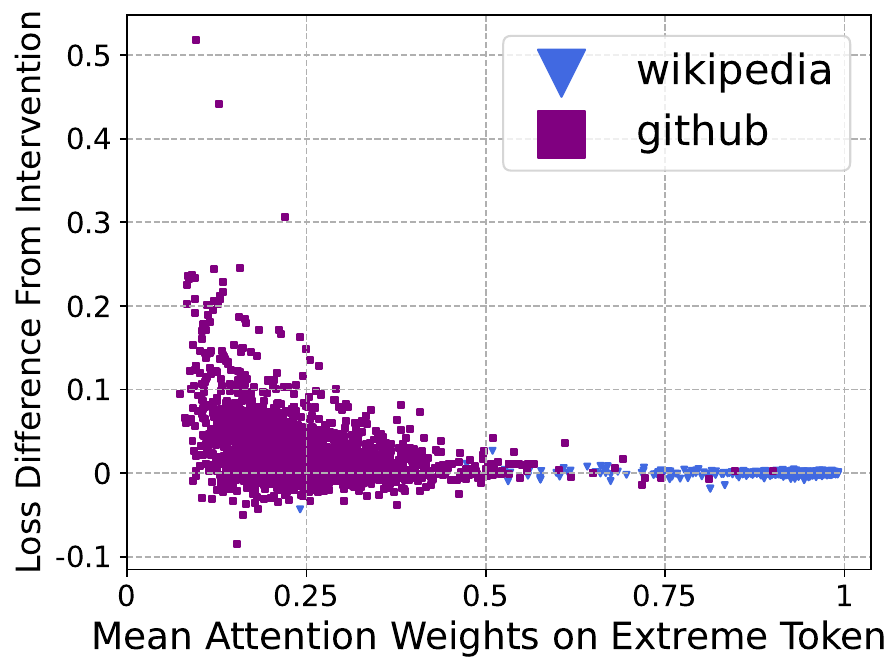}
    \end{subfigure}
    \vspace{-0.5em}
    \caption{\small \textbf{Active-dormant mechanism of Layer 16 Head 25 (L16H25) of Llama 2-7B-Base.} We observe that L16H25 is active on GitHub data and dormant on Wikipedia data, both sourced from RedPajama-1T \citep{together2023redpajama}. \textit{Left (a)}: Attention weights of L16H25, prompted by three randomly selected samples from each domain. \textit{Right (b)}: Results of an intervention study showing the change in cross-entropy loss when the output of L16H25 (specifically, its value states) is set to zero across sequences in both domains. The findings indicate that the model's performance for GitHub data, measured by cross-entropy loss, strongly relies on the output of this attention head. 
    }
    \label{fig:dormant_heads_domain_dependent}
\end{figure}


Our study of the BB model leads to the following prediction with respect to the extreme-token phenomena, which we hypothesize also applies to LLMs:  
\begin{center}
    \textit{Attention heads are controlled by an active-dormant mechanism (cf.\ Claim \ref{claim:active-dormant}). The presence of attention sinks and value-state drains indicates that an attention head is in a dormant phase.}
\end{center}

This hypothesis suggests that in LLMs, whether an attention head becomes a sink depends on the context. Specifically, the attention head may become entirely irrelevant for selecting the next tokens in certain contexts or tasks, but not in others. When this irrelevance occurs, the attention head transitions into an attention sink. This hypothesis was confirmed in small transformers and the BB task, as demonstrated in Section~\ref{sec:bb_task}.

Accordingly, we aim to identify instances of attention heads in pretrained LLMs that exhibit this active-dormant behavior, i.e., heads that are dormant in some domains but active in others. In \Cref{fig:dormant_heads_domain_dependent}, we display a particular attention head---Layer 16 Head 25 (L16H25) of Llama 2-7B-Base \citep{touvron2023llama}---which demonstrates a clear active-dormant distinction across two distinct contexts (e.g., tokens from the GitHub subset versus the Wikipedia subset of RedPajama \citep{together2023redpajama}). While many attention heads show similar context-dependent behavior (see \Cref{sec:more_heads}), we focus on this one because the conditions for its activation are straightforward and interpretable, whereas other heads may have more nuanced criteria. 

\Cref{fig:github_wikipedia_weights} shows the attention maps of L16H25 on samples from both the GitHub and Wikipedia subsets of RedPajama. It demonstrates that L16H26 is \textit{dormant} (i.e., an attention sink) on samples from Wikipedia, which resemble prose, and \textit{active} (i.e., not an attention sink) on samples from GitHub, which resemble code. Additionally, \Cref{fig:github_wikipedia_zero_out} compares the loss difference when L16H25 is zeroed out for prompts from both domains. The results show that zeroing out this head significantly decreases model performance on GitHub sequences, while having minimal impact on Wikipedia sequences. This observation also confirms the head behaves as dormant in some contexts and active in others---in some contexts, removing this head has no effect on model performance, while in others, its removal causes significant performance drops.

\subsection{Extreme-token phenomena along training dynamics of LLMs}\label{sub:olmo_dynamics}

\begin{figure}[t]
    \centering
    \begin{subfigure}[t]{0.32\textwidth}
        \centering 
        \caption{\small Attention sink dynamics}
        \includegraphics[width=0.9\textwidth]{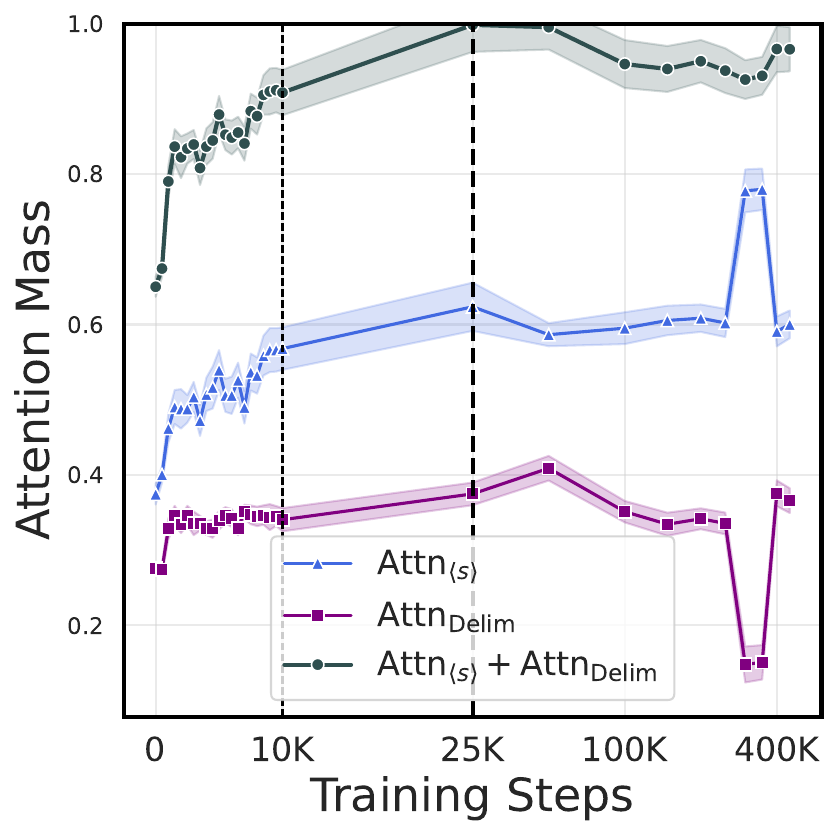}
        \label{fig:olmo_sink}
    \end{subfigure}
    \begin{subfigure}[t]{0.32\textwidth}
        \centering 
        \caption{\small Value state dynamics}
        \includegraphics[width=0.9\textwidth]{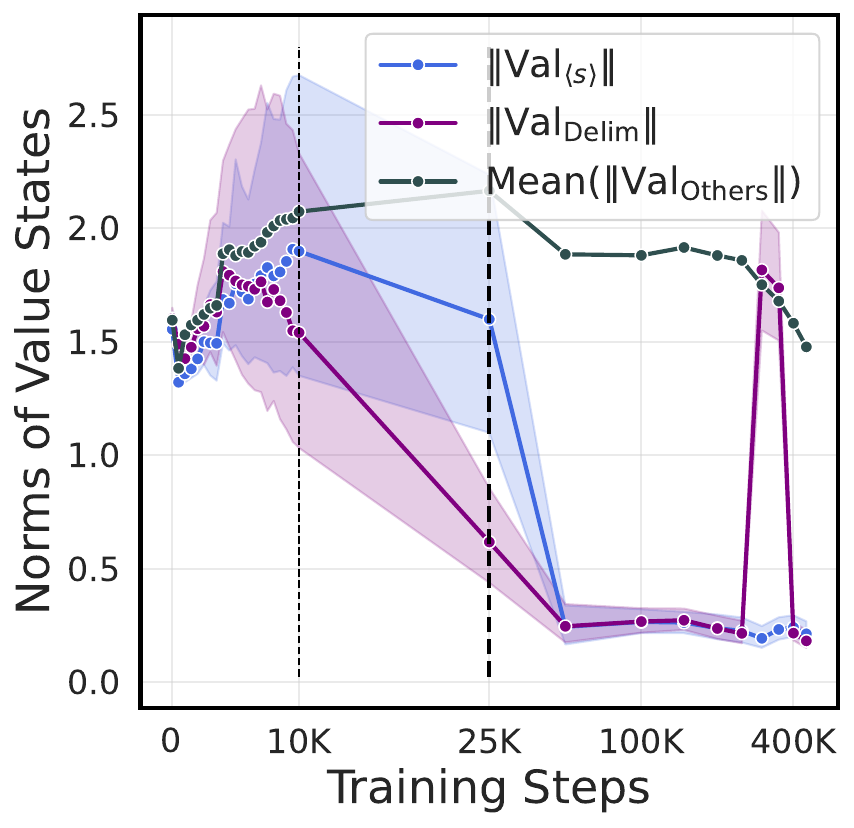}
        \label{fig:olmo_drain}
    \end{subfigure}
    \begin{subfigure}[t]{0.32\textwidth}
        \centering 
        \caption{\small Residual state dynamics}
        \includegraphics[width=0.9\textwidth]{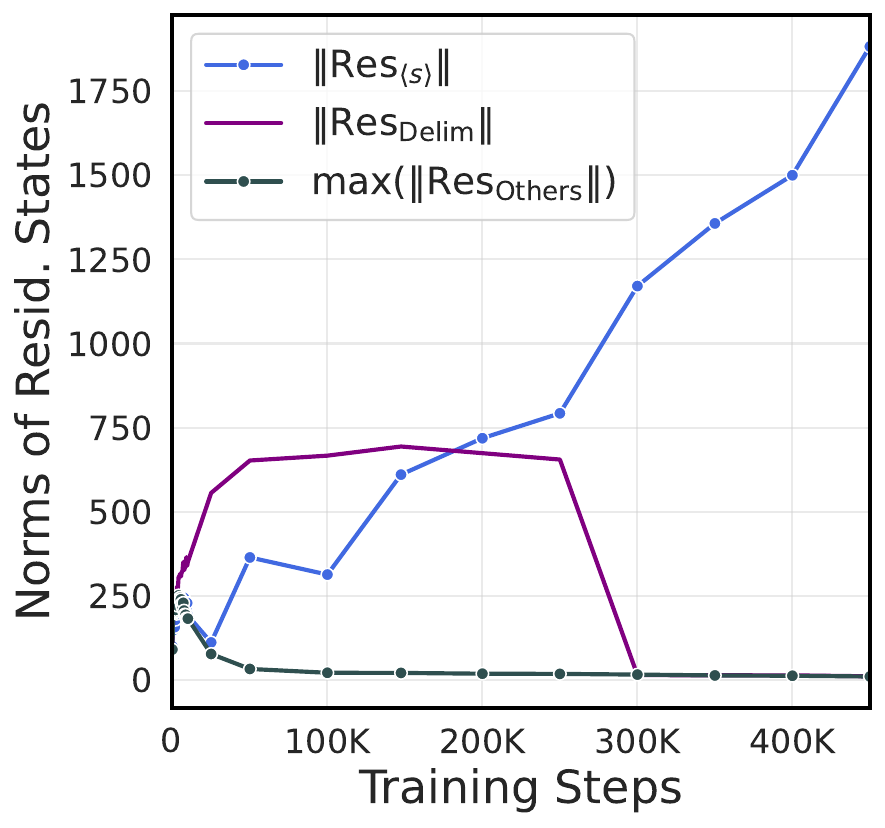}
        \label{fig:olmo_peak}
    \end{subfigure}
    \caption{\small \textbf{Attention weights, value state norms, and residual state norms of Layer 24 during the training dynamics of OLMo.} \textit{Left (a)}: The total attention mass on extreme tokens \bos~and ``\text{Delim}''(\period) at Layer 24, averaged across all attention heads. The horizontal axis is logarithmically scaled after step $10$k. We observe a rapid increase followed by stabilization within the range \([0.9, 1]\) for the rest of training, consistent with our predictions. \textit{Middle (b)}: The value state norms of each token at Layer 24 during training, averaged over all heads. The horizontal axis is logarithmically scaled after step $10k$. Initially, the value states of all tokens shrink, eventually converging, while the value states of the extreme tokens shrink to significantly lower levels compared to other tokens. Figure \textit{(a)} and \textit{(b)} coincide with the trends in Figure~\ref{fig:dynamics} under the BB task. \textit{Right (c)}: The residual state norms of each token at Layer 24 during training. The residual state norm of \bos~increases linearly in magnitude throughout training, matching Figure~\ref{fig:sgd} in the BB task.}

    \label{fig:olmo_predictions_phase0}
\end{figure}
\begin{figure}[h]
    \centering
    \hfill
    \begin{subfigure}[t]{0.32\textwidth}
        \centering 
        \caption{\small Logit dynamics}
        \includegraphics[width=0.9\textwidth]{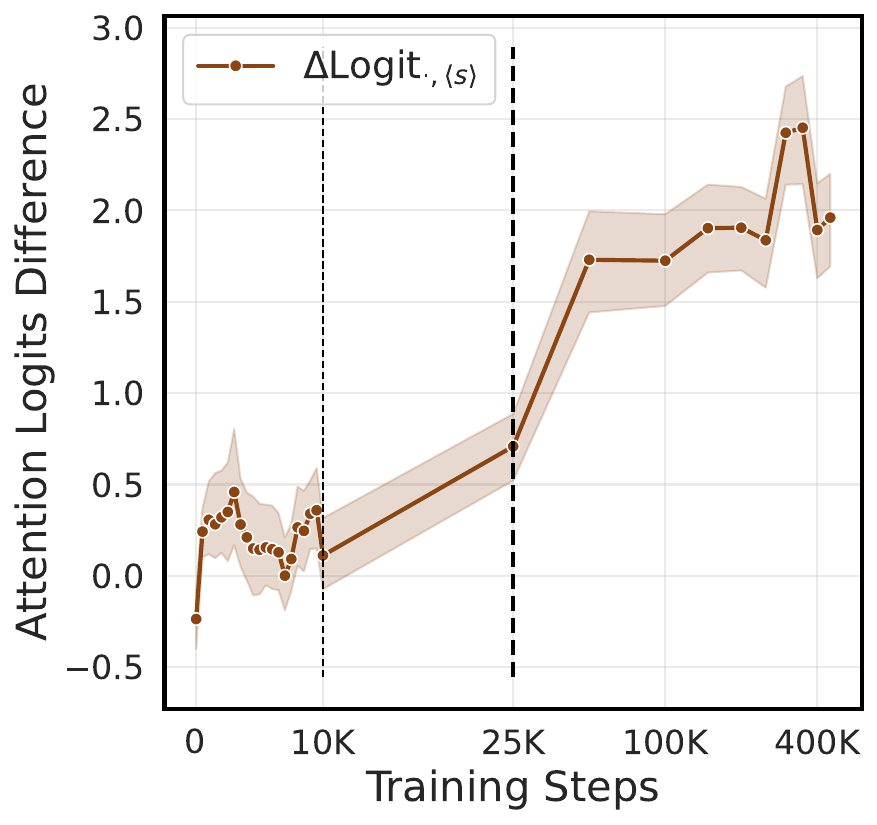}
        \label{fig:attention_logits_olmo_dynamic}
    \end{subfigure}
    \hfill
    \begin{subfigure}[t]{0.32\textwidth}
        \centering 
        \caption{\small Sink-logits concentration}
        \includegraphics[width=0.9\textwidth]{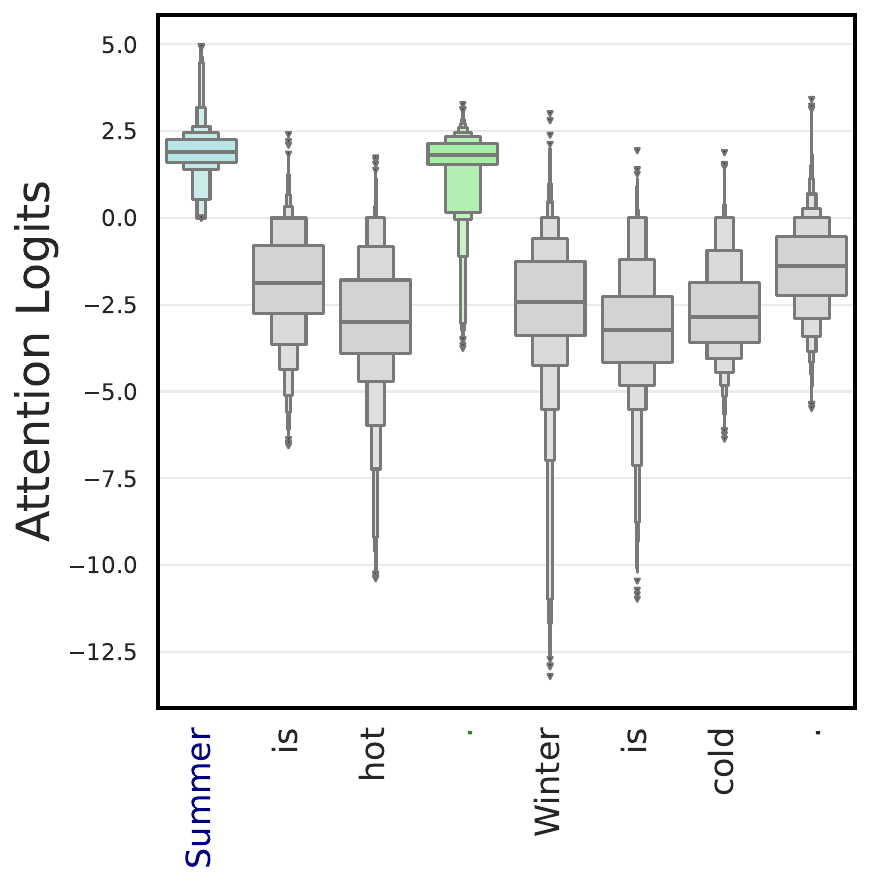}
        \label{fig:attention_logits_olmo_static}
    \end{subfigure}
    \hfill
    \phantom{.}
        \caption{\small \textbf{Attention logits of Layer 24.}  \textit{Left (a)}: Attention logits difference of all tokens' query states against \bos's key state during training. The difference in attention logits is computed as \(\Delta\mathrm{logit}_{\cdot,\bos} = \query_{\cdot}^\top \key_{\bos} - \text{Mean}[\query_{\cdot}^\top \key_{\text{Others}}]\). The horizontal axis is logarithmically scaled after step $10k$. We observe that $\Delta \mathrm{logit}_{\cdot,\bos}$ increases approximately in logarithmic scale during training steps $10$k to $100$k, matching the decreasing phase of the value states in Figure~\ref{fig:olmo_drain}. 
        \textit{Right (b)}: Attention logits of the last token's query state against all token's key states for pretrained OLMo. In this experiment, we generate \(128\) randomly sampled test tokens with IDs from \(100\) to \(50000\) in the OLMo tokenizer. We append each token separately to the test phrase ``Summer is warm\period~Winter is cold\period'', creating \(128\) different samples, which we feed to the LLM to examine the model behavior. We plot the distribution of (un-shifted) attention logits \( \text{logit}_{\cdot,\tok}=\query_{\mathrm{test}}^\top \key_{\tok}\) across all heads at Layer 24 and all test tokens. The distribution of $\text{logit}_{\cdot,\bos}$ and $\text{logit}_{\cdot,\text{Delim}}$ have considerably small variance compared with other logits, confirming the sink-logits concentration phenomenon. }
    \label{fig:olmo_predictions_phase1}
\end{figure}

Our study of the BB model leads to the following prediction about the dynamical behavior of the extreme-token phenomena, which we hypothesize also applies to LLMs:  
\begin{center}
    \textit{Attention heads undergo an attention-increasing and value-state-shrinking phase driven by the mutual reinforcement mechanism (cf.\ Claim~\ref{claim:mutual-reinforcement}). This is followed by a stable phase, where all non-trigger tokens have large, nearly identical attention logits on the extreme token. Simultaneously, the residual state norms of the extreme tokens increase linearly during pretraining.}
\end{center}

We confirm these predictions below. To observe the training dynamics of a large-scale LLM, we use the setup of OLMo-7B-0424 \citep{groeneveld2024olmo} (henceforth just referred to as OLMo), which provides open-sourced weights at various stages of their training.\footnote{We did not analyze Llama for dynamics, as they do not provide open-source intermediate checkpoints along pretraining.} For our analysis, we inspect OLMo at multiple checkpoints: every 500 steps for the first 10,000 steps, then at 25,000 steps, 50,000 steps, and every 50,000 steps up to 449,000 steps (approximately the end of their training).\footnote{For the single 150,000-step checkpoint, we observed that its statistics were outliers, which we hypothesize is due to a system failure. We address this by using the average of nearby checkpoints to represent its statistics.} The input we use for this analysis is again ``Summer is warm\period~Winter is cold\period''\footnote{Note that OLMo does not have a \bos~token, but attention sinks still form in the majority of heads. In particular, the first token always behaves as an attention sink. We discuss this further in \Cref{sub:fixed_bos}.} In this prompt, the ``$\mathrm{Delim}$'' token, namely ``\period'', also becomes a sink token along with \bos. We believe this occurs because the period is not semantically meaningful and is not useful for predicting future tokens (cf.\  \Cref{sub:fixed_bos}) 

\Cref{fig:olmo_predictions_phase0} illustrates the dynamics of attention weights, value state norms, and the residual state norms for attention heads in Layer 24 of OLMo. The figure shows that the average attention on extreme tokens (\bos~and $\mathrm{Delim}$) increases rapidly at the beginning of training before stablizing, while the value state norms of these extreme tokens decrease during training steps 10k-100k. The synchronized evolution of attention weights and value state norms aligns with the prediction of the mutual reinforcement mechanism.  Additionally, the residual states of \bos~increase linearly, while those of other tokens converge to a small number. \Cref{fig:olmo_predictions_phase1} provides a more detailed examination of the attention logits in Layer 24 of OLMo. Figure~\ref{fig:attention_logits_olmo_dynamic} presents the dynamics of the difference in attention logits, showing that $\Delta \text{logit}_{\cdot,\bos}$ increase during training steps 10k-100k, matching the decreasing phase of the value states.
Figure~\ref{fig:attention_logits_olmo_static} also demonstrates the \textit{sink-logits concentration} phenomenon. Specifically, it shows that the sink logits will eventually converge to a stable phase, in which logits corresponding to the key of the sink token and queries of all non-sink tokens are nearly identical. These findings coincide with the dynamical behavior predicted by the BB model, as outlined in Theorem~\ref{thm:main}(c) and corroborated by the experimental results in \Cref{figure:verify-assumptions}. 

\section{Conclusions} \label{sec:conclusion}

In this work, we investigated the \textit{extreme-token phenomena}, specifically \textit{attention sinks}, \textit{value-state drains}, and \textit{residual-state peaks}. We analyzed simple transformers trained on the Bigram-Backcopy (BB) task, both theoretically and empirically, demonstrating that these models exhibit the same extreme-token phenomena observed in large language models (LLMs). Building on the insights from the BB task, we made several detailed predictions about the behavior of extreme-token phenomena in LLMs. In particular, we identified the \textit{active-dormant mechanism} governing attention heads in both the BB model and LLMs, with attention sinks and value-state drains serving as indicators of dormant phase, and a \textit{mutual reinforcement mechanism} that induces these phenomena during pretraining. Using insights from these mechanisms, we applied simple modifications to the model architecture and optimization procedure, effectively mitigating the extreme-token phenomena in the BB model. Overall, our work uncovers the underlying mechanisms of extreme-token phenomena and suggests potential pathways to mitigate these issues during LLM pretraining.

We believe the most compelling direction for future work is to explore whether eliminating the extreme-token phenomena is essential or beneficial for building powerful transformer-based LLMs. While it is possible to mitigate these phenomena through simple modifications to the architecture or training algorithms, it remains unclear whether their elimination significantly improves downstream tasks such as inference and quantization. Given the resource-intensive nature of pretraining large-scale LLMs, we anticipate that pretraining a model at the scale of GPT-2 could both provide valuable insight into this issue and help point the way to architectures that can reduce the pretraining burden. 

\section*{Acknowledgements}
TG thanks Yaodong Yu, Licong Lin, and Ruiqi Zhang for insightful discussions.
YB thanks Caiming Xiong and Huan Wang for the many insightful discussions in the early stages of this work.
This project is supported by NSF DMS-2210827, CCF-2315725, CAREER DMS-2339904, ONR N00014-24-S-B001, a UC Berkeley College of Engineering fellowship, an Amazon Research Award, a Google Research Scholar Award, an Okawa Foundation Research Grant, and the European Union (ERC-2022-SYG-OCEAN-101071601).



\bibliographystyle{plainnat}
\bibliography{main_arxiv.bbl}

\begin{thebibliography}{72}
\providecommand{\natexlab}[1]{#1}
\providecommand{\url}[1]{\texttt{#1}}
\expandafter\ifx\csname urlstyle\endcsname\relax
  \providecommand{\doi}[1]{doi: #1}\else
  \providecommand{\doi}{doi: \begingroup \urlstyle{rm}\Url}\fi

\bibitem[Ahn et~al.(2023)Ahn, Cheng, Song, Yun, Jadbabaie, and
  Sra]{ahn2023linear}
Kwangjun Ahn, Xiang Cheng, Minhak Song, Chulhee Yun, Ali Jadbabaie, and Suvrit
  Sra.
\newblock Linear attention is (maybe) all you need (to understand transformer
  optimization).
\newblock \emph{arXiv preprint arXiv:2310.01082}, 2023.

\bibitem[Ahn et~al.(2024)Ahn, Cheng, Daneshmand, and Sra]{ahn2024transformers}
Kwangjun Ahn, Xiang Cheng, Hadi Daneshmand, and Suvrit Sra.
\newblock Transformers learn to implement preconditioned gradient descent for
  in-context learning.
\newblock \emph{Advances in Neural Information Processing Systems}, 36, 2024.

\bibitem[Allen-Zhu and Li(2023)]{allen2023physics}
Zeyuan Allen-Zhu and Yuanzhi Li.
\newblock Physics of language models: {P}art 1, context-free grammar.
\newblock \emph{arXiv preprint arXiv:2305.13673}, 2023.

\bibitem[Biderman et~al.(2023)Biderman, Schoelkopf, Anthony, Bradley,
  O’Brien, Hallahan, Khan, Purohit, Prashanth, Raff,
  et~al.]{biderman2023pythia}
Stella Biderman, Hailey Schoelkopf, Quentin~Gregory Anthony, Herbie Bradley,
  Kyle O’Brien, Eric Hallahan, Mohammad~Aflah Khan, Shivanshu Purohit,
  USVSN~Sai Prashanth, Edward Raff, et~al.
\newblock Pythia: A suite for analyzing large language models across training
  and scaling.
\newblock In \emph{International Conference on Machine Learning}, pages
  2397--2430. PMLR, 2023.

\bibitem[Bietti et~al.(2024)Bietti, Cabannes, Bouchacourt, Jegou, and
  Bottou]{bietti2024birth}
Alberto Bietti, Vivien Cabannes, Diane Bouchacourt, Herve Jegou, and Leon
  Bottou.
\newblock Birth of a transformer: A memory viewpoint.
\newblock \emph{Advances in Neural Information Processing Systems}, 36, 2024.

\bibitem[Bondarenko et~al.(2021)Bondarenko, Nagel, and
  Blankevoort]{bondarenko2021understanding}
Yelysei Bondarenko, Markus Nagel, and Tijmen Blankevoort.
\newblock Understanding and overcoming the challenges of efficient transformer
  quantization.
\newblock \emph{arXiv preprint arXiv:2109.12948}, 2021.

\bibitem[Bondarenko et~al.(2023)Bondarenko, Nagel, and
  Blankevoort]{bondarenko2023quantizable}
Yelysei Bondarenko, Markus Nagel, and Tijmen Blankevoort.
\newblock Quantizable transformers: Removing outliers by helping attention
  heads do nothing.
\newblock \emph{Advances in Neural Information Processing Systems},
  36:\penalty0 75067--75096, 2023.

\bibitem[Charton(2022)]{charton2022my}
Fran{\c{c}}ois Charton.
\newblock What is my math transformer doing? {T}hree results on
  interpretability and generalization.
\newblock \emph{arXiv preprint arXiv:2211.00170}, 2022.

\bibitem[Chen et~al.(2024)Chen, Zhao, Liu, Bai, Lin, Zhou, and
  Chang]{chen2024image}
Liang Chen, Haozhe Zhao, Tianyu Liu, Shuai Bai, Junyang Lin, Chang Zhou, and
  Baobao Chang.
\newblock An image is worth 1/2 tokens after layer 2: Plug-and-play inference
  acceleration for large vision-language models.
\newblock \emph{arXiv preprint arXiv:2403.06764}, 2024.

\bibitem[Computer(2023)]{together2023redpajama}
Together Computer.
\newblock Red{P}ajama: An open source recipe to reproduce {Llama} training
  dataset, 2023.
\newblock URL \url{https://github.com/togethercomputer/RedPajama-Data}.

\bibitem[Darcet et~al.(2023)Darcet, Oquab, Mairal, and
  Bojanowski]{darcet2023vision}
Timoth{\'e}e Darcet, Maxime Oquab, Julien Mairal, and Piotr Bojanowski.
\newblock Vision transformers need registers.
\newblock \emph{arXiv preprint arXiv:2309.16588}, 2023.

\bibitem[Deora et~al.(2023)Deora, Ghaderi, Taheri, and
  Thrampoulidis]{deora2023optimization}
Puneesh Deora, Rouzbeh Ghaderi, Hossein Taheri, and Christos Thrampoulidis.
\newblock On the optimization and generalization of multi-head attention.
\newblock \emph{arXiv preprint arXiv:2310.12680}, 2023.

\bibitem[Dettmers and Zettlemoyer(2023)]{dettmers2023case}
Tim Dettmers and Luke Zettlemoyer.
\newblock The case for 4-bit precision: k-bit inference scaling laws.
\newblock In \emph{International Conference on Machine Learning}, pages
  7750--7774. PMLR, 2023.

\bibitem[Dettmers et~al.(2022)Dettmers, Lewis, Belkada, and
  Zettlemoyer]{dettmers2022gpt3}
Tim Dettmers, Mike Lewis, Younes Belkada, and Luke Zettlemoyer.
\newblock {LLM.int8()}: 8-bit matrix multiplication for transformers at scale.
\newblock \emph{Advances in Neural Information Processing Systems},
  35:\penalty0 30318--30332, 2022.

\bibitem[Dubey et~al.(2024)Dubey, Jauhri, Pandey, Kadian, Al-Dahle, Letman,
  Mathur, Schelten, Yang, Fan, et~al.]{dubey2024llama}
Abhimanyu Dubey, Abhinav Jauhri, Abhinav Pandey, Abhishek Kadian, Ahmad
  Al-Dahle, Aiesha Letman, Akhil Mathur, Alan Schelten, Amy Yang, Angela Fan,
  et~al.
\newblock The {Llama} 3 herd of models.
\newblock \emph{arXiv preprint arXiv:2407.21783}, 2024.

\bibitem[Elhage et~al.(2021)Elhage, Nanda, Olsson, Henighan, Joseph, Mann,
  Askell, Bai, Chen, Conerly, et~al.]{elhage2021mathematical}
Nelson Elhage, Neel Nanda, Catherine Olsson, Tom Henighan, Nicholas Joseph, Ben
  Mann, Amanda Askell, Yuntao Bai, Anna Chen, Tom Conerly, et~al.
\newblock A mathematical framework for transformer circuits.
\newblock \emph{Transformer Circuits Thread}, 1:\penalty0 1, 2021.

\bibitem[Elhage et~al.(2023)Elhage, Lasenby, and Olah]{elhage2023privileged}
Nelson Elhage, Robert Lasenby, and Christopher Olah.
\newblock Privileged bases in the transformer residual stream.
\newblock \emph{Transformer Circuits Thread}, 2023.

\bibitem[Fan et~al.(2020)Fan, Stock, Graham, Grave, Gribonval, Jegou, and
  Joulin]{fan2020training}
Angela Fan, Pierre Stock, Benjamin Graham, Edouard Grave, R{\'e}mi Gribonval,
  Herve Jegou, and Armand Joulin.
\newblock Training with quantization noise for extreme model compression.
\newblock \emph{arXiv preprint arXiv:2004.07320}, 2020.

\bibitem[Feng and Steinhardt(2023)]{feng2023language}
Jiahai Feng and Jacob Steinhardt.
\newblock How do language models bind entities in context?
\newblock \emph{arXiv preprint arXiv:2310.17191}, 2023.

\bibitem[Fu(2024)]{fu2024attentionpattern}
Yao Fu.
\newblock How do language models put attention weights over long context?
\newblock \emph{Yao Fu’s Notion}, 2024.
\newblock URL
  \url{https://yaofu.notion.site/How-Do-Language-Models-put-Attention-Weights-over-Long-Context-10250219d5ce42e8b465087c383a034e?pvs=4}.

\bibitem[Geva et~al.(2023)Geva, Bastings, Filippova, and
  Globerson]{geva2023dissecting}
Mor Geva, Jasmijn Bastings, Katja Filippova, and Amir Globerson.
\newblock Dissecting recall of factual associations in auto-regressive language
  models.
\newblock \emph{arXiv preprint arXiv:2304.14767}, 2023.

\bibitem[Gholami et~al.(2022)Gholami, Kim, Dong, Yao, Mahoney, and
  Keutzer]{gholami2022survey}
Amir Gholami, Sehoon Kim, Zhen Dong, Zhewei Yao, Michael~W Mahoney, and Kurt
  Keutzer.
\newblock A survey of quantization methods for efficient neural network
  inference.
\newblock In \emph{Low-Power Computer Vision}, pages 291--326. Chapman and
  Hall/CRC, 2022.

\bibitem[Groeneveld et~al.(2024)Groeneveld, Beltagy, Walsh, Bhagia, Kinney,
  Tafjord, Jha, Ivison, Magnusson, Wang, et~al.]{groeneveld2024olmo}
Dirk Groeneveld, Iz~Beltagy, Pete Walsh, Akshita Bhagia, Rodney Kinney, Oyvind
  Tafjord, Ananya~Harsh Jha, Hamish Ivison, Ian Magnusson, Yizhong Wang, et~al.
\newblock Olmo: Accelerating the science of language models.
\newblock \emph{arXiv preprint arXiv:2402.00838}, 2024.

\bibitem[Gu et~al.(2024)Gu, Pang, Du, Liu, Zhang, Du, Wang, and
  Lin]{gu2024attention}
Xiangming Gu, Tianyu Pang, Chao Du, Qian Liu, Fengzhuo Zhang, Cunxiao Du,
  Ye~Wang, and Min Lin.
\newblock When attention sink emerges in language models: An empirical view.
\newblock \emph{arXiv preprint arXiv:2410.10781}, 2024.

\bibitem[Guo et~al.(2023)Guo, Hu, Mei, Wang, Xiong, Savarese, and
  Bai]{guo2023transformers}
Tianyu Guo, Wei Hu, Song Mei, Huan Wang, Caiming Xiong, Silvio Savarese, and
  Yu~Bai.
\newblock How do transformers learn in-context beyond simple functions? {A}
  case study on learning with representations.
\newblock \emph{arXiv preprint arXiv:2310.10616}, 2023.

\bibitem[Guo et~al.(2024)Guo, Kamigaito, and Watanabe]{guo2024attention}
Zhiyu Guo, Hidetaka Kamigaito, and Taro Watanabe.
\newblock Attention score is not all you need for token importance indicator in
  {KV} cache reduction: {V}alue also matters.
\newblock \emph{arXiv preprint arXiv:2406.12335}, 2024.

\bibitem[Gurnee et~al.(2024)Gurnee, Horsley, Guo, Kheirkhah, Sun, Hathaway,
  Nanda, and Bertsimas]{gurnee2024universal}
Wes Gurnee, Theo Horsley, Zifan~Carl Guo, Tara~Rezaei Kheirkhah, Qinyi Sun,
  Will Hathaway, Neel Nanda, and Dimitris Bertsimas.
\newblock Universal neurons in {GPT2} language models.
\newblock \emph{arXiv preprint arXiv:2401.12181}, 2024.

\bibitem[Han et~al.(2023)Han, Wang, Xiong, Chen, Ji, and Wang]{han2023lm}
Chi Han, Qifan Wang, Wenhan Xiong, Yu~Chen, Heng Ji, and Sinong Wang.
\newblock {LM-Infinite}: Simple on-the-fly length generalization for large
  language models.
\newblock \emph{arXiv preprint arXiv:2308.16137}, 2023.

\bibitem[Horn and Johnson(2012)]{horn2012matrix}
Roger~A Horn and Charles~R Johnson.
\newblock \emph{Matrix Analysis}.
\newblock Cambridge University Press, 2012.

\bibitem[Hu et~al.(2024)Hu, Chang, Luo, Chen, Li, Wang, and Liu]{hu2024outlier}
Jerry Yao-Chieh Hu, Pei-Hsuan Chang, Robin Luo, Hong-Yu Chen, Weijian Li,
  Wei-Po Wang, and Han Liu.
\newblock Outlier-efficient hopfield layers for large transformer-based models.
\newblock \emph{arXiv preprint arXiv:2404.03828}, 2024.

\bibitem[Huang et~al.(2023)Huang, Cheng, and Liang]{huang2023context}
Yu~Huang, Yuan Cheng, and Yingbin Liang.
\newblock In-context convergence of transformers.
\newblock \emph{arXiv preprint arXiv:2310.05249}, 2023.

\bibitem[Jacob et~al.(2018)Jacob, Kligys, Chen, Zhu, Tang, Howard, Adam, and
  Kalenichenko]{jacob2018quantization}
Benoit Jacob, Skirmantas Kligys, Bo~Chen, Menglong Zhu, Matthew Tang, Andrew
  Howard, Hartwig Adam, and Dmitry Kalenichenko.
\newblock Quantization and training of neural networks for efficient
  integer-arithmetic-only inference.
\newblock In \emph{Proceedings of the IEEE conference on computer vision and
  pattern recognition}, pages 2704--2713, 2018.

\bibitem[Jiang et~al.(2023)Jiang, Sablayrolles, Mensch, Bamford, Chaplot,
  Casas, Bressand, Lengyel, Lample, Saulnier, et~al.]{jiang2023mistral}
Albert~Q Jiang, Alexandre Sablayrolles, Arthur Mensch, Chris Bamford,
  Devendra~Singh Chaplot, Diego de~las Casas, Florian Bressand, Gianna Lengyel,
  Guillaume Lample, Lucile Saulnier, et~al.
\newblock Mistral 7b.
\newblock \emph{arXiv preprint arXiv:2310.06825}, 2023.

\bibitem[Karimi et~al.(2016)Karimi, Nutini, and Schmidt]{karimi2016linear}
Hamed Karimi, Julie Nutini, and Mark Schmidt.
\newblock Linear convergence of gradient and proximal-gradient methods under
  the polyak-{\l}ojasiewicz condition.
\newblock In \emph{Machine Learning and Knowledge Discovery in Databases:
  European Conference, ECML PKDD 2016, Riva del Garda, Italy, September 19-23,
  2016, Proceedings, Part I 16}, pages 795--811. Springer, 2016.

\bibitem[Kim et~al.(2024)Kim, Nakamaki, and Suzuki]{kim2024transformers}
Juno Kim, Tai Nakamaki, and Taiji Suzuki.
\newblock Transformers are minimax optimal nonparametric in-context learners.
\newblock \emph{arXiv preprint arXiv:2408.12186}, 2024.

\bibitem[Lin et~al.(2024{\natexlab{a}})Lin, Xu, Wu, Cui, Zhang, Mou, Song, Sun,
  and Wei]{lin2024duquant}
Haokun Lin, Haobo Xu, Yichen Wu, Jingzhi Cui, Yingtao Zhang, Linzhan Mou, Linqi
  Song, Zhenan Sun, and Ying Wei.
\newblock Duquant: Distributing outliers via dual transformation makes stronger
  quantized llms.
\newblock \emph{arXiv preprint arXiv:2406.01721}, 2024{\natexlab{a}}.

\bibitem[Lin et~al.(2024{\natexlab{b}})Lin, Tang, Tang, Yang, Chen, Wang, Xiao,
  Dang, Gan, and Han]{lin2024awq}
Ji~Lin, Jiaming Tang, Haotian Tang, Shang Yang, Wei-Ming Chen, Wei-Chen Wang,
  Guangxuan Xiao, Xingyu Dang, Chuang Gan, and Song Han.
\newblock Awq: Activation-aware weight quantization for on-device llm
  compression and acceleration.
\newblock \emph{Proceedings of Machine Learning and Systems}, 6:\penalty0
  87--100, 2024{\natexlab{b}}.

\bibitem[Lin et~al.(2023)Lin, Bai, and Mei]{lin2023transformers}
Licong Lin, Yu~Bai, and Song Mei.
\newblock Transformers as decision makers: Provable in-context reinforcement
  learning via supervised pretraining.
\newblock \emph{arXiv preprint arXiv:2310.08566}, 2023.

\bibitem[Lin et~al.(2020)Lin, Li, Liu, Xiao, Liu, and Zhu]{lin2020towards}
Ye~Lin, Yanyang Li, Tengbo Liu, Tong Xiao, Tongran Liu, and Jingbo Zhu.
\newblock Towards fully 8-bit integer inference for the transformer model.
\newblock \emph{arXiv preprint arXiv:2009.08034}, 2020.

\bibitem[Liu et~al.(2024)Liu, Bai, Lin, Li, Gao, Xu, Hou, Yao, and
  Yuan]{liu2024intactkv}
Ruikang Liu, Haoli Bai, Haokun Lin, Yuening Li, Han Gao, Zhengzhuo Xu, Lu~Hou,
  Jun Yao, and Chun Yuan.
\newblock Intactkv: Improving large language model quantization by keeping
  pivot tokens intact.
\newblock \emph{arXiv preprint arXiv:2403.01241}, 2024.

\bibitem[Liu et~al.(2022)Liu, Kitouni, Nolte, Michaud, Tegmark, and
  Williams]{liu2022towards}
Ziming Liu, Ouail Kitouni, Niklas~S Nolte, Eric Michaud, Max Tegmark, and Mike
  Williams.
\newblock Towards understanding grokking: An effective theory of representation
  learning.
\newblock \emph{Advances in Neural Information Processing Systems},
  35:\penalty0 34651--34663, 2022.

\bibitem[Meng et~al.(2022)Meng, Bau, Andonian, and Belinkov]{meng2022locating}
Kevin Meng, David Bau, Alex Andonian, and Yonatan Belinkov.
\newblock Locating and editing factual associations in gpt.
\newblock \emph{Advances in Neural Information Processing Systems},
  35:\penalty0 17359--17372, 2022.

\bibitem[Nagel et~al.(2021)Nagel, Fournarakis, Amjad, Bondarenko, Van~Baalen,
  and Blankevoort]{nagel2021white}
Markus Nagel, Marios Fournarakis, Rana~Ali Amjad, Yelysei Bondarenko, Mart
  Van~Baalen, and Tijmen Blankevoort.
\newblock A white paper on neural network quantization.
\newblock \emph{arXiv preprint arXiv:2106.08295}, 2021.

\bibitem[Nanda et~al.(2023)Nanda, Chan, Lieberum, Smith, and
  Steinhardt]{nanda2023progress}
Neel Nanda, Lawrence Chan, Tom Lieberum, Jess Smith, and Jacob Steinhardt.
\newblock Progress measures for grokking via mechanistic interpretability.
\newblock \emph{arXiv preprint arXiv:2301.05217}, 2023.

\bibitem[Nichani et~al.(2024)Nichani, Damian, and Lee]{nichani2024transformers}
Eshaan Nichani, Alex Damian, and Jason~D Lee.
\newblock How transformers learn causal structure with gradient descent.
\newblock \emph{arXiv preprint arXiv:2402.14735}, 2024.

\bibitem[Olsson et~al.(2022)Olsson, Elhage, Nanda, Joseph, DasSarma, Henighan,
  Mann, Askell, Bai, Chen, et~al.]{olsson2022context}
Catherine Olsson, Nelson Elhage, Neel Nanda, Nicholas Joseph, Nova DasSarma,
  Tom Henighan, Ben Mann, Amanda Askell, Yuntao Bai, Anna Chen, et~al.
\newblock In-context learning and induction heads.
\newblock \emph{arXiv preprint arXiv:2209.11895}, 2022.

\bibitem[Radford et~al.(2019)Radford, Wu, Child, Luan, Amodei, Sutskever,
  et~al.]{radford2019language}
Alec Radford, Jeffrey Wu, Rewon Child, David Luan, Dario Amodei, Ilya
  Sutskever, et~al.
\newblock Language models are unsupervised multitask learners.
\newblock \emph{OpenAI blog}, 1\penalty0 (8):\penalty0 9, 2019.

\bibitem[Reddy(2023)]{reddy2023mechanistic}
Gautam Reddy.
\newblock The mechanistic basis of data dependence and abrupt learning in an
  in-context classification task.
\newblock In \emph{The Twelfth International Conference on Learning
  Representations}, 2023.

\bibitem[Soldaini et~al.(2024)Soldaini, Kinney, Bhagia, Schwenk, Atkinson,
  Authur, Bogin, Chandu, Dumas, Elazar, et~al.]{soldaini2024dolma}
Luca Soldaini, Rodney Kinney, Akshita Bhagia, Dustin Schwenk, David Atkinson,
  Russell Authur, Ben Bogin, Khyathi Chandu, Jennifer Dumas, Yanai Elazar,
  et~al.
\newblock Dolma: An open corpus of three trillion tokens for language model
  pretraining research.
\newblock \emph{arXiv preprint arXiv:2402.00159}, 2024.

\bibitem[Son et~al.(2024)Son, Park, Han, Kim, and Lee]{son2024prefixing}
Seungwoo Son, Wonpyo Park, Woohyun Han, Kyuyeun Kim, and Jaeho Lee.
\newblock Prefixing attention sinks can mitigate activation outliers for large
  language model quantization.
\newblock \emph{arXiv preprint arXiv:2406.12016}, 2024.

\bibitem[Sun et~al.(2024)Sun, Chen, Kolter, and Liu]{sun2024massive}
Mingjie Sun, Xinlei Chen, J~Zico Kolter, and Zhuang Liu.
\newblock Massive activations in large language models.
\newblock \emph{arXiv preprint arXiv:2402.17762}, 2024.

\bibitem[Tian et~al.(2023{\natexlab{a}})Tian, Wang, Chen, and Du]{tian2023scan}
Yuandong Tian, Yiping Wang, Beidi Chen, and Simon~S Du.
\newblock Scan and {S}nap: Understanding training dynamics and token
  composition in 1-layer transformer.
\newblock \emph{Advances in Neural Information Processing Systems},
  36:\penalty0 71911--71947, 2023{\natexlab{a}}.

\bibitem[Tian et~al.(2023{\natexlab{b}})Tian, Wang, Zhang, Chen, and
  Du]{tian2023joma}
Yuandong Tian, Yiping Wang, Zhenyu Zhang, Beidi Chen, and Simon Du.
\newblock Joma: Demystifying multilayer transformers via joint dynamics of mlp
  and attention.
\newblock \emph{arXiv preprint arXiv:2310.00535}, 2023{\natexlab{b}}.

\bibitem[Todd et~al.(2023)Todd, Li, Sharma, Mueller, Wallace, and
  Bau]{todd2023function}
Eric Todd, Millicent~L Li, Arnab~Sen Sharma, Aaron Mueller, Byron~C Wallace,
  and David Bau.
\newblock Function vectors in large language models.
\newblock \emph{arXiv preprint arXiv:2310.15213}, 2023.

\bibitem[Touvron et~al.(2023)Touvron, Martin, Stone, Albert, Almahairi, Babaei,
  Bashlykov, Batra, Bhargava, Bhosale, et~al.]{touvron2023llama}
Hugo Touvron, Louis Martin, Kevin Stone, Peter Albert, Amjad Almahairi, Yasmine
  Babaei, Nikolay Bashlykov, Soumya Batra, Prajjwal Bhargava, Shruti Bhosale,
  et~al.
\newblock Llama 2: Open foundation and fine-tuned chat models.
\newblock \emph{arXiv preprint arXiv:2307.09288}, 2023.

\bibitem[Vaswani(2017)]{vaswani2017attention}
A~Vaswani.
\newblock Attention is all you need.
\newblock \emph{Advances in Neural Information Processing Systems}, 2017.

\bibitem[Vershynin(2018)]{vershynin2018high}
Roman Vershynin.
\newblock \emph{High-Dimensional Probability: An Introduction with Applications
  in Data Science}, volume~47.
\newblock Cambridge University Press, 2018.

\bibitem[Wang et~al.(2022)Wang, Variengien, Conmy, Shlegeris, and
  Steinhardt]{wang2022interpretability}
Kevin Wang, Alexandre Variengien, Arthur Conmy, Buck Shlegeris, and Jacob
  Steinhardt.
\newblock Interpretability in the wild: a circuit for indirect object
  identification in {GPT}-2 small.
\newblock \emph{arXiv preprint arXiv:2211.00593}, 2022.

\bibitem[Wu et~al.(2023{\natexlab{a}})Wu, Zou, Chen, Braverman, Gu, and
  Bartlett]{wu2023many}
Jingfeng Wu, Difan Zou, Zixiang Chen, Vladimir Braverman, Quanquan Gu, and
  Peter~L Bartlett.
\newblock How many pretraining tasks are needed for in-context learning of
  linear regression?
\newblock \emph{arXiv preprint arXiv:2310.08391}, 2023{\natexlab{a}}.

\bibitem[Wu et~al.(2023{\natexlab{b}})Wu, Li, Aminabadi, Yao, and
  He]{wu2023understanding}
Xiaoxia Wu, Cheng Li, Reza~Yazdani Aminabadi, Zhewei Yao, and Yuxiong He.
\newblock Understanding int4 quantization for language models: latency speedup,
  composability, and failure cases.
\newblock In \emph{International Conference on Machine Learning}, pages
  37524--37539. PMLR, 2023{\natexlab{b}}.

\bibitem[Xiao et~al.(2023{\natexlab{a}})Xiao, Lin, Seznec, Wu, Demouth, and
  Han]{xiao2023smoothquant}
Guangxuan Xiao, Ji~Lin, Mickael Seznec, Hao Wu, Julien Demouth, and Song Han.
\newblock Smoothquant: Accurate and efficient post-training quantization for
  large language models.
\newblock In \emph{International Conference on Machine Learning}, pages
  38087--38099. PMLR, 2023{\natexlab{a}}.

\bibitem[Xiao et~al.(2023{\natexlab{b}})Xiao, Tian, Chen, Han, and
  Lewis]{xiao2023efficient}
Guangxuan Xiao, Yuandong Tian, Beidi Chen, Song Han, and Mike Lewis.
\newblock Efficient streaming language models with attention sinks.
\newblock \emph{arXiv preprint arXiv:2309.17453}, 2023{\natexlab{b}}.

\bibitem[Yao et~al.()Yao, Aminabadi, Zhang, Wu, Li, and He]{yao2206efficient}
Z~Yao, RY~Aminabadi, M~Zhang, X~Wu, C~Li, and Y~Zeroquant He.
\newblock Efficient and affordable post-training quantization for large-scale
  transformers, 2022.
\newblock \emph{URL https://arxiv. org/abs/2206.01861}.

\bibitem[Yao et~al.(2022)Yao, Yazdani~Aminabadi, Zhang, Wu, Li, and
  He]{yao2022zeroquant}
Zhewei Yao, Reza Yazdani~Aminabadi, Minjia Zhang, Xiaoxia Wu, Conglong Li, and
  Yuxiong He.
\newblock Zeroquant: Efficient and affordable post-training quantization for
  large-scale transformers.
\newblock \emph{Advances in Neural Information Processing Systems},
  35:\penalty0 27168--27183, 2022.

\bibitem[Yu et~al.(2024)Yu, Wang, Fu, Shi, Shaikh, and Lin]{yu2024unveiling}
Zhongzhi Yu, Zheng Wang, Yonggan Fu, Huihong Shi, Khalid Shaikh, and
  Yingyan~Celine Lin.
\newblock Unveiling and harnessing hidden attention sinks: Enhancing large
  language models without training through attention calibration.
\newblock \emph{arXiv preprint arXiv:2406.15765}, 2024.

\bibitem[Zafrir et~al.(2019)Zafrir, Boudoukh, Izsak, and
  Wasserblat]{zafrir2019q8bert}
Ofir Zafrir, Guy Boudoukh, Peter Izsak, and Moshe Wasserblat.
\newblock Q8bert: Quantized 8bit bert.
\newblock In \emph{2019 Fifth Workshop on Energy Efficient Machine Learning and
  Cognitive Computing-NeurIPS Edition (EMC2-NIPS)}, pages 36--39. IEEE, 2019.

\bibitem[Zhai et~al.(2023)Zhai, Likhomanenko, Littwin, Busbridge, Ramapuram,
  Zhang, Gu, and Susskind]{zhai2023stabilizing}
Shuangfei Zhai, Tatiana Likhomanenko, Etai Littwin, Dan Busbridge, Jason
  Ramapuram, Yizhe Zhang, Jiatao Gu, and Joshua~M Susskind.
\newblock Stabilizing transformer training by preventing attention entropy
  collapse.
\newblock In \emph{International Conference on Machine Learning}, pages
  40770--40803. PMLR, 2023.

\bibitem[Zhang et~al.(2023)Zhang, Frei, and Bartlett]{zhang2023trained}
Ruiqi Zhang, Spencer Frei, and Peter~L Bartlett.
\newblock Trained transformers learn linear models in-context.
\newblock \emph{arXiv preprint arXiv:2306.09927}, 2023.

\bibitem[Zhang et~al.(2024)Zhang, Wu, and Bartlett]{zhang2024context}
Ruiqi Zhang, Jingfeng Wu, and Peter~L Bartlett.
\newblock In-context learning of a linear transformer block: Benefits of the
  {MLP} component and one-step {GD} initialization.
\newblock \emph{arXiv preprint arXiv:2402.14951}, 2024.

\bibitem[Zhang et~al.(2022)Zhang, Backurs, Bubeck, Eldan, Gunasekar, and
  Wagner]{zhang2022unveiling}
Yi~Zhang, Arturs Backurs, S{\'e}bastien Bubeck, Ronen Eldan, Suriya Gunasekar,
  and Tal Wagner.
\newblock Unveiling transformers with {LEGO}: A synthetic reasoning task.
\newblock \emph{arXiv preprint arXiv:2206.04301}, 2022.

\bibitem[Zhu et~al.(2024)Zhu, Huang, Zhang, Jordan, Jiao, Tian, and
  Russell]{zhu2024towards}
Hanlin Zhu, Baihe Huang, Shaolun Zhang, Michael Jordan, Jiantao Jiao, Yuandong
  Tian, and Stuart Russell.
\newblock Towards a theoretical understanding of the `reversal curse' via
  training dynamics.
\newblock \emph{arXiv preprint arXiv:2405.04669}, 2024.

\bibitem[Zhu and Li(2023)]{zhu2023physics}
Zeyuan~Allen Zhu and Yuanzhi Li.
\newblock Physics of language models: {P}art 3.1, knowledge storage and
  extraction.
\newblock \emph{arXiv preprint arXiv:2309.14316}, 2023.

\end{thebibliography}

\appendix
\makeatletter
\def\renewtheorem#1{%
  \expandafter\let\csname#1\endcsname\relax
  \expandafter\let\csname c@#1\endcsname\relax
  \gdef\renewtheorem@envname{#1}
  \renewtheorem@secpar
}
\def\renewtheorem@secpar{\@ifnextchar[{\renewtheorem@numberedlike}{\renewtheorem@nonumberedlike}}
\def\renewtheorem@numberedlike[#1]#2{\newtheorem{\renewtheorem@envname}[#1]{#2}}
\def\renewtheorem@nonumberedlike#1{  
\def\renewtheorem@caption{#1}
\edef\renewtheorem@nowithin{\noexpand\newtheorem{\renewtheorem@envname}{\renewtheorem@caption}}
\renewtheorem@thirdpar
}
\def\renewtheorem@thirdpar{\@ifnextchar[{\renewtheorem@within}{\renewtheorem@nowithin}}
\def\renewtheorem@within[#1]{\renewtheorem@nowithin[#1]}
\makeatother

\renewtheorem{theorem}{Theorem}[section]
\renewtheorem{lemma}[theorem]{Lemma}
\renewtheorem{remark}{Remark}
\renewtheorem{corollary}[theorem]{Corollary}
\renewtheorem{corollary*}{Corollary}
\renewtheorem{observation}[theorem]{Observation}
\renewtheorem{proposition}[theorem]{Proposition}
\renewtheorem{definition}[theorem]{Definition}
\renewtheorem{claim}{Claim}[section]
\renewtheorem{fact}[theorem]{Fact}
\renewtheorem{assumption}{Assumption}
\renewcommand{\theassumption}{\Alph{assumption}}
\renewtheorem{conjecture}[theorem]{Conjecture}

\clearpage
\tableofcontents
\clearpage
\section{Proofs of Theorem~\ref{thm:construction} and~\ref{thm:main}}
\label{sec:proof}

We introduce new notations that are frequently used in the proofs. Recall that in Eq.~\eqref{eqn:total_loss}, we used $\{\pi_v\}_{v \in \vocab}$ to denote the stable distribution across all tokens. We further define the stable distribution excluding trigger tokens as follows:
\begin{equation}\label{appeqn:stable-dist}
\Tilde{\bm{\stable}}\in \R^\vocabsize,~~~\Tilde{\stable}_i=\stable_i \bm{1}\{i\in\vocab\setminus\cT\}.
\end{equation}
Section~\ref{sec:simple-model} defines the bigram transition probability in the Bigram-Backcopy task as $\transition_{\tok\tokk}=\sf{P}(\tokk\mid\tok)$.
We further define the bigram transition probability matrix as 
\begin{equation}\label{appeqn:P-matrix}
\Transition = \left(\begin{matrix}
\transition_{11} & \ldots & \transition_{1\vocabsize}\\
\vdots & \ddots & \vdots \\
\transition_{\vocabsize 1} & \ldots & \transition_{\vocabsize\vocabsize}\\
\end{matrix}\right) = \left(\begin{matrix}
\bm{\transition}_1^\top\\
\vdots \\
\bm{\transition}_\vocabsize^\top\\
\end{matrix}\right).
\end{equation}
Given a token $\tok$, define the predicted probability at token $\tok$ as the logit output passed through the softmax activation. Let $\bH=[\bos; v_{1:n-1}; v]$. Using the form of $\TF(\bH)_n$ defined in Eq.~\eqref{eqn:q}, we denote 
\begin{equation}\label{appeqn:pred-prob}
\bm{\ppred}_{\tok} = \softmax(\TF(\bH)_n)=(\ppred_{\tok 1},\ldots,\ppred_{\tok \vocabsize}),~~~\text{with}~~~\ppred_{\tok \toki}=\frac{\transition_{\tok\toki} \exp\Big[\frac{\mass_\toki\xi_\toki+e^{\sink}\ivalue_\toki}{e^{\sink}+\mass}\Big]}{\sum_{\tokk=1}^\vocabsize \transition_{\tok\tokk}\exp\Big[\frac{\mass_\tokk\xi_\tokk+e^{\sink}\ivalue_\tokk}{e^{\sink}+\mass}\Big]}.
\end{equation}
Similar to Eq.~\eqref{appeqn:P-matrix}, we define the full output probability matrix as
\begin{equation}\label{appeqn:Q-matrix}
\Ppred = \left(\begin{matrix}
\ppred_{11} & \ldots & \ppred_{1\vocabsize}\\
\vdots & \ddots & \vdots \\
\ppred_{\vocabsize 1} & \ldots & \ppred_{\vocabsize\vocabsize}\\
\end{matrix}\right) = \left(\begin{matrix}
\bm{\ppred}_1^\top\\
\vdots \\
\bm{\ppred}_\vocabsize^\top\\
\end{matrix}\right).
\end{equation}
Using the notation $\bm{\ppred}_\tok$ and $\Tilde{\stable}_\tok$, we can rewrite the loss functions defined in Eq.~\eqref{eqn:loss_single} and Eq.~\eqref{eqn:total_loss} as follows:
\begin{equation}\label{appeqn:loss}
\loss_\tok(\sink_\tok,\vecvalue) = -\sum_{\tokk=1}^\vocabsize \transition_{\tok\tokk} \log \ppred_{\tok\tokk},~~~~\loss_\tok(\vecsink,\vecvalue) = \sum_{\tok=1}^\vocabsize \Tilde{\stable}_\tok \loss_\tok (\sink_\tok,\vecvalue).
\end{equation}
We always have that $\sum_{\tokk}\transition_{\tok\tokk}=1$ and $\sum_{\tokk}\ppred_{\tok\tokk}=1$. The total variation norm and KL-divergence are then defined as:  
\begin{equation}\label{appeqn:kl-divergence}
\| \bm{\transition}_\tok - \bm{\ppred}_\tok\|_{\text{TV}} = \sum_{\tokk} |\transition_{\tok\tokk}-\ppred_{\tok\tokk}|,~~~~\text{KL}(\bm{\transition}_\tok~||~\bm{\ppred}_\tok) = -\sum_{\tokk} \transition_{\tok\tokk} \log(\ppred_{\tok\tokk}/\transition_{\tok\tokk}).
\end{equation}
Given any vector $\bm{u}=[u_1;\ldots;u_d]$, define the corresponding diagonal matrix as
\[
\diag(\bm{u}) = \left(\begin{matrix}
u_{1} & 0 & \ldots & 0\\
\vdots & \ddots &  & \vdots \\
\vdots & & \ddots & \vdots \\
0 & \ldots & 0 & u_d \\
\end{matrix}\right).
\]
Given any $\bm{\transition}_\tok$ defined in Eq.~\eqref{appeqn:P-matrix}, denote 
\begin{equation}\label{appeqn:g-probs}
\gppred_\tok^{\Transition} = \diag(\bm{\transition}_\tok) - \bm{\transition}_\tok \bm{\transition}_\tok^\top, \quad \gppred_\tok^{\Ppred} = \diag(\bm{\ppred}_\tok) - \bm{\ppred}_\tok \bm{\ppred}_\tok^\top.
\end{equation}
We now present technical lemmas concerning $\gppred_\tok^{\Transition}$ and $\gppred_\tok^{\ppred}$.
\begin{lemma}\label{appthm:positive-definite}
The matrices $\gppred^\Transition_\tok \in \R^{V \times V}$ and $\gppred^\Ppred_\tok \in \R^{V \times V}$ are positive semi-definite for any $\tok \in \vocab$. 
\end{lemma}
\begin{proof}[Proof of \Cref{appthm:positive-definite}]
Since $\sum_{\tokk=1}^\vocabsize \transition_{\tok\tokk} = 1$ and $\sum_{\tokk=1}^\vocabsize \ppred_{\tok\tokk} = 1$ for any $\tok$, we have that
\begin{align*}
(\gppred^\Transition_\tok)_{\toki\toki}=\transition_\toki - \transition_{\toki}^2 & = \transition_\toki(\sum_{\tokk\neq\toki} \transition_\tokk) \geq \sum_{\tokk\neq\toki} |(\gppred^\Transition_\tok)_{\toki\tokk}|, \\
(\gppred^\Ppred_\tok)_{\toki\toki}=\ppred_\toki - \ppred_{\toki}^2 & = \ppred_\toki(\sum_{\tokk\neq\toki} \ppred_\tokk) \geq \sum_{\tokk\neq\toki} |(\gppred^\Ppred_\tok)_{\toki\tokk}|.
\end{align*}
This shows that both $\gppred^\Transition_\tok$ and $\gppred^\Ppred_\tok$ are diagonally dominant matrices. By Corollary 6.2.27 in \citet{horn2012matrix}, they are positive semi-definite.
\end{proof}
\begin{lemma}\label{appthm:min-eigenvalue}
Suppose that $\Tilde{\stable}_\tok > 0$ for any $\tok\in\vocab\setminus \cT$. For any $\bm{\eta} \in \R^\vocabsize$ with $\bm{\eta} \perp \bm{1}$, there exists $\omega>0$ such that
\[
\bm{\eta}^\top \Big[ \sum_{\tokk=1}^\vocabsize \Tilde{\stable}_\tokk \gppred_{\tokk}^{\Transition} \Big] \bm{\eta} \geq \omega \|\bm{\eta}\|_2^2.
\]
\end{lemma}
\begin{proof}[Proof of Lemma \ref{appthm:min-eigenvalue}]
Denote the null spaces of $\gppred_{\tok}^{\Transition}$ for $\tok \in \vocab$ as $\sf S_\tok$. We solve for each $\sf S_\tok$. Setting $\gppred_{\tok}^\Transition \bm{\eta}=0$ gives that 
\[
[\transition_{\tok\tokj} - \transition_{\tok\tokj}(\sum_{\tokk} \transition_{\tok\tokk})] \eta_\tokj = 0~\text{for any $\tokj\in\vocab$.}
\]
If $\transition_{\tok\tokj}\neq 0$, we divide each side with $\transition_{\tok\tokj}$ and get that $\eta_\tokj = \sum_{\tokk} \transition_{\tok\tokk} \eta_\tokk$. As a result, we get that 
\[
\text{\sf{S}}_\tok = \set{\bm{\eta}\mid \eta_\tokj\text{ is constant for } \transition_{\tok\tokj}\neq0}.
\]
Since all ${\stable}_\tokk > 0$, for any $\tokk\in\vocab\setminus \cT$, there is $\tok\in\vocab\setminus\cT$ such that $\transition_{\tok\tokk} > 0$, we get that
\[
\cap_{\tok\in\vocab\setminus\cT} \text{\sf{S}}_\tok = \set{c \cdot \bm{1}\mid c\in\R}.
\]
Since $\bm{\eta} \perp \bm{1}$, we get that $\bm{\eta} \perp \cap_{\tok\in\vocab\setminus\cT} \sf{S}_\tok$. We denote the minimal non-zero eigenvalues of $\gppred_{\tok}^\Ppred$ for $\tok\in\vocab\setminus\cT$ as $\lambda$. We get that
\[
\bm{\eta}^\top \Big[ \sum_{\tokk=1}^\vocabsize \Tilde{\stable}_\tokk\gppred_{\tokk}^{\Transition} \Big] \bm{\eta} \geq \Big[\min_{\tok \in \vocab\setminus \cT} \Tilde{\stable}_\tok\Big] \lambda \|\bm{\eta}\|_2^2.
\]
Setting $\omega = \lambda \cdot \min_{\tok \in \vocab\setminus \cT} \Tilde{\stable}_\tok>0$, this proves Lemma \ref{appthm:min-eigenvalue}.
\end{proof}

\begin{lemma}\label{appthm:min-eigenvalue-q}
Given $\omega$ defined in Lemma \ref{appthm:min-eigenvalue}, 
suppose that 
\begin{equation}\label{appeqn:error-q}
\max_{\tok,\tokk} \abs{\transition_{\tok\tokk}-\ppred_{\tok\tokk}} = \delta \leq \min\set{\omega/(6\vocabsize), 1}.
\end{equation}
For any $\bm{\eta} \in \R^\vocabsize$ with $\bm{\eta} \perp \bm{1}$, we have that
\[
\bm{\eta}^\top \Big[ \sum_{\tokk=1}^\vocabsize \Tilde{\stable}_\tokk \gppred_{\tokk}^{\Ppred} \Big] \bm{\eta} \geq \frac{\omega}{2} \|\bm{\eta}\|_2^2.
\]
\end{lemma}
\begin{proof}[Proof of Lemma \ref{appthm:min-eigenvalue-q}]
Denote $\delta = \max_{\tok,\tokk} \abs{\transition_{\tok\tokk}-\ppred_{\tok\tokk}}$. Suppose that $\delta \leq 1$.
For any $\tokk\in \vocab\setminus\cT$, we can verify that
\[\Big|(\gppred_\tokk^\Transition)_{ij}-(\gppred_\tokk^\Ppred)_{ij} \Big|\leq 3 \delta,\]
for any $\toki,\tokj\in [\vocabsize]$.
We denote 
\[
\bE = \sum_{\tokk=1}^\vocabsize \Tilde{\stable}_{\tokk} \gppred_\tokk^\Transition - \sum_{\tokk=1}^\vocabsize \Tilde{\stable}_{\tokk} \gppred_\tokk^\Ppred.
\]
Therefore, $|\bE_{\toki\tokj}|\leq 3\delta$ for any $\toki,\tokj \in [\vocabsize]$.
This means that
\[
\bm{\eta}^\top \bE \bm{\eta} \leq \|\bE\|_2 \|\bm{\eta}\|_2^2 \leq \|\bE\|_{F} \|\bm{\eta}\|_2^2 \leq V\cdot 3\delta \cdot \|\bm{\eta}\|_2^2.
\]
As a result, when $\delta \leq \min\set{\omega/(6\vocabsize), 1}$, we get that
\[
\bm{\eta}^\top \Big[ \sum_{\tokk=1}^\vocabsize \Tilde{\stable}_\tokk \gppred_{\tokk}^{\Ppred} \Big] \bm{\eta} \geq \omega \|\bm{\eta}\|_2^2 - \bm{\eta}^\top \bE \bm{\eta} \geq \frac{\omega}{2} \|\bm{\eta}\|_2^2.
\]
This proves Lemma \ref{appthm:min-eigenvalue-q}.
\end{proof}

\subsection{Proof of Theorem~\ref{thm:construction}}\label{app:proof-construction}


We denote the hidden dimension as $d$ and the sequence length as $N$. Recall that the token $\tok$ at position $\toki$ is encoded as $\embd_\toki(\tok)$. We begin with the assumption regarding the transformer's embedding dimension:
\begin{assumption}\label{ass:linear_indp}
We have $\set{\embd_0(\bos)} \cup \set{\embd_\toki(\tok)}_{\toki\in\{ 0 \} \cup [N-1],\tok\in\vocab} \subseteq \R^d$, where the embedding dimension  $d\geq \vocabsize N + 1$. 
\end{assumption}
Assumption~\ref{ass:linear_indp} requires a large embedding dimension $d\geq \vocabsize N +  1$.  This assumption is used to ensure that there are enough orthonormal bases in the embedding space. Given the fact that there are $\text{O}(\exp(d))$ approximately linearly independent vectors for large $d$ \citep{vershynin2018high}, it is possible to relax the assumption to be $d \gg \log(\vocabsize N)$. However, since Assumption~\ref{ass:linear_indp} pertains only to the construction of $\lambda$ for trigger tokens and is unrelated to Theorem~\ref{thm:main}, we adopt it to simplify the proof of Theorem~\ref{thm:construction}.

\begin{theorem}[Formal statement of Theorem \ref{thm:construction}] \label{appthm:formal-contruct} Let Assumption \ref{ass:linear_indp} hold. For any parameters $(\vecsink \in \R^{V}, \vecvalue \in \R^V, \bm{\xi} \in \R^V, \lambda \in \R)$, there exists a one-layer transformer (\ref{eqn:simplified_transformer}) with weight matrices $(\bQ, \bK, \bV, \bW_1, \bW_2)$ such that Eq. (\ref{eqn:simplification_TF_1}), (\ref{eqn:simplification_TF_2}), and (\ref{eqn:simplification_TF_3}) hold. Consider the Bigram-Backcopy task, where given an input $\bH = [\bos; \tok_{1:n-1}, \tok]$, the ground-truth transition gives ${\sf P}(\tok^\prime\mid \bH) = \transition_{\tok \tok^\prime}$ for $\tok\in\vocab\setminus\cT$, and ${\sf P}(\tok^\prime\mid \bH ) = \indic{\tok^\prime=\tok_{n-1}}$ for $\tok \in \cT$. 
There exists a sequence $\min_{\tok\in\vocab} \sink_\tok \to \infty$, $\min_{\tok\in\vocab} \xi_\tok \to \infty$, $\lambda\to\infty$, and $\vecvalue=\bm{0}$ such that this transformer generates the ground-truth transition in the limit, i.e., 
\begin{equation}\label{eqn:TF_ground_truth_match}
\softmax(\TF(\bH)_n) \to {\sf P}(\, \cdot\, | \bH).
\end{equation}
\end{theorem}

\begin{proof}[Proof of Theorem~\ref{appthm:formal-contruct}] ~


\noindent
{\bf Step 1. Construction for the attention head.} We let $\{\embd_0(\bos)\} \cup \{ \embd_\toki(\tok) \}_{i \in \{ 0 \} \cup [N-1], \tok \in \cV} \cup \set{\bm{e}_\tok}_{\tok\in\vocab}$ to be a set of orthonormal basis in $\R^d$, and denote $\set{\bm{\eta}_\toki}_{\toki\in\{ 0 \} \cup [N-1]} \subseteq \R^d$ by a set of orthonormal basis in $\R^d$ (the existence is guaranteed by Assumption~\ref{ass:linear_indp}). Therefore, for any parameters $(\vecsink \in \R^{V}, \vecvalue \in \R^V, \bm{\xi} \in \R^V, \lambda \in \R)$, there exists a query matrix $\bQ \in \R^{d \times N}$ such that
\begin{equation}\label{appeqn:mlp-construct}
\begin{aligned}
&~\bQ \cdot \embd_\toki(\tok) = \lambda\bm{\eta}_{i-1}~~~\text{for }i>1,~~\tok\in\cT,\\
&~\bQ \cdot \embd_\toki(\tok) = \sink_{\tok} \bm{\eta}_{0}~~~\text{for }i>0,~~\tok\in\vocab\setminus\cT.\\
\end{aligned}
\end{equation}
Meanwhile, there is a key matrix $\bK \in \R^{d \times N}$ such that
\begin{equation}
\begin{aligned}
&~\bK \cdot \embd_\toki(\tok) = \bm{\eta}_{i}~~~\text{for }i>0,~~\tok\in\vocab,\\
&~\bK \cdot \embd_0(\bos) = \bm{\eta}_{0}.
\end{aligned}
\end{equation}
Denote $\set{\bm{e}_\tok}_{\tok\in\vocab}$ as an orthonormal basis in $\R^\vocabsize$.
There is a matrix $\bV\in \R^{d \times \vocabsize}$ such that 
\begin{equation}
\begin{aligned}
&~ \bV \cdot \embd_i(\tok)  = \xi_\tok \bm{e}_{\tok} \in \R^{\vocabsize},  ~~~\text{with $\xi_\tok=0$ for $\tok\in\cT$, and $\xi_\tok\geq 0$ for $\tok\in\vocab\setminus\cT$.}\\
&~ \bV \cdot \embd_0(\bos) = \vecvalue\in\R^\vocabsize.\\
\end{aligned}
\end{equation}
This construction matches Eq. (\ref{eqn:simplification_TF_1}) and (\ref{eqn:simplification_TF_2}). 

As a result, for $\tok_n \in \vocab\setminus \cT$, by Eq.~(\ref{eqn:simplified_transformer}), denoting $\bH=[\bos; v_{1:n-1}; v_n]$ and $\attn(\bH)_n$ to be the last column of $\attn(\bH)$, we have
\begin{align*}
\attn(\bH)_n = &~ \sum_{i=0}^n \frac{\exp[ \embd_n(\tok_n)^\top \bQ^\top \bK\cdot \embd_i(\tok_i)] \bV \cdot  \embd_\toki(\tok_\toki)}{\sum_{j=0}^n\exp[\embd_n(\tok_n)^\top\bQ^\top \bK\cdot \embd_j(\tok_j)]}\\
= &~ \frac{\exp[\sink_{\tok_n} \bm{\eta}_{0}^\top \bm{\eta}_{0}]\cdot \bbeta + \sum_{i=1}^n \set{\exp[\sink_{\tok_n} \bm{\eta}_{0}^\top \bm{\eta}_{i}] \xi_{\tok_i} \cdot \bm{e}_{\tok_i}}}{\exp[\sink_{\tok_n} \bm{\eta}_{0}^\top \bm{\eta}_{0}]+\sum_{j=1}^n\exp[\sink_{\tok_n} \bm{\eta}_{0}^\top \bm{\eta}_{j}]}\\
= &~  \frac{e^{\sink_{\tok_n}}}{e^{\sink_{\tok_n}} + n} \cdot \bbeta +  \sum_{i=1}^n \frac{1}{e^{\sink_{\tok_n}} + n} \cdot \xi_{\tok_i} \bm{e}_{\tok_i}.
\end{align*}
For $\tok_n \in \cT$, we have
\begin{align*}
\attn(\bH)_n = &~ \sum_{i=0}^n \frac{\exp[\embd_n(\tok_n)^\top \bQ^\top \bK\cdot\embd_i(\tok_i)] \bV \cdot \embd_\toki(\tok_\toki)}{\sum_{j=0}^n\exp[\embd_n(\tok_n)^\top \bQ^\top \bK\cdot \embd_j(\tok_j)]}\\
= &~ \frac{\exp[\lambda \bm{\eta}_{n-1}^\top \bm{\eta}_{0}]\cdot \bbeta + \sum_{i=1}^n \set{\exp[\lambda \bm{\eta}_{n-1}^\top \bm{\eta}_{i}] \xi_{\tok_i} \cdot \bm{e}_{\tok_i}}}{\exp[\lambda \bm{\eta}_{n-1}^\top \bm{\eta}_{0}]+\sum_{j=1}^n\exp[\lambda \bm{\eta}_{n-1}^\top \bm{\eta}_{j}]}\\
= &~  \frac{1}{e^\lambda + n} \cdot \bbeta +  \sum_{i\neq n-1} \frac{1}{e^\lambda + n} \cdot \xi_{\tok_i} \bm{e}_{\tok_i} + \frac{e^\lambda}{e^\lambda + n} \cdot \xi_{\tok_{n-1}} \bm{e}_{\tok_{n-1}}.
\end{align*}

\noindent
{\bf Step 2. Construction for the MLP layer.} Further, define the weights for the \mlp~layer such that
\begin{equation}
\begin{aligned}
&~ \bW_1 \cdot \embd_\toki(\tok) = \bm{e}_\tok \in \R^\vocabsize, ~~~\bW_2\bm{e}_\tok =\log \bm{\transition}_{\tok} \cdot 1\{ \tok \not\in \cT  \} \in \R^\vocabsize  ~~~ \text{for } i\in[N],~~\tok \in \vocab,
\end{aligned}
\end{equation}
where $\set{\bm{e}_\tok}$ is the eorthonormal basis in $\R^\vocabsize$ and $\bm{\transition}_\tok \in \R^\vocabsize$ is defined in Eq.~(\ref{appeqn:P-matrix}). As a result, $\mlp(\bH)_n = \bW_2\text{ReLU}(\bW_1 \embd_n(\tok)) =\bW_2 \bm{e}_\tok =\log \bm{\transition}_\tok \cdot \bm{1}\set{\tok \notin \cT} $.
 This matches the Eq.~(\ref{eqn:simplification_TF_3}). 

\noindent
{\bf Step 3. The output of the transformer.} By Eq.~(\ref{eqn:simplified_transformer}) again, on non-trigger token $\tok\in\vocab\setminus \cT$, the transformer output gives that 
\begin{align*}
\TF(\bH)_n&=\mlp(\embd_n(\tok))+\attn(\bH)_n\\
&=\log \bm{\transition}_\tok + \frac{e^{\sink_{\tok_n}}}{e^{\sink_{\tok_n}} + n} \cdot \bbeta +  \sum_{i=0}^n \frac{1}{e^{\sink_{\tok_n}} + n} \cdot \xi_{\tok_i} \bm{e}_{\tok_i}.
\end{align*}
On trigger token $\tok\in\cT$, the transformer output gives that
\begin{align*}
\TF(\bH)_n&=\mlp(\embd_n(\tok))+\attn(\bH)_n\\
&=\frac{1}{e^\lambda + n} \cdot \bbeta +  \sum_{i\neq n-1} \frac{1}{e^\lambda + n} \cdot \xi_{\tok_i} \bm{e}_{\tok_i} + \frac{e^\lambda}{e^\lambda + n} \cdot \xi_{\tok_{n-1}} \bm{e}_{\tok_{n-1}}.
\end{align*}

There exists a sequence $\min_{v \in \vocab} \sink_v \to\infty$, $\min_{v \in \vocab} \xi_v \to\infty$, $\lambda\to \infty$, and $\vecvalue=0$, we get that
\begin{equation*}
\softmax[\TF(\bH)_n] \to \bm{\transition}_{\tok_n} ~~~\text{for }n>0,~~\tok_n\in\vocab \setminus \cT,
\end{equation*}
\begin{equation*}
\softmax[\TF(\bH)_n] \to ( 1\{ v = \tok_{n-1} \} )_{v \in \cV}~~~\text{for }n>0,~~\tok_n\in\cT.
\end{equation*}
This proves Eq.~(\ref{eqn:TF_ground_truth_match}), indicating that the transformer output matches the ground truth transition. This finishes the proof of Theorem~\ref{appthm:formal-contruct}. 
\end{proof}

\subsection{Proof of Theorem~\ref{thm:main}(c): Stable phase}\label{app:proof-main-3}




We first state \Cref{appthm:g-ppred} and \Cref{appthm:gradients} that are used to prove~\Cref{thm:main}(c). Lemma~\ref{appthm:g-ppred} computes the gradients of $\ppred_{\toki\tokk}$ as defined in Eq. \eqref{appeqn:pred-prob}. 
\begin{lemma}\label{appthm:g-ppred} 
Given $\ppred_{\toki\tokk}$ defined in Eq.~\eqref{appeqn:pred-prob}, for any $\toki$, $\tokk$, $\tok$,  and any value of $\sink_\tok$ and $\ivalue_\tok$, we have that
\begin{align*}
\frac{\partial \ppred_{\toki\tokk}}{\partial \sink_\tok} = &~ \frac{\bm{1}\{\toki=\tok\} \ppred_{\toki\tokk}e^{\sink_\toki}}{(e^{\sink_\toki}+\mass)^2} \Big[\mass\ivalue_\tokk-\mass_\tokk \xi_\tokk - \sum_{\tokj=1}^\vocabsize \ppred_{\toki\tokj} (\mass\ivalue_\tokj -\mass_\tokj\xi_\tokj)\Big],\\
\frac{\partial \ppred_{\toki\tokk}}{\partial \ivalue_\tok} = &~ \frac{e^{\sink_\toki}}{e^{\sink_\toki}+\mass}[\ppred_{\toki\tokk}\bm{1}\{\tokk=\tok\} - \ppred_{\toki\tokk}\ppred_{\toki\tok}].\\
\end{align*}
Furthermore, we have
\[
\sum_{\tokk=1}^\vocabsize \frac{\partial \ppred_{\toki\tokk}}{\partial \sink_\tok} = 0~~~\text{for any }\toki,\ \tok,\ \vecsink,\ \text{and }\vecvalue, \quad \sum_{\tok=1}^\vocabsize \frac{\partial \ppred_{\toki\tokk}}{\partial \ivalue_\tok} = 0~~~\text{for any }\toki,\ \tokk,\ \vecsink,\ \text{and }\vecvalue.
\]
\end{lemma}
\begin{proof}[Proof of Lemma~\ref{appthm:g-ppred}]
We repeatedly use the following two facts:
\begin{align*}
\frac{\partial\Big\{\exp\Big[\frac{\mass_\tokk\xi_\tokk+e^{\sink_\toki}\ivalue_\tokk}{e^{\sink_\toki}+\mass}\Big]\Big\}}{\partial \sink_\tok} = &~ \frac{\bm{1}\{\toki=\tok\}e^{\sink_\toki}(\mass\ivalue_\tokk-\mass_\tokk\xi_\tokk)}{(e^{\sink_\toki}+\mass)^2} \exp\Big[\frac{\mass_\tokk\xi_\tokk+e^{\sink_\toki}\ivalue_\tokk}{e^{\sink_\toki}+\mass}\Big],
\nonumber\\
\frac{\partial\Big\{\exp\Big[\frac{\mass_\tokk\xi_\tokk+e^{\sink_\toki}\ivalue_\tokk}{e^{\sink_\toki}+\mass}\Big]\Big\}}{\partial \ivalue_\tok} = &~ \frac{\bm{1}\{\tokk=\tok\}e^{\sink_\toki}}{e^{\sink_\toki}+\mass} \exp\Big[\frac{\mass_\tokk\xi_\tokk+e^{\sink_\toki}\ivalue_\tokk}{e^{\sink_\toki}+\mass}\Big].
\nonumber
\end{align*}

When $\toki\neq \tok$, $\ppred_{\toki\tokk}$ has zero gradients with respect to $\sink_\tok$. When $\toki=\tok$, we have that
\begin{align*}
\frac{\partial \ppred_{\tok\tokk}}{\partial \sink_\tok} = &~ \ppred_{\tok\tokk} e^{\sink_\tok} \Big[\frac{\mass\ivalue_\tokk-\mass_\tokk\xi_\tokk}{(e^{\sink_\tok}+\mass)^2}\Big] - \frac{\ppred_{\tok\tokk}\sum_{\toki=1}^\vocabsize \transition_{\tok\toki} e^{\sink_\tok} \Big[\frac{\mass\ivalue_\toki-\mass_\toki\xi_\toki}{(e^{\sink_\tok}+\mass)^2}\Big]\exp\Big[\frac{\mass_\toki\xi_\toki+e^{\sink_\tok}\ivalue_\toki}{e^{\sink_\tok}+\mass}\Big]}{\sum_{\toki=1}^\vocabsize \transition_{\tok\toki}\exp\Big[\frac{\mass_\toki\xi_\toki+e^{\sink_\tok}\ivalue_\toki}{e^{\sink_\tok}+\mass}\Big]}\\
= &~ \frac{e^{\sink_\tok}}{(e^{\sink_\tok}+\mass)^2} \Big\{ \ppred_{\tok\tokk} [\mass\ivalue_\tokk-\mass_\tokk \xi_\tokk] - \ppred_{\tok\tokk}\sum_{\tokj=1}^\vocabsize \ppred_{\tok\tokj} (\mass\ivalue_\tokj -\mass_\tokj\xi_\tokj)\Big\},
\end{align*}
and
\begin{align*}
\frac{\partial \ppred_{\toki\tokk}}{\partial \ivalue_\tok} = &~ \Big[\frac{e^{\sink_\toki}}{e^{\sink_\toki}+\mass}\Big]\ppred_{\toki\tokk} \bm{1}\{\tokk=\tok\} - \frac{\Big[\frac{e^{\sink_\toki}}{e^{\sink_\toki}+\mass}\Big]\transition_{\toki\tok}\exp\Big[\frac{\mass_\tok\xi_\tok+e^{\sink_\toki}\ivalue_\tok}{e^{\sink_\toki}+\mass}\Big]\transition_{\toki\tokk}\exp\Big[\frac{\mass_\tokk\xi_\tokk+e^{\sink_\toki}\ivalue_\tokk}{e^{\sink_\toki}+\mass}\Big]}{\Big(\sum_{j=1}^\vocabsize \transition_{\toki\tokj}\exp\Big[\frac{\mass_\tokj\xi_\tokj+e^{\sink_\toki}\ivalue_\tokj}{e^{\sink_\toki}+\mass}\Big]\Big)^2}\\
= &~ \Big[\frac{e^{\sink_\toki}}{e^{\sink_\toki}+\mass}\Big] [ \ppred_{\toki\tokk} \bm{1}\{\tokk=\tok\} - \ppred_{\toki\tokk} \ppred_{\toki\tok} ].
\end{align*}
We can verify that 
\begin{align*}
\sum_{\tokk=1}^\vocabsize \frac{\partial \ppred_{\toki\tokk}}{\partial \sink_\tok} = &~  \frac{e^{\sink_\tok}}{(e^{\sink_\tok}+\mass)^2} \sum_{\tokk=1}^\vocabsize \Big\{ \ppred_{\tok\tokk} [\mass\ivalue_\tokk-\mass_\tokk \xi_\tokk] - \ppred_{\tok\tokk}\sum_{\tokj=1}^\vocabsize \ppred_{\tok\tokj} (\mass\sink_\tokj -\mass_\tokj\xi_\tokj)\Big\}\\
= &~ \frac{e^{\sink_\tok}}{(e^{\sink_\tok}+\mass)^2}  \Big\{\sum_{\tokk=1}^\vocabsize \ppred_{\tok\tokk} [\mass\ivalue_\tokk-\mass_\tokk \xi_\tokk] - \sum_{\tokj=1}^\vocabsize \ppred_{\tok\tokj} (\mass\sink_\tokj -\mass_\tokj\xi_\tokj)\Big\}\\
= &~ 0,
\end{align*}
and
\begin{align*}
\sum_{\tok=1}^\vocabsize \frac{\partial \ppred_{\toki\tokk}}{\partial \ivalue_\tok} = &~   \Big[\frac{e^{\sink_\toki}}{e^{\sink_\toki}+\mass}\Big] \sum_{\tok=1}^\vocabsize [\ppred_{\toki\tokk} \bm{1}\{\tokk=\tok\} - \ppred_{\toki\tokk} \ppred_{\toki\tok} ]\\
= &~ \Big[\frac{e^{\sink_\toki}}{e^{\sink_\toki}+\mass}\Big] [\ppred_{\toki\tokk} - \ppred_{\toki\tokk} ]\\
= &~ 0.
\end{align*}
This finishes the proof of Lemma~\ref{appthm:g-ppred}.
\end{proof}

Proposition~\ref{appthm:gradients} computes the gradient of $\loss$ with respect to $\vecsink$ and $\vecvalue$, giving the ODE of the gradient flow.
\begin{proposition}\label{appthm:gradients}
Consider the gradient flow of optimizing $\loss(\vecsink, \vecvalue)$ given by 
\begin{equation}\label{appeqn:def-gradient-flow}
\dot{\vecsink}(t) = -\nabla_{\vecsink} \loss(\vecsink(t),\vecvalue(t)), ~~~ \dot{\vecvalue}(t) = -\nabla_{\vecvalue} \loss(\vecsink(t),\vecvalue(t)).
\end{equation}
Simplifying the dynamics using Lemma~\ref{appthm:g-ppred} gives that
\begin{align*}
\dot{\sink}_\tok(t) & = \frac{\Tilde{\stable}_\tok e^{\sink_\tok}}{(e^{\sink_\tok}+\mass)^2} 
\sum_{\toki=1}^\vocabsize({\transition}_{\tok\toki}-{\ppred}_{\tok\toki})(\mass\ivalue_\toki-\mass_\toki\xi_\toki),\\
\dot{\ivalue}_\tok(t) & = \sum_{\tokk=1}^\vocabsize \Big\{\frac{\Tilde{\stable}_\tokk e^{\sink_\tokk} [\transition_{\tokk\tok} - \ppred_{\tokk\tok}]}{e^{\sink_\tokk}+\mass}\Big\}.
\end{align*}
\end{proposition}
\begin{proof}[Proof of Proposition~\ref{appthm:gradients}]
Taking the derivative of $\loss(\vecsink,\vecvalue)$ gives that
\begin{align*}
\frac{\partial \loss(\vecsink,\vecvalue)}{\partial \sink_\tok}
= ~& \Tilde{\stable}_\tok \sum_{\tokk=1}^\vocabsize \transition_{\tok\tokk} \cdot \frac{-1}{\ppred_{\tok\toki}}\cdot \frac{\partial \ppred_{\tok\toki}}{\partial \sink_\tok}\\
= ~& \frac{\Tilde{\stable}_\tok e^{\sink_\tok}}{(e^{\sink_\tok}+\mass)^2} \Big\{  \sum_{\toki=1}^\vocabsize \ppred_{\tok\toki} [\mass\ivalue_\toki -\mass_\toki \xi_\toki] - \sum_{\tokk=1}^\vocabsize \transition_{\tok\tokk} [\mass\ivalue_\tokk-\mass_\tokk\xi_\tokk]\Big\}\\
= ~& \frac{\Tilde{\stable}_\tok e^{\sink_\tok}}{(e^{\sink_\tok}+\mass)^2} \sum_{\tokk=1}^\vocabsize \Big\{  [\ppred_{\tok\tokk}-\transition_{\tok\tokk}] [\mass\ivalue_\tokk-\mass_\tokk\xi_\tokk]\Big\}.
\end{align*}
Similarly, we have that
\begin{align*}
\frac{\partial \loss(\vecsink,\vecvalue)}{\partial \ivalue_\tok} 
= ~& \sum_{j=1}^\vocabsize \Tilde{\stable}_\tokj \sum_{\tokk=1}^\vocabsize \transition_{\tokj\tokk} \Big\{ \frac{e^{\sink_\tokj} \ppred_{\tokj\tok}}{e^{\sink_\tokj}+\mass} - \frac{e^{\sink_\tokj} \bm{1}\{\tokk=\tok\}}{e^{\sink_\tokj}+\mass} \Big\}\\
= ~& \sum_{j=1}^\vocabsize \Big\{\frac{\Tilde{\stable}_{\tokj}e^{\sink_\tokj}[\ppred_{\tokj\tok}-\transition_{\tokj\tok}]}{e^{\sink_\tokj}+\mass}\Big\}.
\end{align*}
Plug them in Eq.~\eqref{appeqn:def-gradient-flow} proves Proposition~\ref{appthm:gradients}.
\end{proof}
\begin{theorem}[Restatement the stable phase part in Theorem~\ref{thm:main}(c)]\label{appthm:main-1} 
Assume $\xi_\tok \ge 0$ for any $\tok$, $\stable_\tok > 0$ for any $\tok\in\vocab$, and $\{ \mass_i \cdot \xi_i \}_{i \in \vocab}$ are not all equal. Consider the gradient flow over the variables $(\vecsink, \vecvalue)$, i.e., $(\dot{\vecsink}(t), \dot{\vecvalue}(t)) = - \nabla_{\vecsink, \vecvalue}\loss(\vecsink(t), \vecvalue(t))$. Any vector of the following form
    \begin{equation}\vecsink^\star = \sink \cdot \bm{1}, \quad \vecvalue^\star = c \cdot \bm{1} - e^{-\sink} \cdot \bm{\mass} \circ \bm{\xi},  ~~~ \sink, c\in\R \end{equation}
 is a stationary point. These are all global minimizers of $\loss(\vecsink,\vecvalue)$.
\end{theorem}
\begin{proof}[Proof of Theorem~\ref{appthm:main-1}]
When $\vecsink=\vecsink^\star$ and $\vecvalue=\vecvalue^\star$, given $\ppred_{\tok\toki}$ defined in Eq.~\eqref{appeqn:pred-prob} with any $\tok$ and $\toki$, we have that
\begin{align*}
\ppred_{\tok\toki} = &~ \frac{\transition_{\tok\toki} \exp\Big[\frac{\mass_\toki\xi_\toki+e^{\sink}\ivalue_\toki}{e^{\sink}+\mass}\Big]}{\sum_{\tokk=1}^\vocabsize \transition_{\tok\tokk}\exp\Big[\frac{\mass_\tokk\xi_\tokk+e^{\sink}\ivalue_\tokk}{e^{\sink}+\mass}\Big]} \\
 = &~ \frac{\transition_{\tok\toki} \exp\Big[\frac{e^{\sink}c}{e^{\sink}+\mass}\Big]}{\sum_{\tokk=1}^\vocabsize \transition_{\tok\tokk}\exp\Big[\frac{e^{\sink}c}{e^{\sink}+\mass}\Big]} \\
 = &~ \transition_{\tok\toki}.
\end{align*}
Plug $\ppred_{\tok\toki}$ into $\partial \loss(\vecsink,\vecvalue)/\partial \vecsink$ and $\partial \loss(\vecsink,\vecvalue)/\partial \vecvalue$, we have
\begin{align*}
\frac{\partial \loss(\vecsink,\vecvalue)}{\partial \sink_\tok}\Big |_{\vecsink^\star,\vecvalue^\star} & = \frac{\Tilde{\stable}_\tok e^{\sink_\tok}}{(e^{\sink_\tok}+\mass)^2} 
\sum_{\tokk=1}^\vocabsize \Big\{(\ppred_{\tok\tokk}-\transition_{\tok\tokk}) [W \ivalue_\tokk -\mass_\tokk \xi_\tokk]
\Big\} = 0,\\
\frac{\partial \loss(\vecsink,\vecvalue)}{\partial \ivalue_\tok}\Big |_{\vecsink^\star,\vecvalue^\star} & = \sum_{\tokk=1}^\vocabsize \Big\{\frac{\Tilde{\stable}_\tokk e^{\sink_\tokk} [\ppred_{\tokk\tok}-\transition_{\tokk\tok} ]}{e^{\sink_\tokk}+\mass}\Big\} = 0.
\end{align*}
This shows that $\vecsink=\vecsink^\star$ and $\vecvalue=\vecvalue^\star$ are stationary points. We further compute the second-order derivative using Lemma~\ref{appthm:g-ppred}. To simplify the notation, we use $z_\tokk = W \ivalue_\tokk -\mass_\tokk \xi_\tokk$ and $\bm{z}=[z_1,\ldots,z_\vocabsize]$. We have that
\begin{align*}
\frac{\partial^2 \loss(\vecsink, \vecvalue)}{\partial \sink_\toki \partial \sink_\tok} \Big |_{\vecsink^\star,\vecvalue^\star} =  ~& \bm{1}\{\tok=\toki\}\cdot \frac{\Tilde{\stable}_\tok e^{\sink}}{(e^{\sink}+\mass)^2} 
\sum_{\tokk=1}^\vocabsize \Big\{\frac{\partial \ppred_{\toki\tokk}}{\partial \sink_\tok} z_\tokk
\Big\}\\
=  ~& \bm{1}\{\tok=\toki\}\cdot \frac{\Tilde{\stable}_\tok e^{2\sink}}{(e^{\sink}+\mass)^4} 
\Big\{ \sum_{\tokk=1}^\vocabsize \ppred_{\toki\tokk} z_\tokk^2 - \Big[\sum_{\tokk=1}^\vocabsize \ppred_{\toki\tokk}z_\tokk \Big]^2
\Big\}\\
=  ~& \bm{1}\{\tok=\toki\}\cdot \frac{\Tilde{\stable}_\tok e^{2\sink}}{(e^{\sink}+\mass)^4} 
\Big\{ \sum_{\tokk=1}^\vocabsize \transition_{\toki\tokk} z_\tokk^2 - \Big[\sum_{\tokk=1}^\vocabsize \transition_{\toki\tokk}z_\tokk \Big]^2
\Big\},
\end{align*}
where in the last line, we plugged in $\ppred_{\tok\toki}=\transition_{\tok\toki}$ for any $\tok$ and $\toki$. Similarly, we compute the second order derivatives with respect to $\sink_\toki$ and $\ivalue_\tok$,
\begin{align*}
\frac{\partial^2 \loss(\vecsink, \vecvalue)}{\partial \sink_\toki \partial \ivalue_\tok} \Big |_{\vecsink^\star,\vecvalue^\star} =  ~& \frac{\Tilde{\stable}_\toki e^{\sink}}{(e^{\sink}+\mass)^2} 
\sum_{\tokk=1}^\vocabsize \Big\{\frac{\partial \ppred_{\toki\tokk}}{\partial \ivalue_\tok} z_\tokk
\Big\}\\
=  ~& \frac{\Tilde{\stable}_\toki e^{2\sink}}{(e^{\sink}+\mass)^3} 
\Big\{ \transition_{\toki\tok} z_\tokk - \transition_{\toki\tok} \sum_{\tokk=1}^\vocabsize  \transition_{\toki\tokk}z_\tokk
\Big\}.
\end{align*}
With the same manner, we compute the second order derivatives with respect to $\ivalue_\toki$ and $\ivalue_\tok$, 
\begin{align*}
\frac{\partial^2 \loss(\vecsink, \vecvalue)}{\partial \ivalue_\toki \partial \ivalue_\tok} \Big |_{\vecsink^\star,\vecvalue^\star} =  ~& \sum_{\tokk=1}^\vocabsize \Big\{\frac{\partial \ppred_{\tokk\toki}}{\partial \ivalue_\tok}\frac{\Tilde{\stable}_\tokk e^{\sink}}{e^{\sink}+\mass}\Big\}\\
= ~& \frac{e^{2\sink}}{(e^\sink+\mass)^2}\sum_{\tokk=1}^\vocabsize \{ \Tilde{\stable}_\tokk [\bm{1}\{\tok=\toki\} \transition_{\tokk\tok} - \transition_{\tokk\toki}\transition_{\tokk\tok}] \}.
\end{align*}
Combining the above computations gives that
\begin{align*}
\text{Hessian}(\loss(\vecsink^\star, \vecvalue^\star)) = \left(\begin{matrix}
\nabla^2_{\vecsink} \loss(\vecsink,\vecvalue) & \nabla_{\vecsink} \nabla_{\vecvalue}  \loss(\vecsink,\vecvalue)\\
\nabla_{\vecvalue} \nabla_{\vecsink} \loss(\vecsink,\vecvalue) & \nabla^2_{\vecsink} \loss(\vecsink,\vecvalue)\\
\end{matrix}\right),
\end{align*}
with
\begin{align*}
\nabla^2_{\vecsink} \loss(\vecsink,\vecvalue) = &~ \frac{e^{2\sink}}{(e^\sink+\mass)^4}\diag\Big\{ \Tilde{\stable} \circ [\bz^\top \gppred^\Transition_1\bz;\ldots;\gppred^\Transition_\vocabsize \bz] \Big\},\\
\nabla_{\vecsink} \nabla_{\vecvalue} \loss(\vecsink,\vecvalue) = &~ \frac{e^{2\sink}}{(e^\sink+\mass)^3}\diag\Big\{ \Tilde{\stable} \Big\} [ \bz^\top\gppred^\Transition_1;\ldots; \bz^\top\gppred^\Transition_\vocabsize] ,\\
\nabla^2_{\vecvalue} \loss(\vecsink,\vecvalue) = &~ \frac{ e^{2\sink}}{(e^{\sink}+\mass)^2} \sum_{\tokk=1}^\vocabsize  \Tilde{\stable}_\tokk \gppred^\Transition_\tokk,
\end{align*}
where $\bG^\Transition_\tokk$ is defined in Eq.~\eqref{appeqn:g-probs}. Furthermore, there exists $\bm{U}$ such that $\bm{U} \text{Hessian}(\loss(\vecsink^\star, \vecvalue^\star)) \bm{U}^\top = \text{Diag-Hessian}(\loss(\vecsink^\star, \vecvalue^\star))$, with
\[
\text{Diag-Hessian}(\loss(\vecsink^\star, \vecvalue^\star)) = \left(\begin{matrix}
\nabla^2_{\vecsink} \loss(\vecsink,\vecvalue) & 0\\
0 & \frac{ e^{2\sink}}{(e^{\sink}+\mass)^2} \bB\\
\end{matrix}\right),
\]
where the $\bB$ is given by
\begin{align*}
\bB = \sum_{\tokk=1}^\vocabsize  \Tilde{\stable}_\tokk \Big(\gppred^\Transition_\tokk - (\bm{z}^\top \gppred_\tokk^\Transition \bm{z})^{-1}\gppred^\Transition_\tokk\bm{z}\bm{z}^\top\gppred^\Transition_\tokk\Big).
\end{align*}
To prove that $\bB$ is positive semi-definite, consider any vector $\bm{\eta}$ with $\|\bm{\eta}\|_2=1$: 
\begin{align*}
\bm{\eta}^\top \bB \bm{\eta} = &~ \sum_{\tokk=1}^\vocabsize  \Tilde{\stable}_\tokk \Big(\bm{\eta}^\top\gppred^\Transition_\tokk\bm{\eta} - \frac{\bm{\eta}^\top\gppred^\Transition_\tokk\bm{z}\bm{z}^\top\gppred^\Transition_\tokk\bm{\eta}}{\bm{z}^\top \gppred_\tokk^\Transition \bm{z}}\Big).
\end{align*}
Since $\gppred^\Transition_\tokk$ is positive semi-definite, the Cauchy inequality gives that
\[
\bm{z}^\top\gppred^\Transition_\tokk\bm{\eta} \leq \sqrt{\bm{z}^\top \gppred_\tokk^\Transition \bm{z} \bm{\eta}^\top \gppred_\tokk^\Transition \bm{\eta}}.
\]
As a result, we have that
\begin{align*}
\bm{\eta}^\top \bB \bm{\eta} \geq &~ \sum_{\tokk=1}^\vocabsize  \Tilde{\stable}_\tokk \Big(\bm{\eta}^\top\gppred^\Transition_\tokk\bm{\eta} - \frac{\bm{z}^\top \gppred_\tokk^\Transition \bm{z} \bm{\eta}^\top \gppred_\tokk^\Transition \bm{\eta}}{\bm{z}^\top \gppred_\tokk^\Transition \bm{z}}\Big) = 0.
\end{align*}
This shows that $\bB$ is positive semi-definite. Therefore, $\text{Hessian}(\loss(\vecsink^\star, \vecvalue^\star))$ is positive semi-definte.  This proves Theorem~\ref{appthm:main-1}.
\end{proof}

\subsection{Proof of Theorem~\ref{thm:main}(a): Attention sinks}\label{appsec:proof-main-1}
\begin{theorem}[Restatement of the attention sink part in Theorem~\ref{thm:main}(a)]\label{appthm:main-2}
Assume $\xi_\tok \ge 0$ for any $\tok$, $\stable_\tok > 0$ for any $\tok\in\vocab$, and $\{ \mass_i \cdot \xi_i \}_{i \in \vocab}$ are not all equal. Fix $\vecvalue= \beta \cdot \bm{1}$ for a constant $\beta$, and consider the gradient flow of the loss function $\loss(\vecsink, \vecvalue)$ over $\vecsink$, i.e., $\dot{\vecsink}(t) = - \nabla \loss(\vecsink(t), \vecvalue)$. With any initial value $\vecsink(0)$, there exists $\bm{r}(t)$ with norm uniformly bounded in time, such that 
    \begin{equation}
    \textstyle \vecsink(t) = \frac{1}{2} \log t \cdot \bm{1} + \bm{r}(t).
    \end{equation}
\end{theorem}
\begin{proof}[Proof of Theorem~\ref{appthm:main-2}]
We separately analyze each entry of $\vecsink$. Focusing on $\sink_\tok$, to simplify the notation, we introduce a random variable $\varphi$ such that $$\P(\varphi=\mass_\tokk\xi_\tokk)=\transition_{\tok\tokk}.$$
Denote 
\[
u = e^{\sink_\tok}.
\]
Therefore, using Lemma~\ref{appthm:gradients}, we get that
\begin{align*}
\frac{\mathrm{d} u}{\mathrm{d} t} = \frac{\Tilde{\stable}_\tok e^{2\sink_\tok}}{(e^{\sink_\tok}+\mass)^2} 
\sum_{\toki=1}^\vocabsize({\transition}_{\tok\toki}-{\ppred}_{\tok\toki})(\mass\ivalue_\toki-\mass_\toki\xi_\toki).
\end{align*}
We take in $\vecvalue = c\cdot \bm{1}$ and expand the expression of $\mathrm{d}u/\mathrm{d}t$. This gives us that
\begin{align*}
\frac{\mathrm{d} u}{\mathrm{d} t} = &~ \frac{\Tilde{\stable}_\tok u^2}{(u+\mass)^2} \frac{\sum_{\tokk=1}^\vocabsize \transition_{\tok\tokk} e^{\mass_\tokk\xi_\tokk/(u+\mass)}\mass_\tokk\xi_\tokk - \sum_{\tokk=1}^\vocabsize \transition_{\tok\tokk} e^{\mass_\tokk\xi_\tokk/(u+\mass)}\sum_{\tokk=1}^\vocabsize \transition_{\tok\tokk} \mass_\tokk\xi_\tokk}{\sum_{\tokk=1}^\vocabsize \transition_{\tok\tokk} e^{\mass_\tokk\xi_\tokk/(u+\mass)}}\\
= &~ \frac{\Tilde{\stable}_\tok u^2}{(u+\mass)^2} \frac{\Cov(e^{\frac{\varphi}{u+\mass}},\varphi)}{\E e^{\frac{\varphi}{u+\mass}}}.
\end{align*}
Since both $e^{x/(u+\mass)}$ and $x$ are monotonically increasing with respect to $x$, $\mathrm{d}u/\mathrm{d}t \geq 0$. Therefore, $u$ is monotonically increasing, and we have that
\[
\frac{u(t)^2}{[u(t)+\mass]^2} \geq \frac{u(0)^2}{[u(0)+\mass]^2},\quad \E e^{\frac{\varphi}{u(t)+\mass}} \leq \E e^{\frac{\varphi}{u(0)+\mass}}.
\]
Meanwhile, the first and second order Taylor expansions of $e^{\varphi/(u+\mass)}$ give that
\[
e^{\frac{\varphi}{u+\mass}} = 1 + \frac{\theta_1(\varphi) \varphi}{u+\mass},\quad e^{\frac{\varphi}{u+\mass}} = 1 + \frac{\varphi}{u+\mass} + \theta_2(\varphi) \Big[\frac{\varphi}{u+\mass}\Big]^2,
\]
where both $\theta_1(\varphi)$ and $\theta_2(\varphi)\varphi^2$ are  monotonically increasing functions of $\varphi$. We also have the bound that
\[
\theta(\varphi) \leq \Big[\exp\Big\{\frac{\max_\tokk \mass_\tokk \xi_\tokk}{u(0)+\mass}\Big\}-1\Big]/\Big[\frac{\max_\tokk \mass_\tokk \xi_\tokk}{u(0)+\mass}-1\Big] = C_\theta.
\]
Therefore, we get two more inequalities:
\[
\Cov(\theta_1(\varphi)\varphi,\varphi)\leq C_\theta \E(\varphi^2),\quad \Cov(\theta_2(\varphi)\varphi^2,\varphi)\geq 0.
\]
We bound $\mathrm{d}u/\mathrm{d}t$ and get that
\begin{align*}
\frac{\mathrm{d}u}{\mathrm{d}t} \leq &~ \Tilde{\stable}_\tok \Cov(e^{\frac{\varphi}{u+\mass}}, \varphi)\\
= &~ \Tilde{\stable}_\tok\Cov(1 + \frac{\theta_1(\varphi)\varphi}{u+\mass}, \varphi)\\
\leq &~ \frac{\Tilde{\stable}_\tok C_\theta \E(\varphi^2)}{u}.
\end{align*}
By solving the ODE, we get that
\begin{align*}
u \leq \sqrt{2 \Tilde{\stable}_\tok C_\theta \E(\varphi^2) t + C_1}.
\end{align*}
To give a lower bound, we have that
\begin{align*}
\frac{\mathrm{d}u}{\mathrm{d}t} \geq &~ \frac{u(0)^2}{[u(0)+\mass]^2}\frac{\Tilde{\stable}_\tok \Cov(e^{\frac{\varphi}{u+\mass}},\varphi)}{\E e^{\frac{\varphi}{u(0)+\mass}}} \\
= &~ \frac{u(0)^2}{[u(0)+\mass]^2}\frac{\Tilde{\stable}_\tok}{\E e^{\frac{\varphi}{u(0)+\mass}}} \Cov(1+\frac{\varphi}{u+\mass}+\theta_2(\varphi) \Big[\frac{\varphi}{u+\mass}\Big]^2,\varphi)\\
\geq &~ \frac{u(0)^2}{[u(0)+\mass]^2}\frac{\Tilde{\stable}_\tok}{\E e^{\frac{\varphi}{u(0)+\mass}}} \frac{\Var(\varphi)}{u+\mass}\\
\geq &~ \frac{u(0)^2}{[u(0)+\mass]^2}\frac{\Tilde{\stable}_\tok}{\E e^{\frac{\varphi}{u(0)+\mass}}} \cdot \frac{u(0)}{u(0)+\mass} \cdot \frac{\Var(\varphi)}{u}\\
= &~  \frac{\Tilde{C}}{u}.
\end{align*}
Therefore, $u \geq \sqrt{\Tilde{C} t + \Tilde{C}_2}$. In conclusion, we have that
\[
y_{\tok} = \log u = \frac{1}{2}\log t + r_{\tok},
\]
with $r_{\tok}$ bounded. This proves Theorem~\ref{appthm:main-2}.
\end{proof}
\subsection{Proof of Theorem~\ref{thm:main}(b): Value-state drains}\label{appsec:proof-main-2}
\begin{theorem}[Restatement of Theorem~\ref{thm:main}(b)]\label{appthm:main-3}
Assume $\xi_\tok \ge 0$ for any $\tok$, $\stable_\tok > 0$ for any $\tok\in\vocab$, and $\{ \mass_i \cdot \xi_i \}_{i \in \vocab}$ are not all equal. Fix $\vecsink = \sink \cdot \bm{1}$ for a constant $\sink$, define $\bar{\ivalue}(0) = V^\inv[\sum_{\tok} \ivalue_\tok(0)]$ and $\meanvalue=\vocabsize^\inv[\sum_{\tok} \mass_\tok \xi_\tok]$. Consider the gradient flow of the loss function $\loss(\vecsink, \vecvalue)$ over $\vecvalue$ for fixed $\vecsink$, i.e., $\dot{\vecvalue}(t) = - \nabla_{\vecvalue} \loss(\vecsink, \vecvalue(t))$. As $t \to \infty$, we have
    \begin{equation}\vecvalue(t) \to \vecvalue^\star = [\bar\ivalue(0)+e^{-\sink} \meanvalue] \cdot \bm{1} - e^{-\sink} \cdot \bm{\mass}\circ \bm{\xi}.\end{equation}
\end{theorem}
\begin{proof}[Proof of  Theorem~\ref{appthm:main-3}]
We plug $\vecvalue^\star$ into the $\loss$ and get that $\loss(\vecsink, \vecvalue^\star) = \sum_{\tok=1}^\vocabsize \Tilde{\stable}_{\tok} \sum_{\tokk=1}^\vocabsize \transition_{\tok\tokk} \log \transition_{\tok\tokk} $. Computing $\nabla^2_{\vecvalue} \loss(\vecsink,\vecvalue)$, we get that
\begin{equation*}
\nabla^2_{\vecvalue} \loss(\vecsink,\vecvalue) = \sum_{\tokk=1}^\vocabsize \Tilde{\stable}_\tokk \gppred^\Ppred_\tokk,
\end{equation*}
where $\gppred^\Ppred_\tokk$ is defined in Eq.~\eqref{appeqn:g-probs}.
Lemma~\ref{appthm:positive-definite} indicates that it is positive semi-definite. Therefore, we have that
\[
\loss(\vecsink, \vecvalue(t))\to \loss(\vecsink, \vecvalue^\star)~~\text{as }t\to \infty.
\]
We choose $\delta$ as defined in Eq.\eqref{appeqn:error-q}. When $t$ is sufficiently large, we have that
\[
\loss(\vecsink, \vecvalue(t))\leq \loss(\vecsink, \vecvalue^\star) + \frac{1}{\min_{\tokk\in\vocab\setminus\cT}\Tilde{\stable}_\tokk}\cdot 2 \delta^2. 
\]
The convexity further implies that for any $\Tilde{\vecvalue}=\theta \vecvalue(t) + (1-\theta)\vecvalue^\star$ ($\theta\in(0,1)$), we have that 
\[
\loss(\vecsink, \Tilde{\vecvalue})\leq \loss(\vecsink, \vecvalue^\star) + \frac{1}{\min_{\tokk\in\vocab\setminus\cT}\Tilde{\stable}_\tokk}\cdot 2 \delta^2. 
\]
Denote $\Tilde{\bm{\ppred}}_\tok={\bm{\ppred}}_\tok(\vecsink, \Tilde{\vecvalue})$ as $\bm{\ppred}$ evaluated on $(\vecsink,\Tilde{\vecvalue})$. Using the definition of the KL-divergence in Eq.~\eqref{appeqn:kl-divergence}, we have that
\[
\sum_{\tok=1}^\vocabsize \Tilde{\stable}_\tok KL(\bm{\transition}_\tok~||~\Tilde{\bm{\ppred}}_\tok) = \loss(\vecsink, \vecvalue(t))- \loss(\vecsink, \vecvalue^\star) \leq \frac{1}{\min_{\tokk\in\vocab\setminus\cT} \Tilde{\stable}_\tokk}\cdot 2 \delta^2.
\]
This further implies that $KL(\bm{\transition}_\tok~||~\Tilde{\bm{\ppred}}_\tok)\leq 2 \delta^2$ for any $\tok$. Using Pinsker's inequality, we get that
\[
\sum_{\tokk=1}^\vocabsize |\transition_{\tok\tokk}-\Tilde\ppred_{\tok\tokk}| = \| \bm{\transition}_\tok - \bm{\ppred}_\tok \|_{\text{TV}} \leq \sqrt{KL(\bm{\transition}_\tok~||~\Tilde{\bm{\ppred}}_\tok)/2} \leq \delta.
\]
Therefore, $\max_{\tok,\tokk}\abs{\transition_{\tok\tokk}-\Tilde\ppred_{\tok\tokk}}\leq \delta$. 
Lemma~\ref{appthm:g-ppred} gives that  $\sum_{\tok=1}^\vocabsize \dot{\ivalue}_\tok(t)=0$. Therefore, $\sum_{\tok=1}^\vocabsize \ivalue_{\tok}(t)/\vocabsize=\bar{\ivalue}(0)$. The choice of $\vecvalue^\star$ guarantees that $\bar{\ivalue}^\star=\bar{\ivalue}(0)$. This shows that $\vecvalue(t)-\vecvalue^\star \perp \bm{1}$. Using Lemma \ref{appthm:min-eigenvalue-q}, there exists $\omega>0$ such that
\[
(\vecvalue(t)-\vecvalue^\star)^\top \nabla^2_{\vecvalue} \loss(\vecsink,\vecvalue) (\vecvalue(t)-\vecvalue^\star) = (\vecvalue(t)-\vecvalue^\star)^\top \Big[\sum_{\tokk=1}^\vocabsize \Tilde{\stable}_\tokk \gppred^\Ppred_\tokk \Big] (\vecvalue(t)-\vecvalue^\star) \geq \frac{\omega}{2} \|\vecvalue(t)-\vecvalue^\star\|_2^2.
\]
Using Taylor expansion, we have that
\begin{align*}
\loss(\vecsink,\vecvalue^\star)-\loss(\vecsink,\vecvalue(t)) = &~ -\nabla_\beta \loss(\vecsink,\vecvalue(t)) (\vecvalue(t)-\vecvalue^\star) + \frac{1}{2}(\vecvalue(t)-\vecvalue^\star)^\top \nabla^2_{\vecvalue} \loss(\vecsink,\Tilde{\vecvalue}) (\vecvalue(t)-\vecvalue^\star)\\
\geq &~ -\nabla_\beta \loss(\vecsink,\vecvalue(t)) (\vecvalue(t)-\vecvalue^\star) + \frac{\omega}{2} \|\vecvalue(t)-\vecvalue^\star\|_2^2\\
\geq &~ - \frac{1}{2\omega} \| \nabla_\beta \loss(\vecsink,\vecvalue(t))\|_2^2.
\end{align*}
This shows that $\loss(\vecsink,\vecvalue(t))$ satisfies the Polyak-Lojasiewicz (PL) condition \citep{karimi2016linear} when $t$ is sufficiently large. 
This proves Theorem~\ref{appthm:main-3}. 
\end{proof}

\clearpage
\section{The Linear Growth of the Residual States}
\label{appsec:res-peak}

\subsection{The minimal model structure to recapitulate residual state peak}
\label{appsec:mini-res-peak}
We give more details for the claim in Section~\ref{sec:res-peak}, stating that ``The residual-state peaks require a three-layer structure.''
Figure~\ref{appfigure:massive_minimal} presents the difference of residual norms between the \bos~token and others ($\|\res_{\bos}\|-\E_{\tok\neq\bos}[\|\res_{\tok}\|]$), with different combinations of model structures. The $3\times \TF$ and $2\times \TF+\mlp$ are the architectures that demonstrate clear evidence of residual state peaks. 

\begin{figure}[h]
    \centering
    \includegraphics[width=0.3\linewidth]{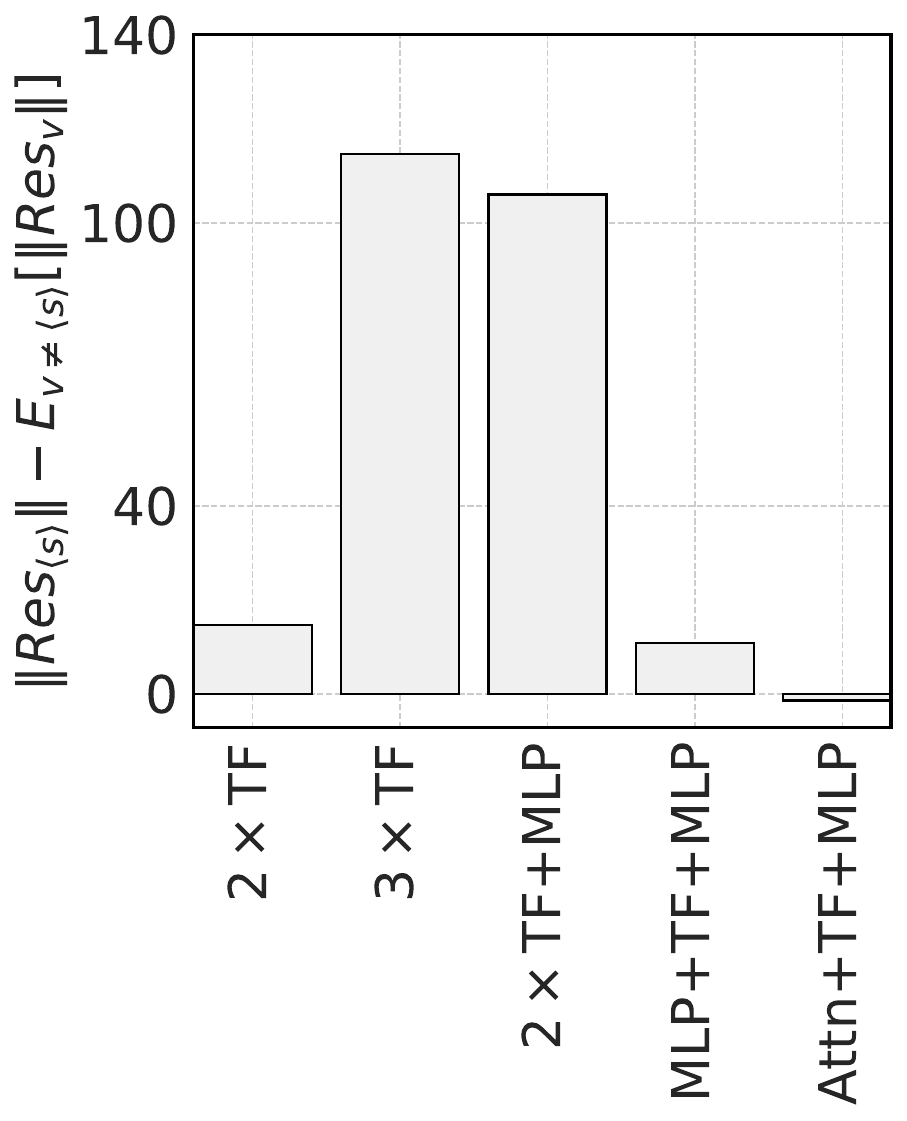}
    \caption{\small Minimal structures to elicit residual state peaks. We use $A+B+C$ to indicate the model with structure $A$, $B$, $C$ in layers 0, 1, and 2, respectively.}
    \label{appfigure:massive_minimal}
\end{figure}

\subsection{Additional plots for the three-layer transformer trained on BB task}
\label{appsec:three-layer-tf}
We provide more results to the three layer transformer model trained on the BB task. They provide supporting evidence for the claim in Section~\ref{sec:res-peak}, stating that “Massive residual states amplify attention sinks and value-state drains in later layers.”
Figures \ref{appfigure:massive-attn}, \ref{appfigure:massive-value-norm}, and \ref{appfigure:massive-norm} show the extreme token phenomena in a three-layer transformer. The residual state peaks show different phenomena from those in LLMs, with the last layer output increasing the residual norms of non-\bos~tokens. Figure \ref{figure:extreme-token}
 demonstrates that the residual state norms of \bos~drop match the magnitudes of other tokens at the last layer. 
\begin{figure}[h]
  \centering
  \begin{minipage}{0.3\textwidth}
      \centering
      \subcaption{\small Layer 0}
      \vspace{-.2em}
      \includegraphics[width=\linewidth]{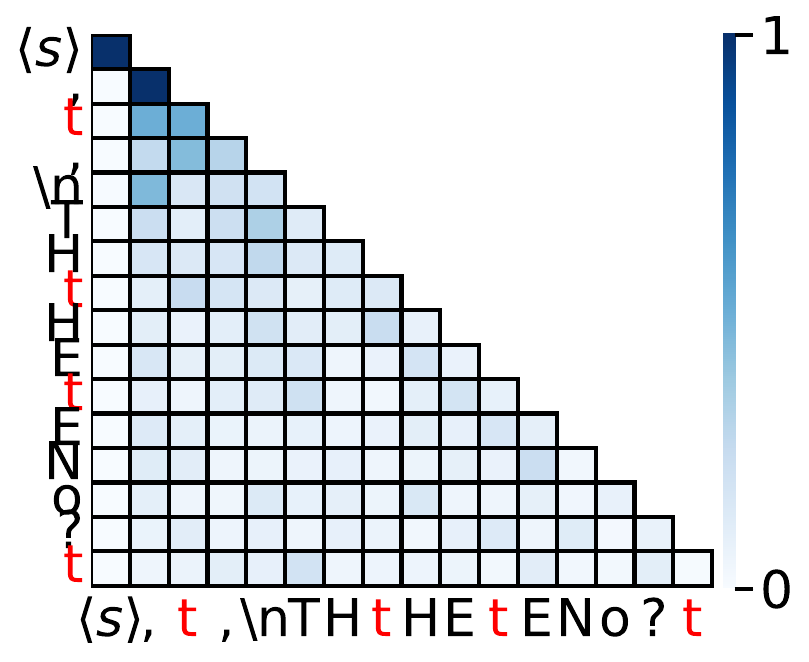}
  \end{minipage}
  \begin{minipage}{0.3\textwidth}
      \centering
      \subcaption{\small Layer 1}
      \vspace{-.2em}
      \includegraphics[width=\linewidth]{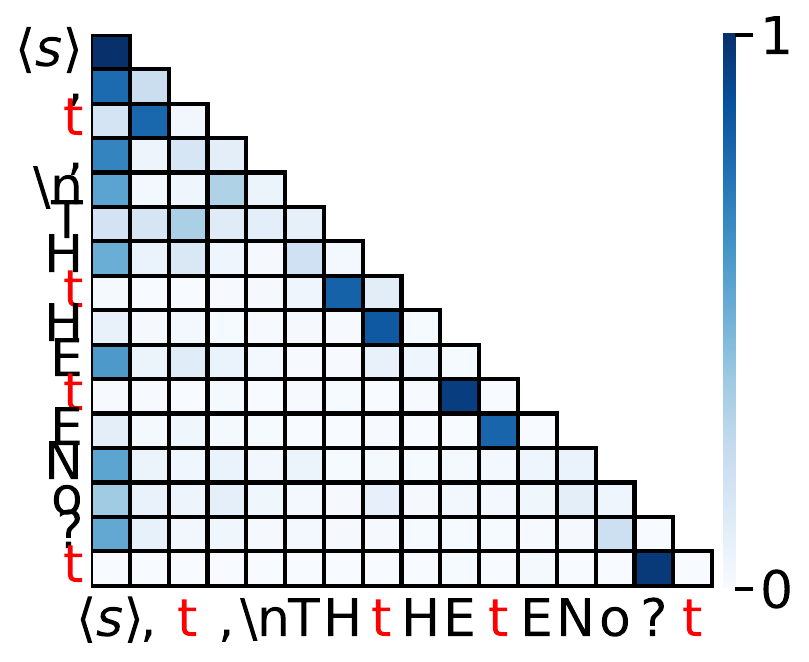}
  \end{minipage}
  \begin{minipage}{0.3\textwidth}
      \centering
      \subcaption{\small Layer 2}
      \vspace{-.2em}
      \includegraphics[width=\linewidth]{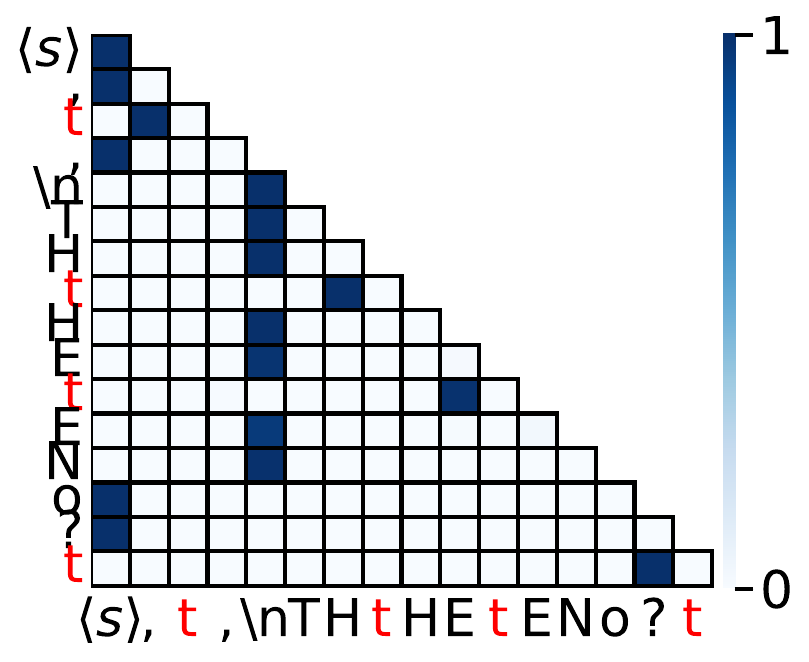}
  \end{minipage}
  \caption{\small Attention weight patterns of three-layer transformer trained on the BB task}
  \label{appfigure:massive-attn}
  \vspace{-1em}
\end{figure}

\begin{figure}[h]
  \centering
  \begin{minipage}{0.3\textwidth}
      \centering
      \subcaption{\small Layer 0}
      \vspace{-.2em}
      \includegraphics[width=\linewidth]{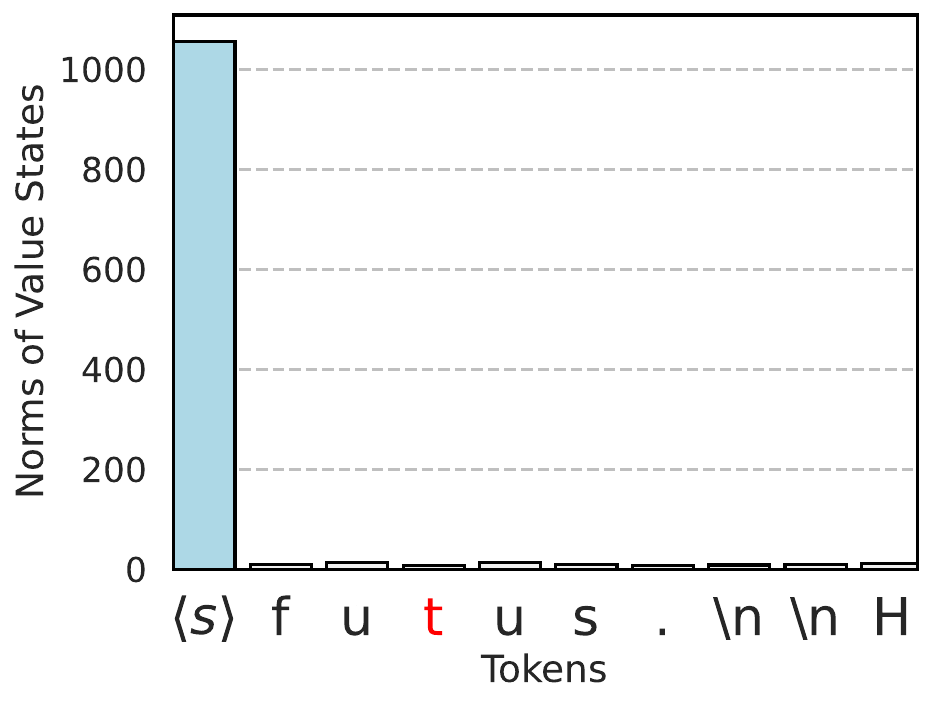}
  \end{minipage}
  \begin{minipage}{0.3\textwidth}
      \centering
      \subcaption{\small Layer 1}
      \vspace{-.2em}
      \includegraphics[width=\linewidth]{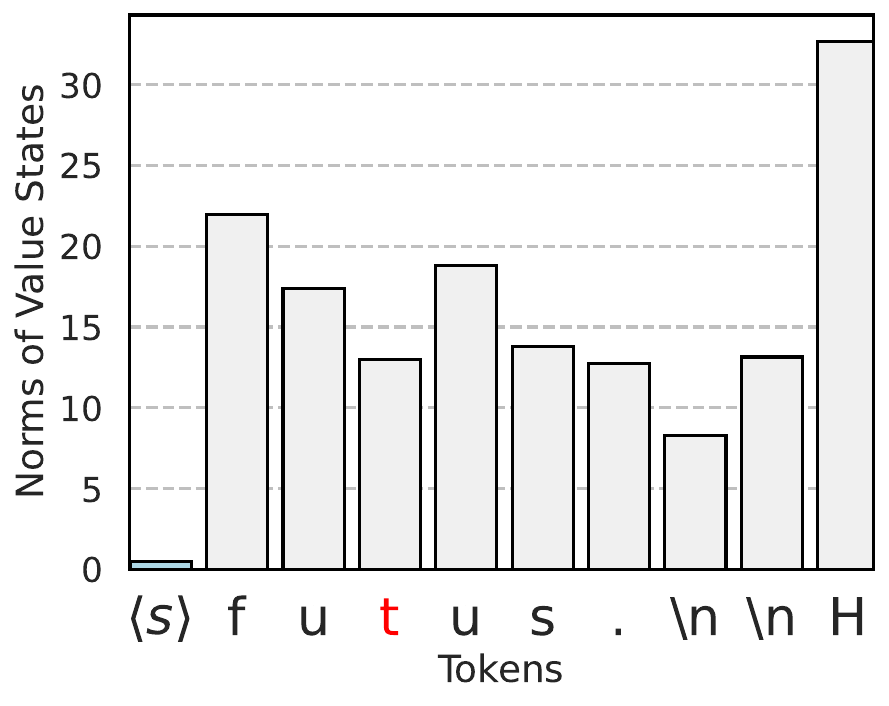}
  \end{minipage}
  \begin{minipage}{0.3\textwidth}
      \centering
      \subcaption{\small Layer 2}
      \vspace{-.2em}
      \includegraphics[width=\linewidth]{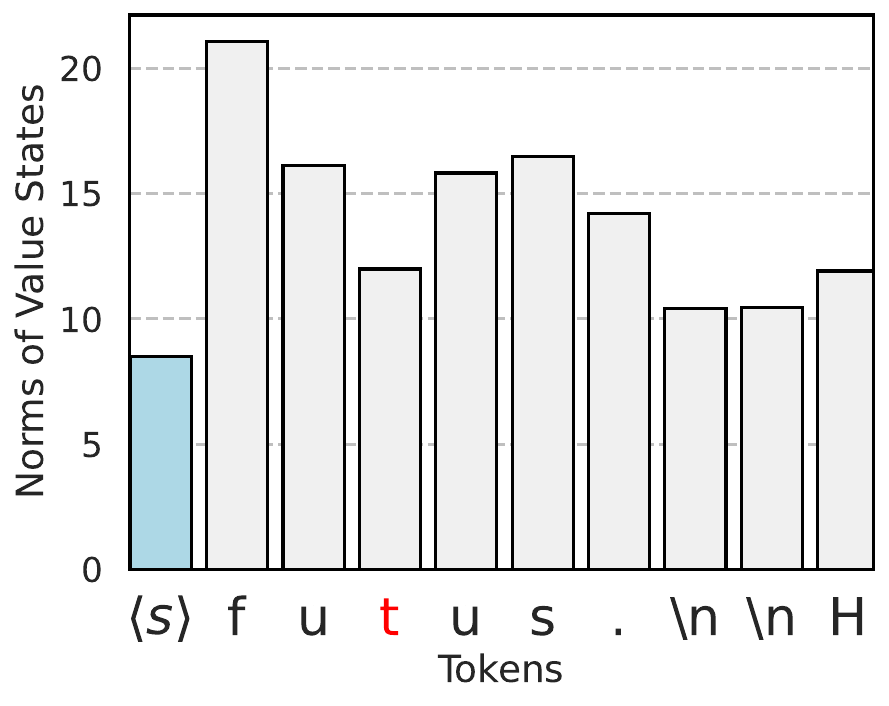}
  \end{minipage}
  \caption{\small Value state norms of three-layer transformer trained on the BB task}
  \label{appfigure:massive-value-norm}
  \vspace{-1em}
\end{figure}

\begin{figure}[h]
  \centering
  \begin{minipage}{0.3\textwidth}
      \centering
      \subcaption{\small Layer 0}
      \vspace{-.2em}
      \includegraphics[width=\linewidth]{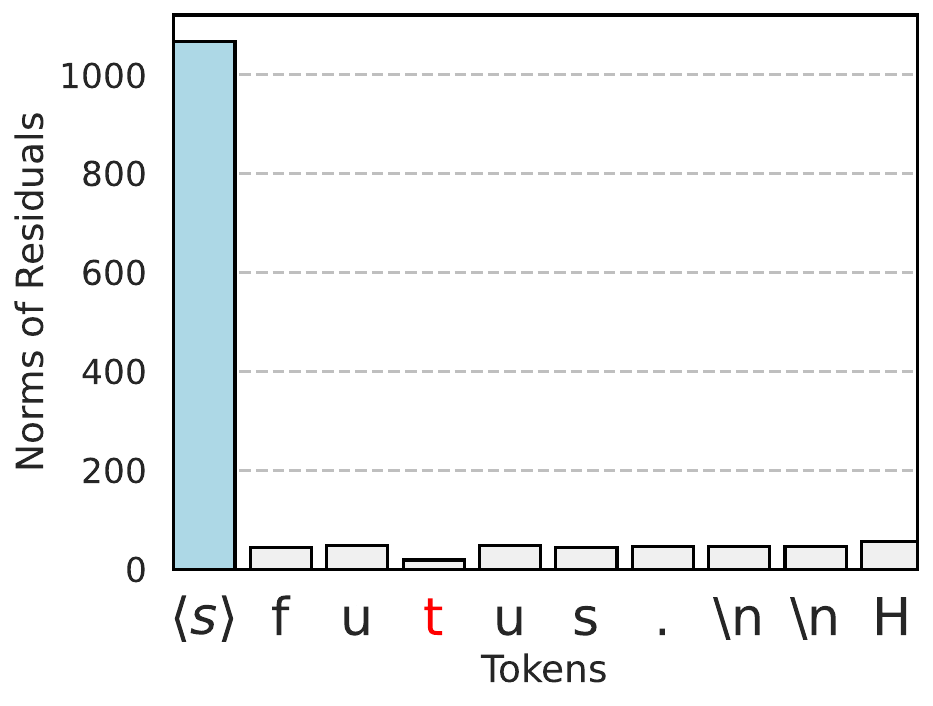}
  \end{minipage}
  \begin{minipage}{0.3\textwidth}
      \centering
      \subcaption{\small Layer 1}
      \vspace{-.2em}
      \includegraphics[width=\linewidth]{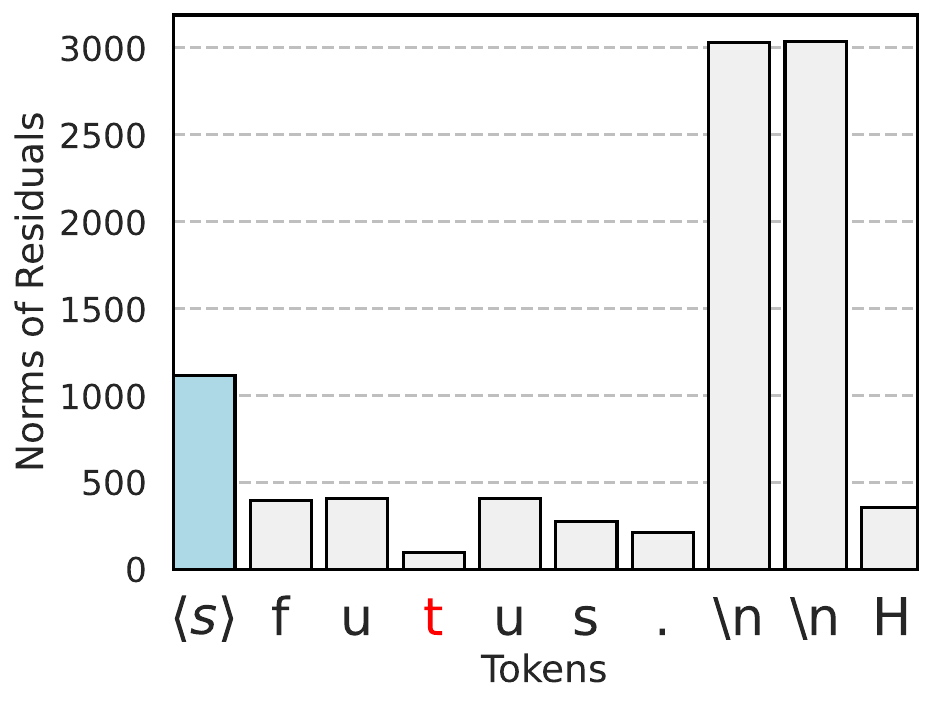}
  \end{minipage}
  \begin{minipage}{0.3\textwidth}
      \centering
      \subcaption{\small Layer 2}
      \vspace{-.2em}
      \includegraphics[width=\linewidth]{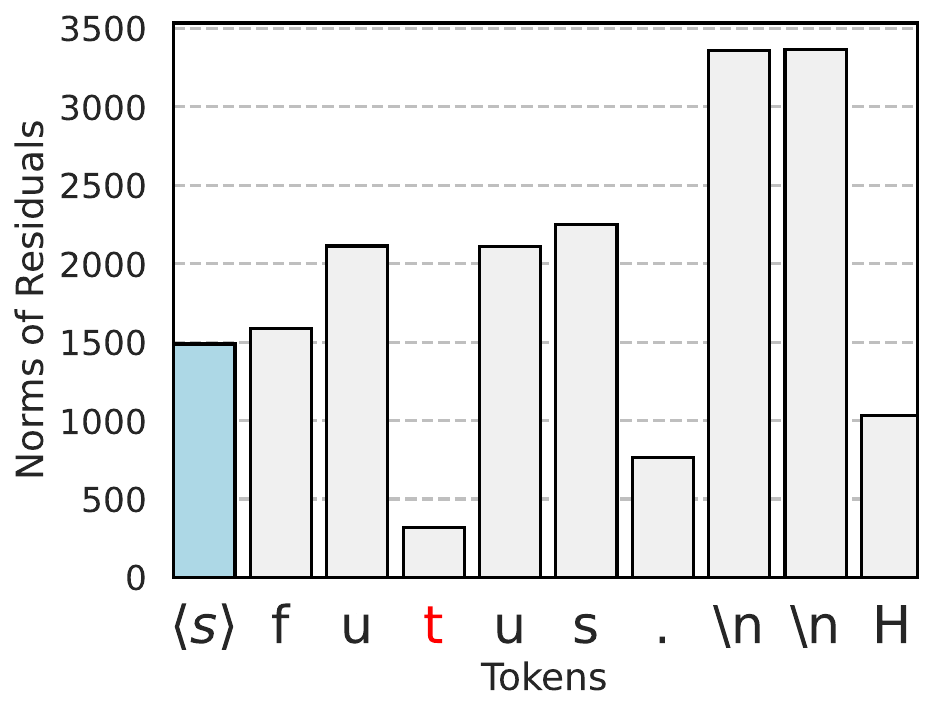}
  \end{minipage}
  \caption{\small Residual state norms of three-layer transformer trained on the BB task}
  \label{appfigure:massive-norm}
  \vspace{-1em}
\end{figure}

\subsection{Potential mechanism for linear growth of the residual state peak in multi-layer models}
\label{appsec:theory-for-res}
We give more details for the claim in Section~\ref{sec:res-peak}, stating that ``The ReLU attention and changing Adam to SGD eliminates the residual state peaks'' We first state Claim~\ref{sec:res-peak}.
\begin{claim}[Potential mechanism for the formation of residual-state peaks]\label{claim:res-peak}
In the training dynamic of a multi-layer transformer, if the mutual reinforcement mechanism (cf.\ Claim~\ref{claim:mutual-reinforcement}) occurs in upper layers:
\begin{enumerate}
    \item The gradients of $\res_\bos$ have the same direction (aligning with the null space of value matrices in upper layers and the $\key_\bos$) along the training dynamics.
    \item The layer-norm operations cause the fast decay of the magnitude of the gradients.
    \item Adam induces diminishing gradients to be constant updates, leading to the linear growth for the norm of the residual state of the extreme token.
\end{enumerate}
\end{claim}
To support the claim, we use the simplified model in Section~\ref{sec:bb_task}, including the residual state norm.
Denote the layer-norm operation as $\lnm$.
Heuristically, we can split the residual state $\res_{\bos}$ to a summation of two directions.
\[
\res_{\bos} = m \cdot \bm{\eta} + \bm{\eps},
\]
where $\bm{\eta},\bm{\eps} \in \R^\vocabsize$ with $\|\bm{\eta}\|_2=\|\bm{\eps}\|_2 = 1$, and $\bm{\eta}^\top \bm{\eps} = \rho>0$. The $\bm{\eta}$ corresponds to the direction of $\key_\bos$ in the original transformer, and $\bm{\eps}$ corresponds to other directions.
Assume that the attention logit from the token $\tok$ to the \bos~token in layer 1 is given by 
\begin{equation}\label{eqn:logits_in_residual_massive} \text{logit}_{\tok,\bos}= \sink_\tok = \Tilde{\sink}_\tok \bm{\eta}^\top \lnm(\res_{\bos})=\Tilde{\sink}_\tok \cdot \frac{m+\rho}{\sqrt{m^2+2m\rho+1}}.\end{equation}
We assume that the scalars $m$ and $\Tilde{\sink}$ are trainable, quantifying the norm of the residual states and magnitude of attention sinks. In the loss function $\loss_\tok$ as defined in Eq.~\eqref{eqn:loss_single}, we replace $\sink_\tok$ by the expression as in Eq.~(\ref{eqn:logits_in_residual_massive}), so that the loss function becomes a function of $(\Tilde{\sink}_\tok, \vecvalue, m)$, denoted as
\[
\wt{\loss}_\tok(\Tilde{\sink}_\tok, \vecvalue, m) = \loss_\tok(\sink_\tok, \vecvalue),
\]
We then consider the total loss as the average of the losses on each non-trigger token, weighted by its proportion in the stable distribution $\{\pi_v\}_{v \in \vocab}$, given by
\begin{equation}\label{appeqn:res_total_loss}
\wt{\loss}(\Tilde{\vecsink}, \vecvalue, m) = \sum_{\tok \in \vocab \setminus \cT} \stable_\tok \cdot \wt{\loss}_\tok(\Tilde{\sink}_\tok, \vecvalue, m).
\end{equation}
\begin{proposition}\label{appthm:massive}
Assume $\xi_\tok\geq 0$ for any $\tok$, $\set{W_k\ivalue_k}_{k\in\vocab}$ are not all equal, and $\rho>0$. Fix $\vecvalue=\bm{0}$, and consider the gradient flow of $\wt{\loss}(\Tilde{\vecsink}, \vecvalue, m)$ over $\Tilde{\vecsink}$ and $m$. With any initial value $\Tilde{\sink}_\tok(0)>0$ for any $\tok$ and $m(0)>0$, we have that
\[
\dot{m}(t)=O\Big(\frac{\log t}{\sqrt{t} m^{3}}\Big).
\]
\end{proposition}
\begin{proof}[Proof of Proposition~\ref{appthm:massive}]
The chain rule gives that
\[
\dot{\Tilde{\sink}}_\tok(t) = \dot{\sink}_\tok \cdot \frac{m+\rho}{\sqrt{m^2+2m\rho+1}},
\]
and
\[
\dot{m}(t) = \sum_{\tok=1}^\vocabsize \Big\{\dot{\sink}_\tok \Tilde{\sink}_\tok \cdot \frac{\mathrm{d} \lnm(\res_\bos)}{\mathrm{d} t} \Big\}.
\]
With the initial values, $\dot{m}(t)\geq 0$ and $\dot{\Tilde{\sink}}_\tok(t) \geq 0$.
We have $m(t)\geq 0$ for any $t$. Hence, $$\dot{\Tilde{\sink}}_\tok \in [\rho \dot{\sink}_\tok,  \dot{\sink}_\tok].$$ 
Therefore, $\Tilde{\vecsink} = 2^\inv \log t \bm{1} + \Tilde{\bm{r}}(t)$ with $\Tilde{\bm{r}}(t)$ uniformly bounded over time. Furthermore, we have that
\begin{align*}
\dot{m}(t) & = \sum_{\tok=1}^\vocabsize \Big\{\dot{\sink}_\tok \Tilde{\sink}_\tok \cdot \frac{\mathrm{d} \lnm(\res_\bos)}{\mathrm{d} t} \Big\}\\
& = O\Big(\frac{\log t}{\sqrt{t}}\Big) \cdot \frac{1-\rho^2}{(m^2+2m\rho+1)^{3/2}}\\
& = O\Big(\frac{\log t}{\sqrt{t}m^3}\Big).
\end{align*}
This proves Proposition~\ref{appthm:massive}.
\end{proof}
We use simulation to demonstrate the effect of Adam. We train the scalar $m$ using Adam with gradient $\mathrm{d}m = \log t / [\sqrt{t}m^3]$. We set $\beta_1=0.9$, $\beta_2=0.999$, weight decay$=10^{-8}$, and the learning rate $\text{lr}=0.3$. Figure~\ref{appfigure:m_dynamics} presents the training dynamics of $m$. We observe the linear growth after a warming-up phase. In contrast, when trained by SGD with learning rate $\text{lr}=0.3$, $m$ remains small. The results match transformer models on BB-task as in Figure~\ref{fig:sgd}.

\begin{figure}[h]
    \centering
    \includegraphics[width=0.5\linewidth]{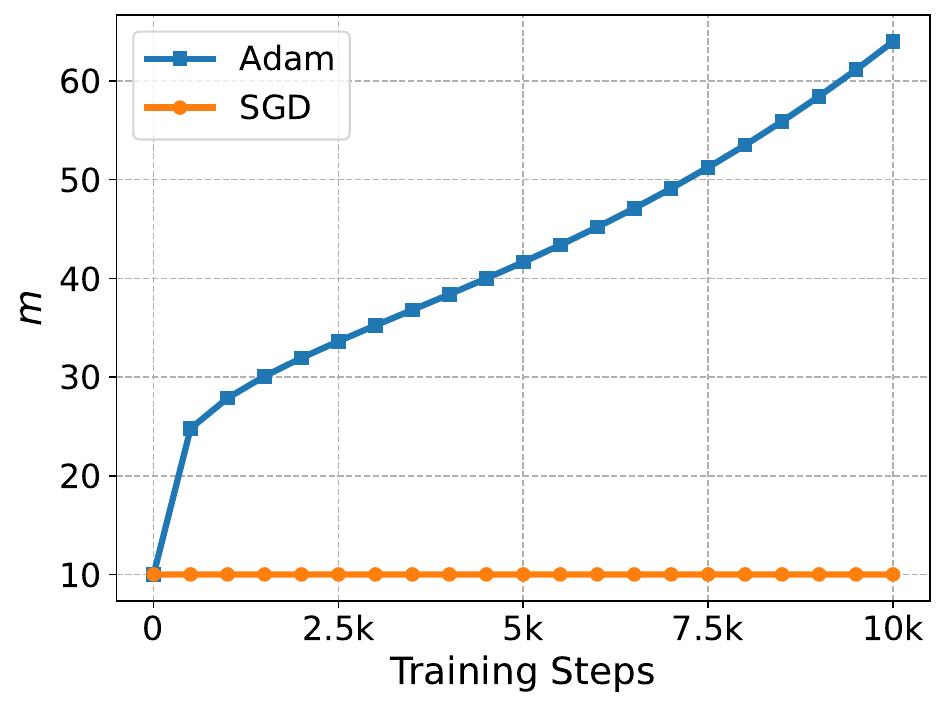}
    \caption{With the gradient formula in Proposition~\ref{appthm:massive}, Adam causes linear growth of $m$. 
    }
    \label{appfigure:m_dynamics}
\end{figure}
\clearpage
\section{Ablations}
\label{sec:ablations}


\subsection{Experimental details}\label{appsec:experiment-detail}
We provide more details for experiments in Section \ref{sec:bb_task}. We train transformers with positional embedding, pre-layer norm, $\softmax$ activation in \attn, and ReLU activation in \mlp. 
We use Adam with constant learning rate $0.0003$, $\beta_1=0.9$, $\beta_2=0.99$, $\eps=10^{-8}$, and a weight decay of $0.01$. We choose a learning rate of $0.03$ for the SGD. In each training step, we resample from the BB task with a batch size of $B=512$ and sequence length $N=256$. Unless otherwise specified, the model is trained for $10,000$ steps. Results are consistent across different random seeds.

\subsection{Additional attention plots of a 1-layer transformer trained on the BB task}
\label{appsec:additional-attn-maps}
We provide more attention plots of the 1-layer transformer on sequences other than those shown in Figure~\ref{fig:bbm-attn}.
Figure~\ref{appfigure:more-attn} presents more attention-weight heat maps of the one-layer transformer model trained on the BB task. All attention maps show the attention sink phenomenon. Some non-trigger tokens present attention patterns other than attention sink. For example, trigger tokens serve as attention sinks in some inputs in Figure~\ref{appfig:trigger-sink}.
\begin{figure}[h]
  \centering
  \begin{minipage}{0.3\textwidth}
      \centering
      \subcaption{\small Sequence 0}
      \vspace{-.2em}
      \includegraphics[width=\linewidth]{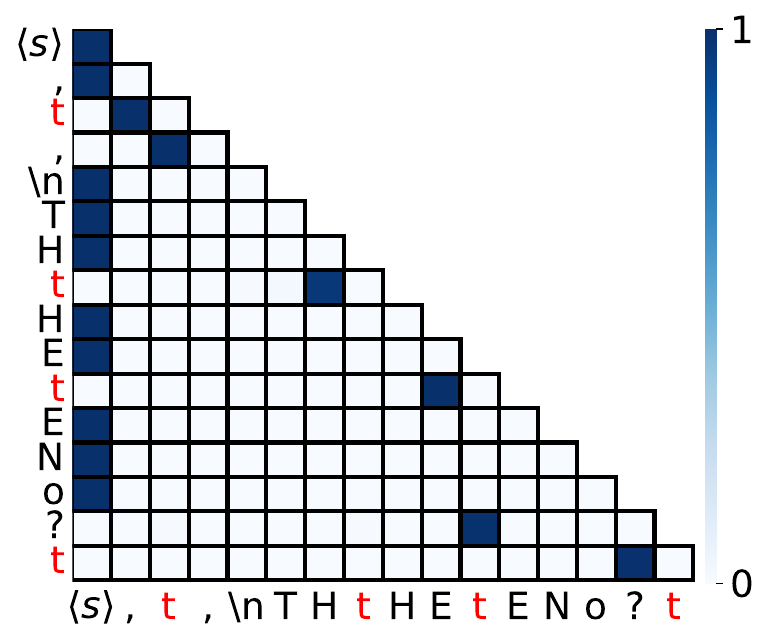}
  \end{minipage}
  \begin{minipage}{0.3\textwidth}
      \centering
      \subcaption{\small Sequence 1}
      \vspace{-.2em}
      \includegraphics[width=\linewidth]{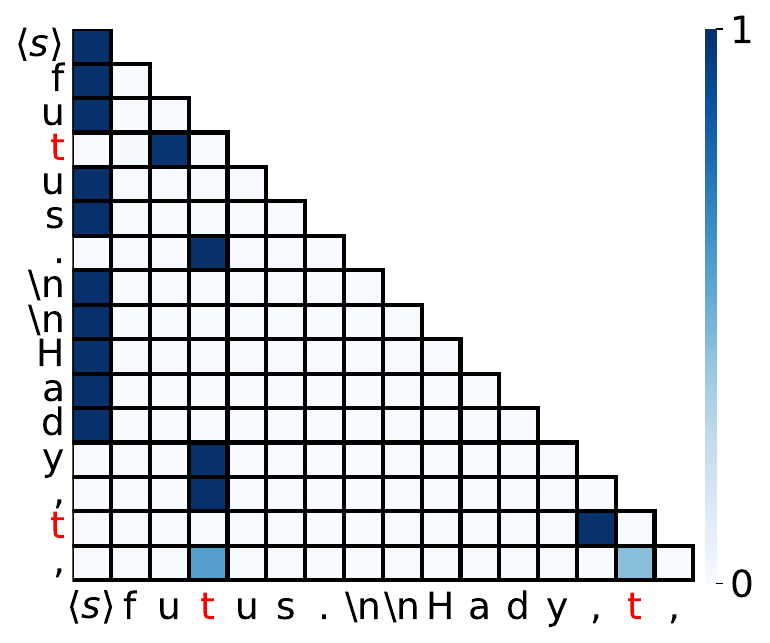}
  \end{minipage}
  \begin{minipage}{0.3\textwidth}
      \centering
      \subcaption{\small Sequence 2}
      \label{appfig:trigger-sink}
      \vspace{-.2em}
      \includegraphics[width=\linewidth]{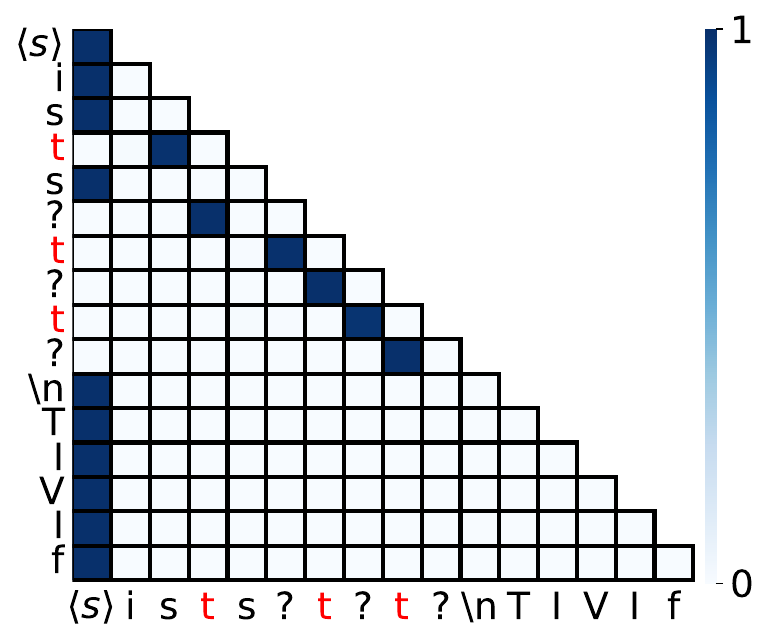}
  \end{minipage}
  \caption{\small Additional attention plots of the one-layer transformer trained on the Bigram-Backcopy task.}
  \label{appfigure:more-attn}
  \vspace{-1em}
\end{figure}

\subsection{Statics and dynamics of the simplified model in Theorem~\ref{thm:main}}
\label{appsec:train-simple}
We provide simulations that justify our model simplifications in Section~\ref{sec:bb_task}.
We pretrrain the simplified model structure in Figure \ref{figure:simple-model} with several modifications: (1) we use a trainable \mlp-layer with random Gaussian initialization; (2) we take $\vall_{\bos}=\bO \vecvalue$, with $\bO\in\R^{\vocabsize\times \vocabsize}$ and $\vecvalue\in \R^\vocabsize$. Both $\bO$ and $\vecvalue$ are trainable. Empirically, with a trainable \mlp~layer but without the trainable matrix $\bO$, $\vall_{\bos}$ becomes a non-negligible bias term instead of converging to zero. Collectively, we update parameters $\mlp$, $\bO$, $\vecsink$, $\vecvalue$, $\lambda$, and $\bm{\xi}$ using Adam with a learning rate of $0.03$. 
Figure \ref{appfigure:simple-static} and \ref{appfigure:simple-dynamic} present statics and dynamics that match the observations in the one-layer transformer.

\begin{figure}[h]
  \centering
  \begin{minipage}{0.3\textwidth}
      \centering
      \subcaption{\small Attention weights}
      \vspace{-.2em}
      \includegraphics[width=\linewidth]{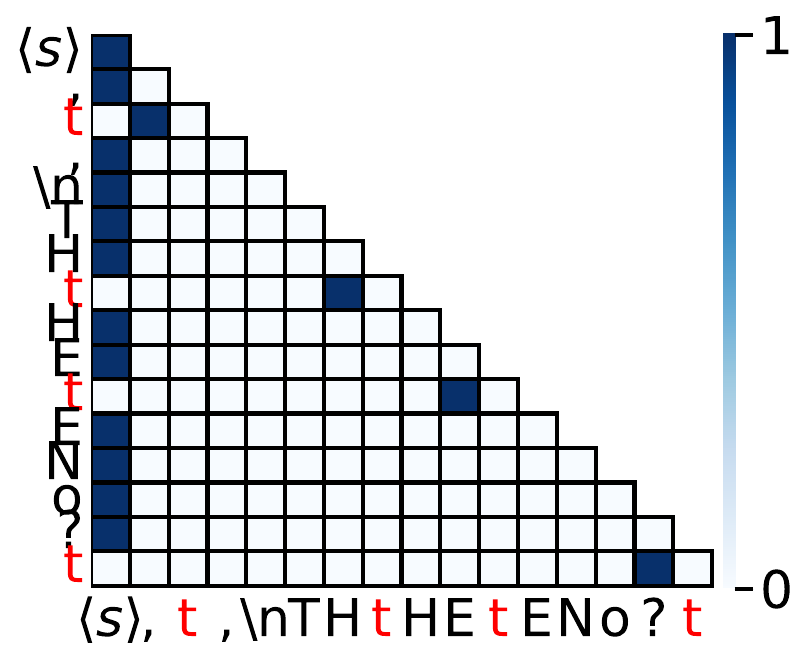}
  \end{minipage}
  \begin{minipage}{0.3\textwidth}
      \centering
      \subcaption{\small Value state norms}
      \vspace{-.2em}
      \includegraphics[width=\linewidth]{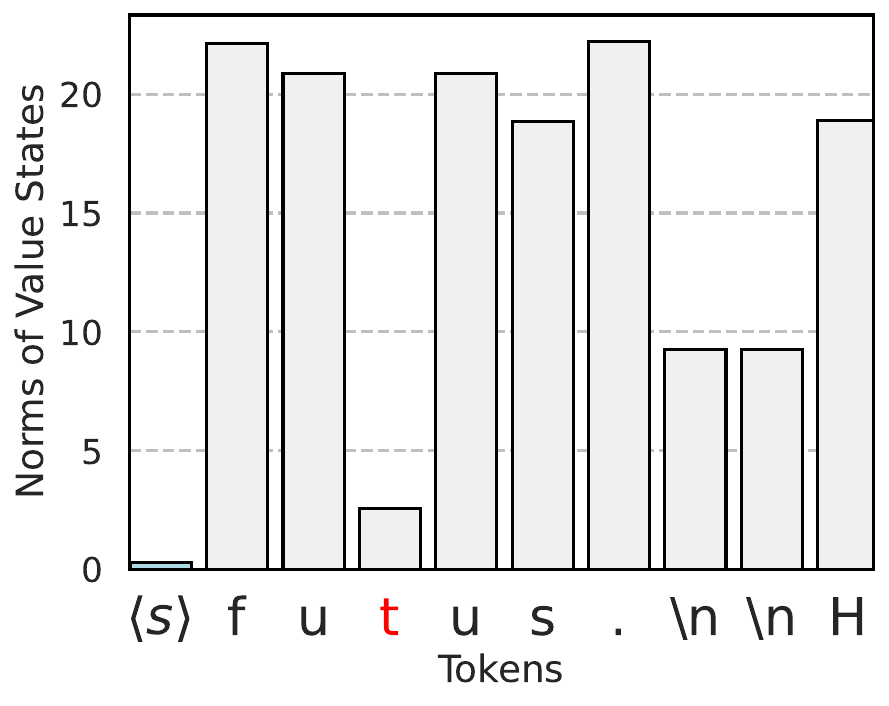}
  \end{minipage}
  \caption{\small The simplified model structure trained on the BB task.}
  \label{appfigure:simple-static}
  \vspace{-1em}
\end{figure}

\begin{figure}[h]
  \centering
      \includegraphics[width=.7\linewidth]{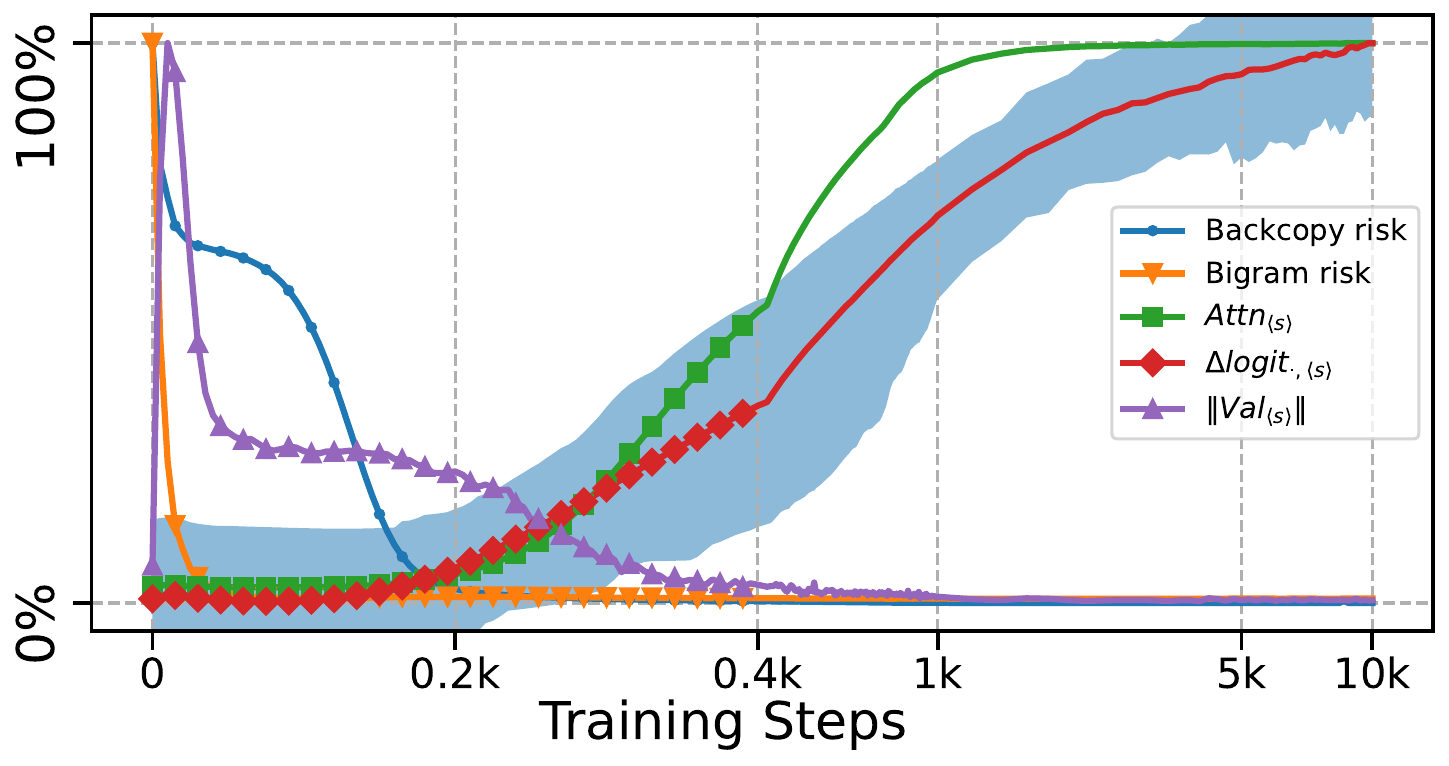}
  \caption{\small The dynamics of the simplified model structure trained on the BB task. The horizontal axis is logarithmatically scaled after steps $400$. The excess risk curves match the one-layer transformer. The logit curve is close to the logarithmic growth predicted in Theorem \ref{thm:main}.}
  \label{appfigure:simple-dynamic}
  \vspace{-1em}
\end{figure}

\subsection{The Bigram-Backcopy task without the \bos~token.}

We train a one-layer transformer on the BB task without the \bos~token. Figure \ref{appfigure:no-bos-extreme} shows that the zeroth token is not a sink token. Instead, trigger tokens and delimiter tokens seem to become sink tokens. In particular, the observation that delimiter tokens become extreme matches the observation in LLMs that delimiter tokens may also become extreme tokens (cf.\ \Cref{sub:fixed_bos}). 
\begin{figure}[h]
  \centering
  \begin{minipage}{0.35\textwidth}
      \centering
      \subcaption{\small Attention weights}
      \vspace{-.2em}
      \includegraphics[width=0.9\linewidth]{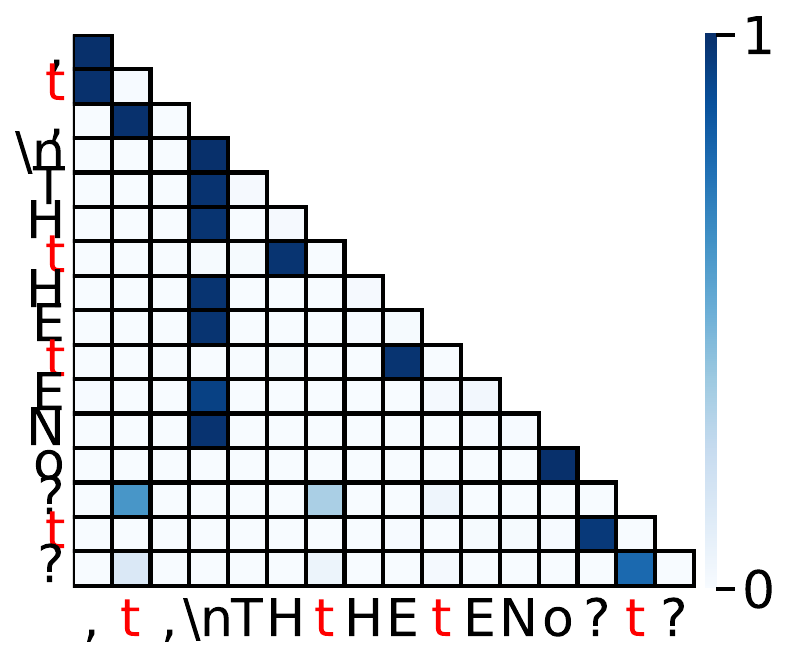}
  \end{minipage}
  \begin{minipage}{0.35\textwidth}
      \centering
      \subcaption{\small Value state norms}
      \vspace{-.2em}
      \includegraphics[width=\linewidth]{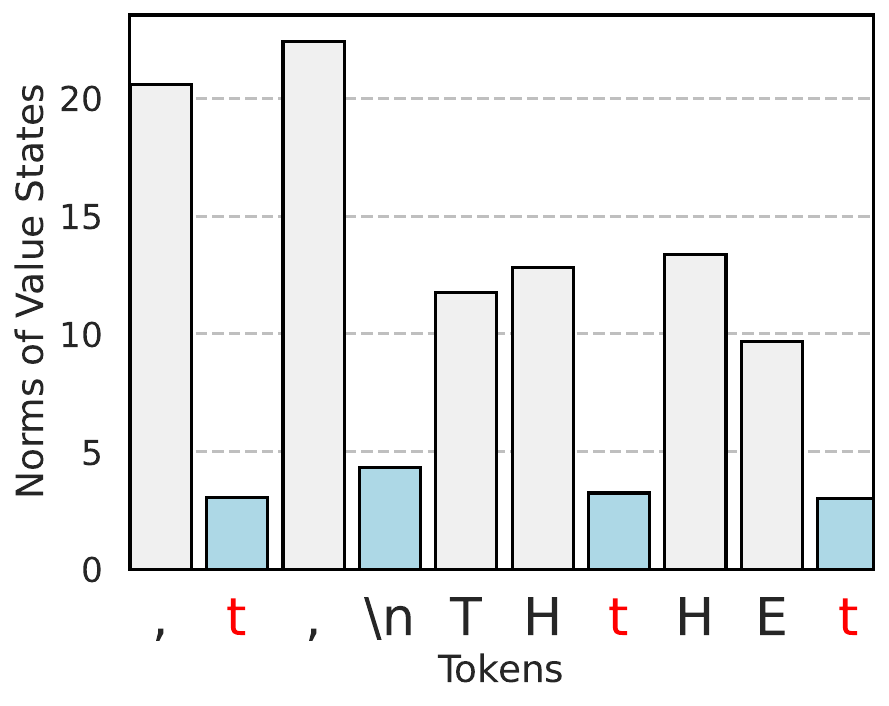}
  \end{minipage}
\caption{\small Attention weights and value state norms of a one-layer transformer trained on the BB task without the \bos~token.}
  \label{appfigure:no-bos-extreme}
  \vspace{-1em}
\end{figure}

\clearpage
\section{More Attention Heads in Dormant and Active Phase}\label{sec:more_heads}
We demonstrate a head with clear \textit{active-dormant mechanism} in Figure~\ref{fig:dormant_heads_domain_dependent}. 
In this section, we present two more active-dormant heads in Llama 2-7B-Base, in \Cref{fig:llama_l16h20,fig:llama_l16h28}, which are more difficult to interpret than Layer 16 Head 25, but become dormant on some inputs and remain active on others. 

\begin{figure}[h]
    \centering
    \begin{subfigure}{0.575\textwidth}
        \centering 
        \caption{Attention patterns}
        \includegraphics[width=\linewidth]{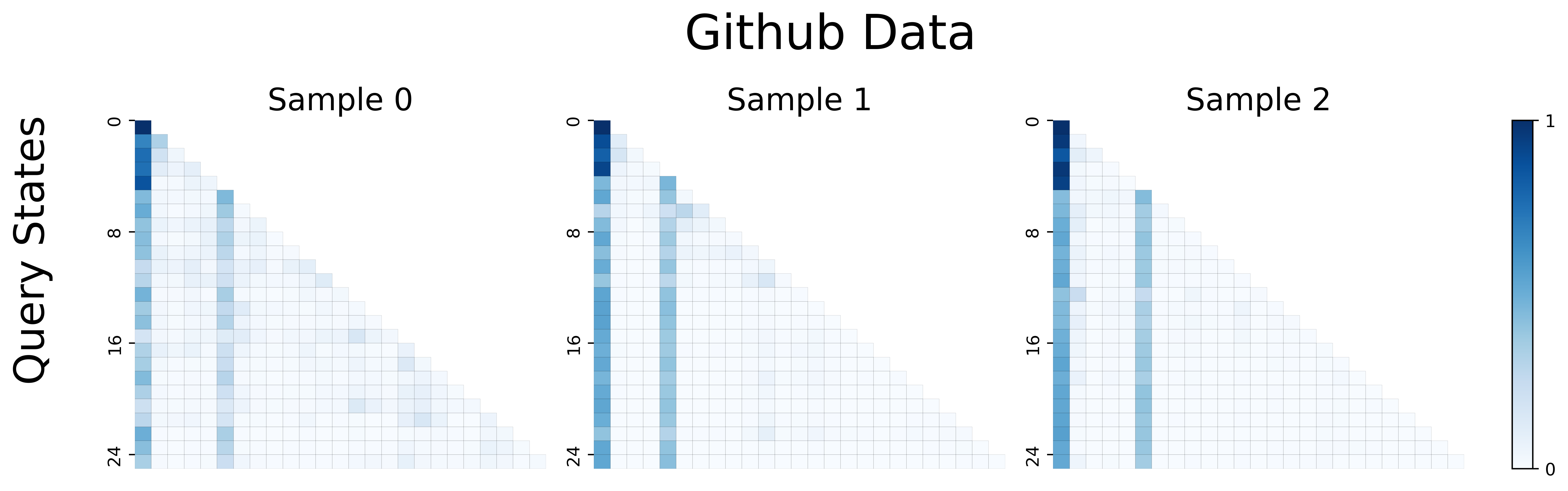}
        \includegraphics[width=\linewidth]{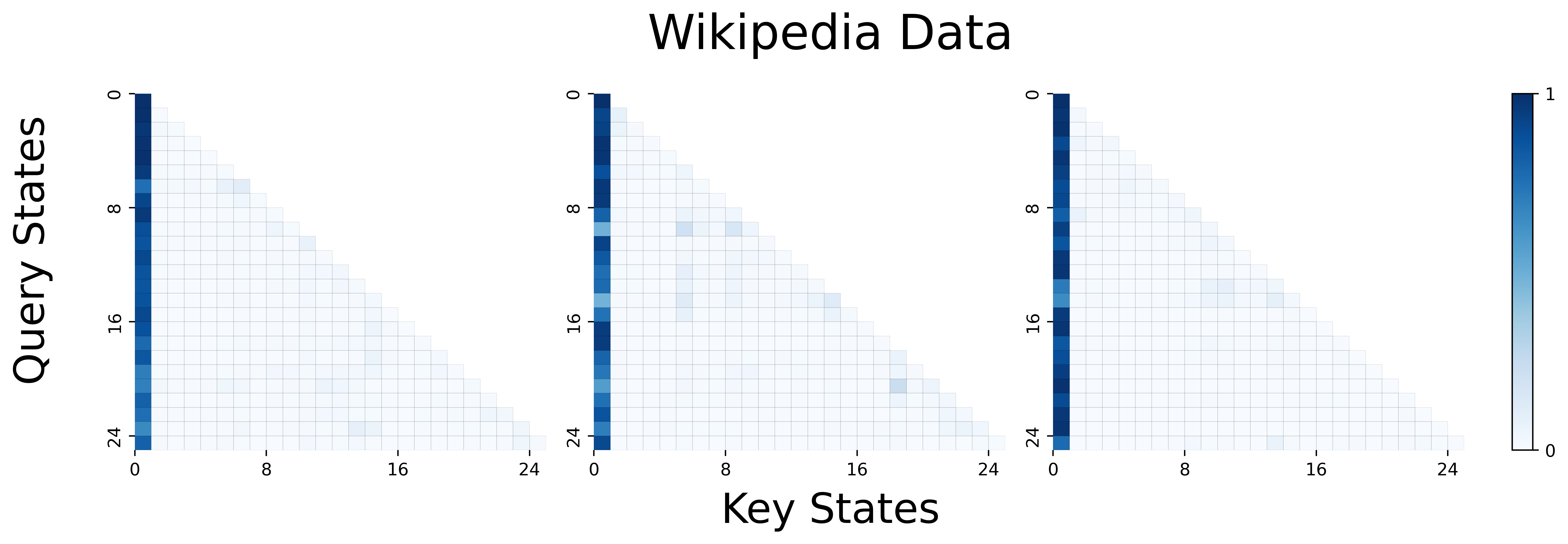}
    \end{subfigure}
    \hfill
    \begin{subfigure}{0.4\textwidth}
        \centering
        \caption{Interventions}
        \includegraphics[width=\linewidth]{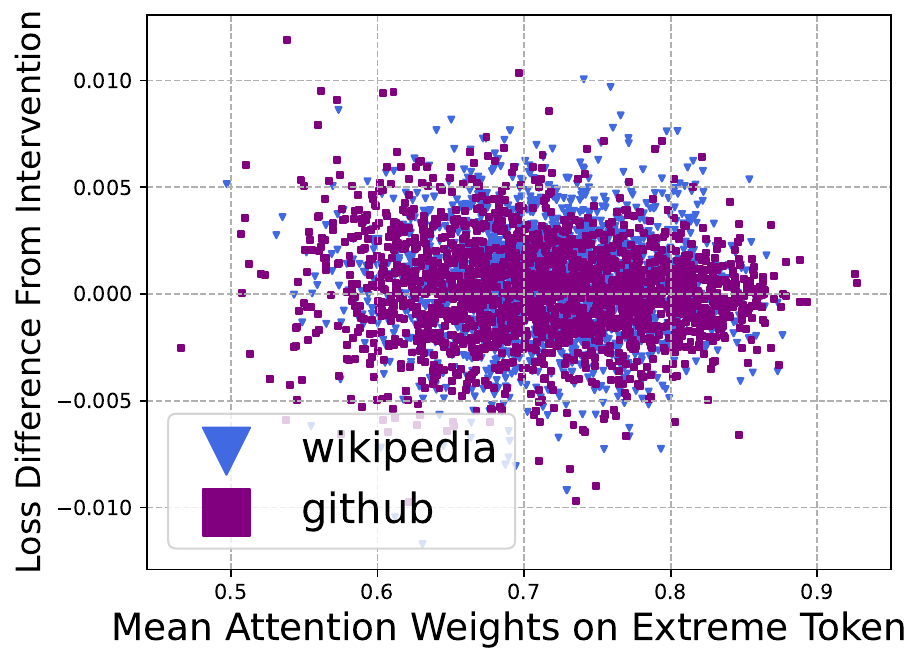}
    \end{subfigure}
    \caption{\small \textbf{Layer 16 Head 20 of Llama 2-7B-Base.} We do not observe difference between the Wikipedia data and the Github data.}
    \label{fig:llama_l16h20}
\end{figure}

\begin{figure}[h]
    \centering
    \begin{subfigure}{0.575\textwidth}
        \centering
        \caption{Attention patterns}
        \includegraphics[width=\linewidth]{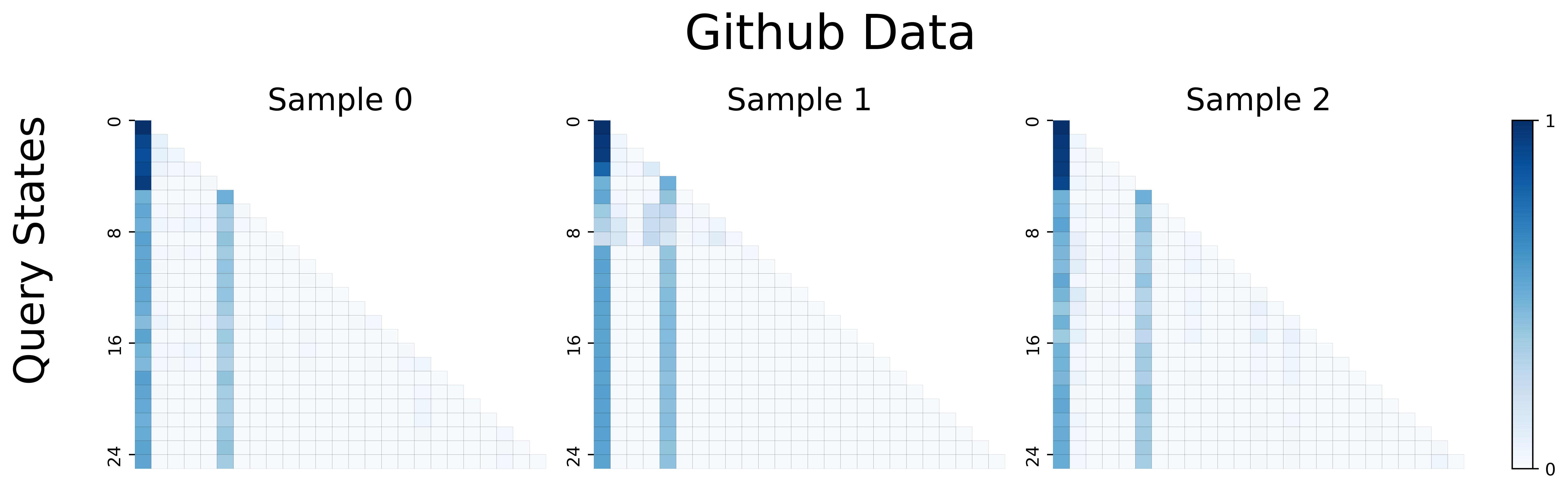}
        \includegraphics[width=\linewidth]{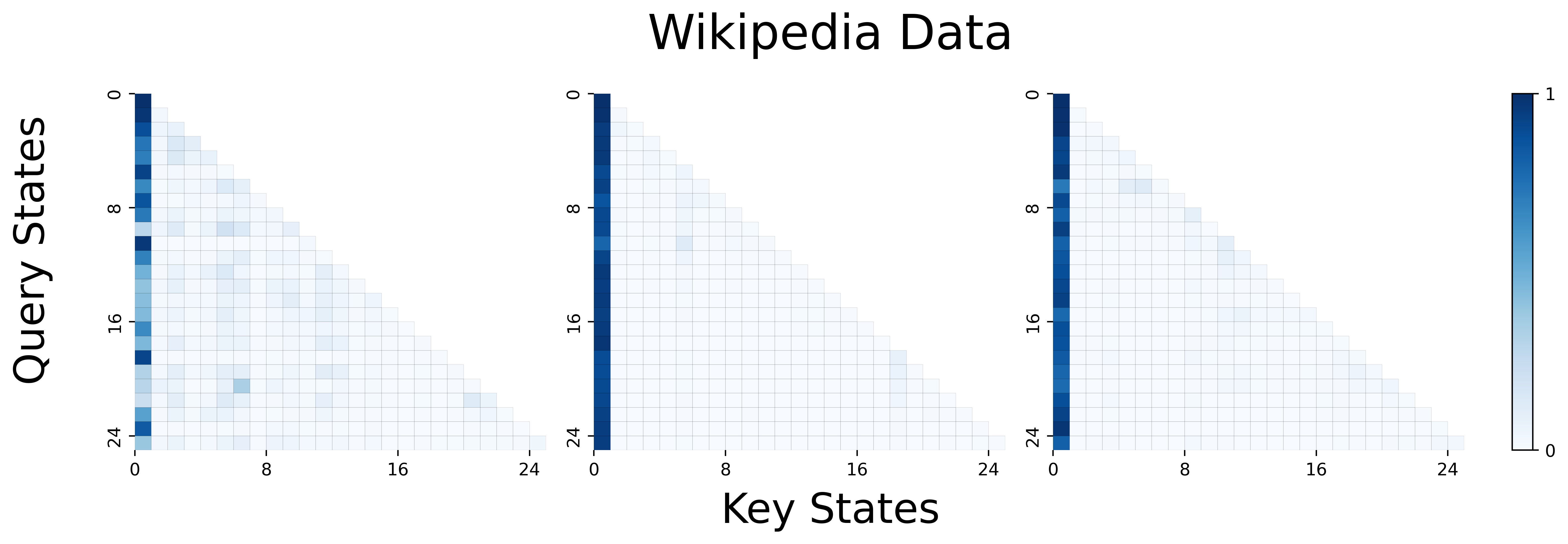}
    \end{subfigure}
    \begin{subfigure}{0.4\textwidth}
        \centering
        \caption{Interventions}
        \includegraphics[width=\linewidth]{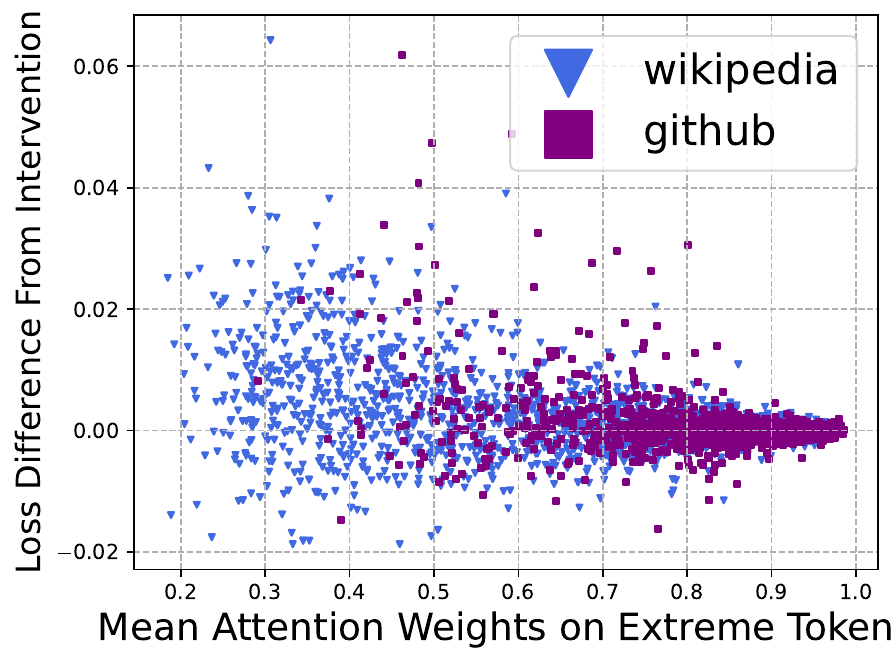}
    \end{subfigure}
    \caption{\small \textbf{Layer 16 Head 28 of Llama 2-7B-Base.} The head is more dormant on the GitHub data, and more active on the Wikipedia data.}
    \label{fig:llama_l16h28}
\end{figure}

\clearpage
\section{Fine-Grained Static Mechanisms for Extreme-Token Phenomena} \label{sec:circuit}

In this section, we identify more fine-grained static mechanisms for extreme-token phenomena in Llama 3.1-8B-Base. To do this, we identify circuits for the origin of attention sinks and small value states. Then, using ablation studies, we study the origin of massive norms. Again, we use the generic test phrase ``\bos{} Summer is warm. Winter is cold.''

\begin{figure}[h]
    \centering
    \includegraphics[width=0.8\textwidth]{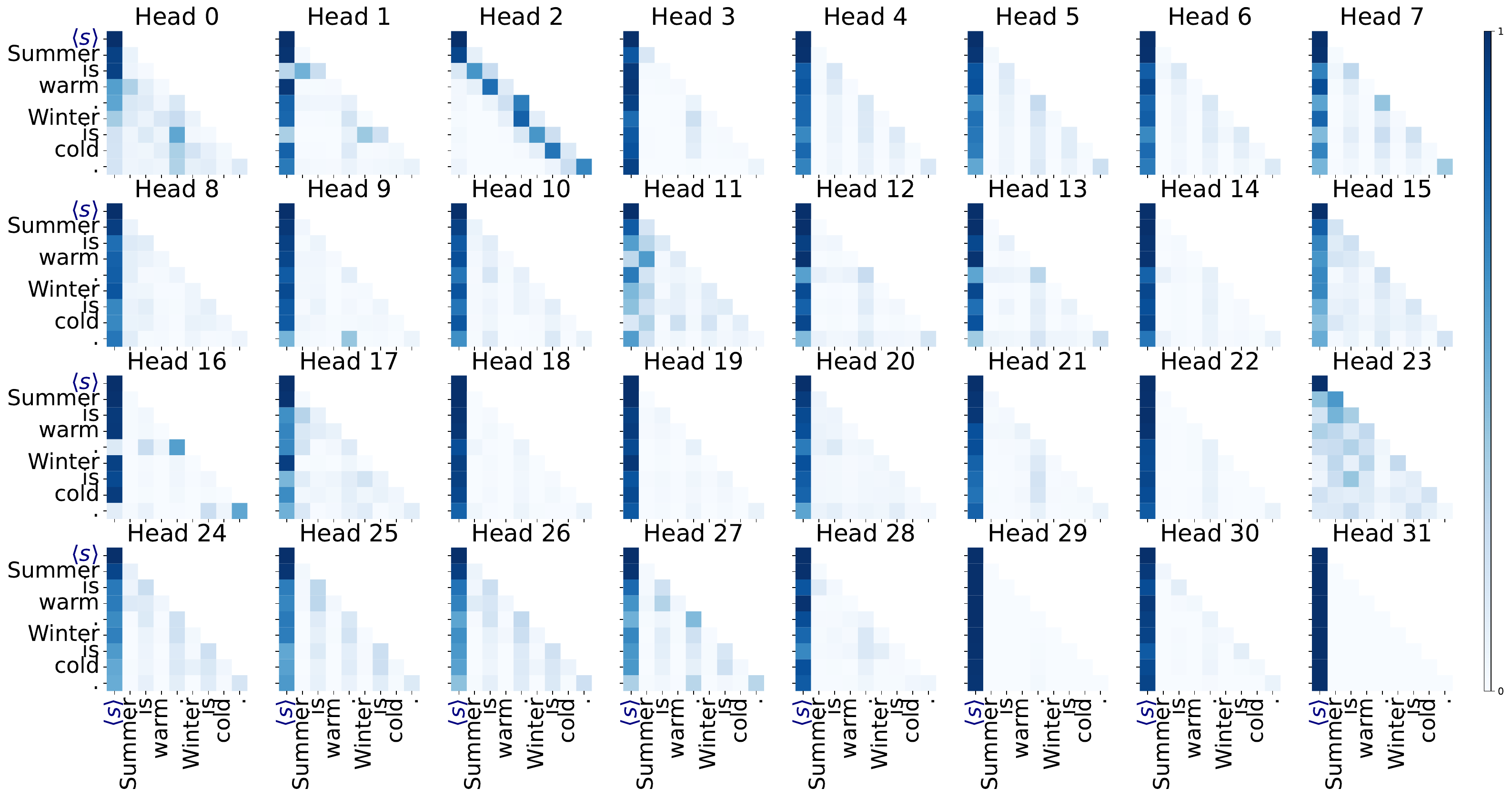}
    \caption{\small \textbf{A visualization of attention heads at Layer 0 of Llama 3.1-8B-Base.} Notice that many heads have the attention sink property, even at Layer 0 without any cross-token interaction. As usual, the test phrase is ``Summer is warm. Winter is cold.'' The most clear attention sink is Head 31.}
    \label{fig:llama_31_attn_layer0}
\end{figure}

\begin{figure}[h]
    \centering
    \begin{subfigure}[t]{0.4\textwidth}
    \caption{Correlations between key and query states.}
        \centering
        \includegraphics[width=0.6\textwidth]{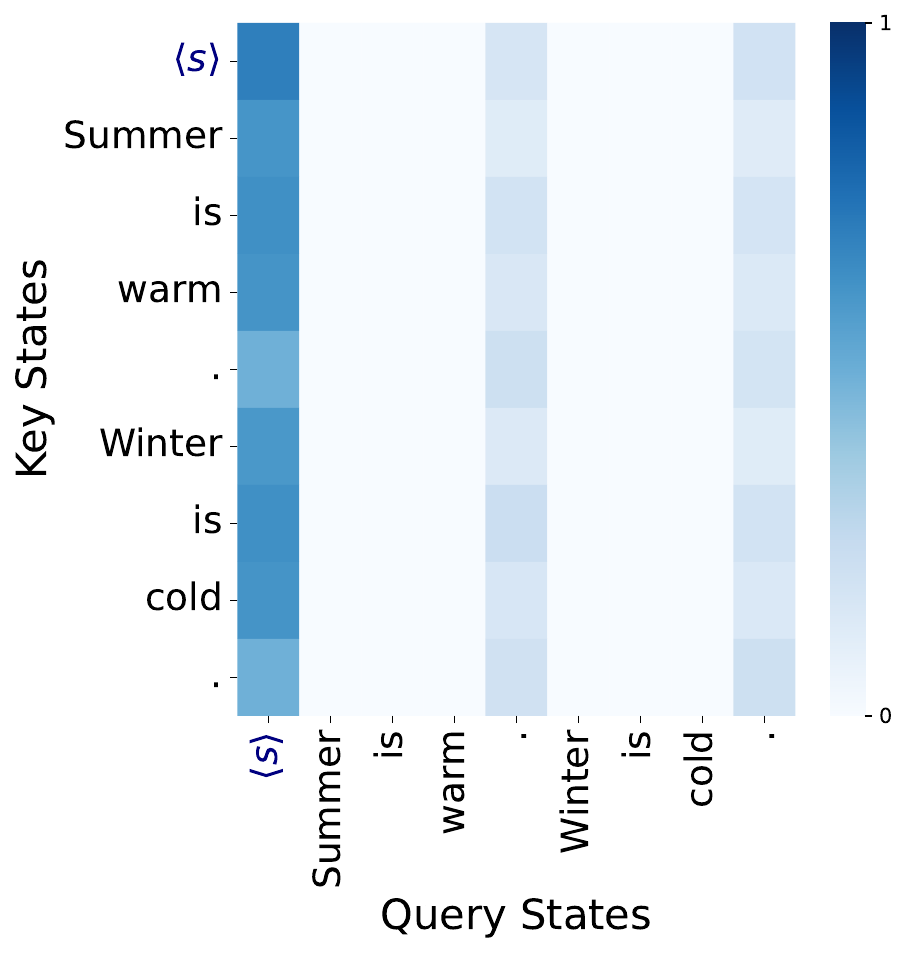}
    \end{subfigure}
    \quad
    \begin{subfigure}[t]{0.4\textwidth}
        \centering
        \caption{Correlations between key states.}
        \includegraphics[width=0.6\textwidth]{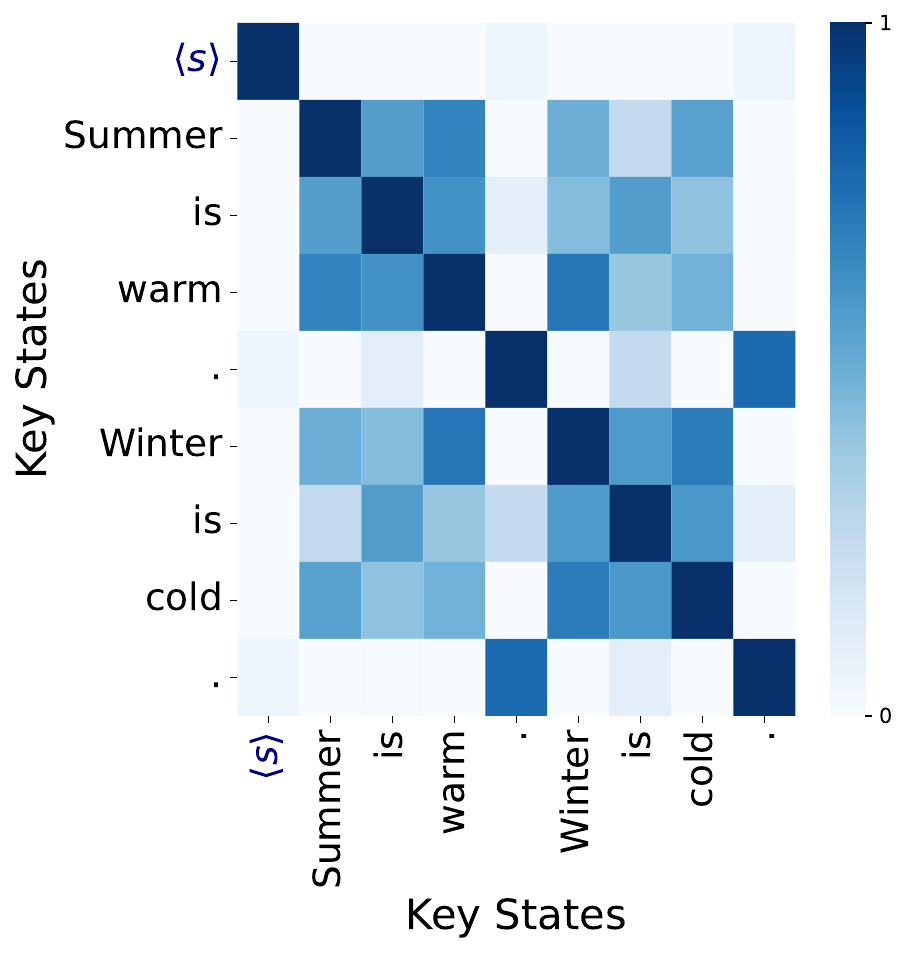}
    \end{subfigure}
    \caption{\small \textbf{Correlations between query states and key states at Layer 0 Head 31 of Llama 3.1-8B-Base.} We observe that the key state of \bos{} have low correlation with other key states, but high correlation with other query states. Meanwhile, all semantically meaningful (i.e., not delimiter) tokens have highly correlated key states.}
    \label{fig:llama_31_qk_kk}
\end{figure}

\paragraph{Attention sinks and global contextual semantics.} There are many attention heads that exhibit attention sinks at layer \(0\), and the \bos{} token is always the sink token (see \Cref{fig:llama_31_attn_layer0}). From now on until the end of this section, we restrict our attention to Head 31 of Layer 0, which is an attention sink. These attention sinks are caused by two linear-algebraic factors, demonstrated in \Cref{fig:llama_31_qk_kk}.
\begin{enumerate}
    \item The key state of the \bos{} token has small dot product with all other key states. 
    \item The query states of all tokens are nearly orthogonal to the key states of all tokens except the \bos{} token.
\end{enumerate}
 These two facts combine to ensure that the key state of the \bos{} token is picked out by each query state, causing the attention sink. Since these query and key states are produced without any cross-token interaction, the alignment of different states is caused purely by the token's global importance or meaning imparted via pretraining. The \bos{} token has no semantic meaning in the context of prose tokens, so its key state is not aligned with key states of meaningful prose tokens. Also, delimiter tokens, often considered secondary attention sinks (cf.~\Cref{sub:fixed_bos}), have the most aligned key states to the key state of the \bos{} token, and are also the tokens with the least semantic meaning in the prose context. Thus, we identify that, at least in this restricted example, query state and key state alignment depends heavily on the contextual semantics of the token.

 \begin{figure}[h]
     \centering
     
     \begin{subfigure}[t]{0.3\textwidth}
        \centering
        \caption{Value-state drains at Layer 0 Head 31 of Llama 3.1-8B-Base.}\label{fig:llama_31_value_states}\includegraphics[width=0.9\textwidth]{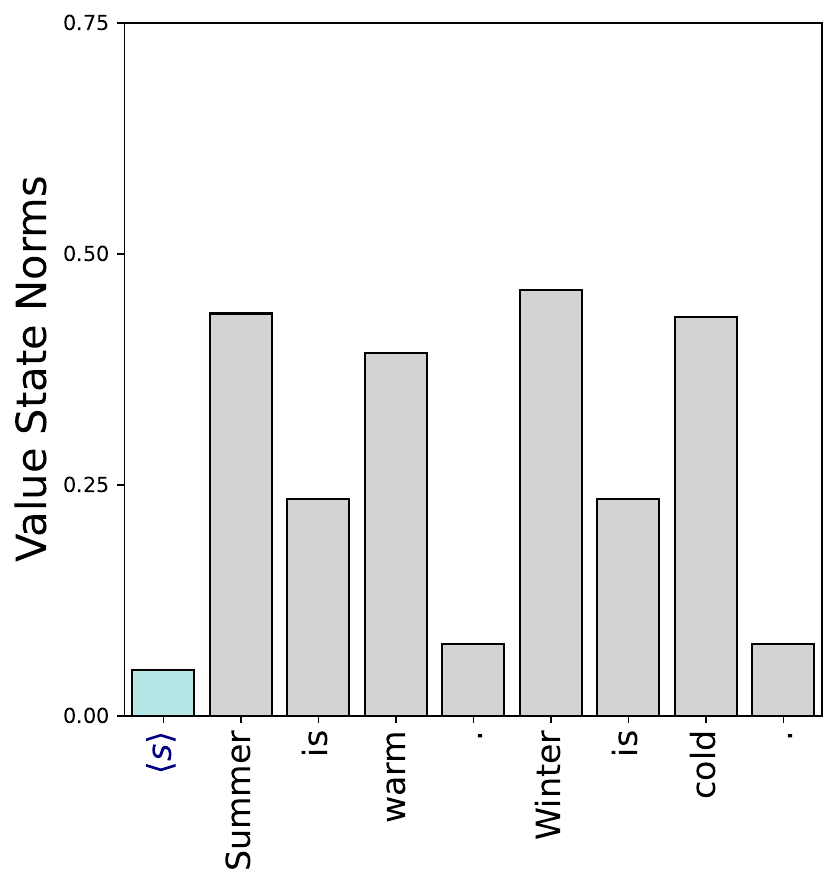}
    \end{subfigure}
    \begin{subfigure}[t]{0.6\textwidth}
    \centering
        \caption{\centering Ablation study on the cause of the residual state peak in Llama 3.1-8B-Base.}\label{fig:llama_31_norms_ablation}\includegraphics[width=0.5\textwidth]{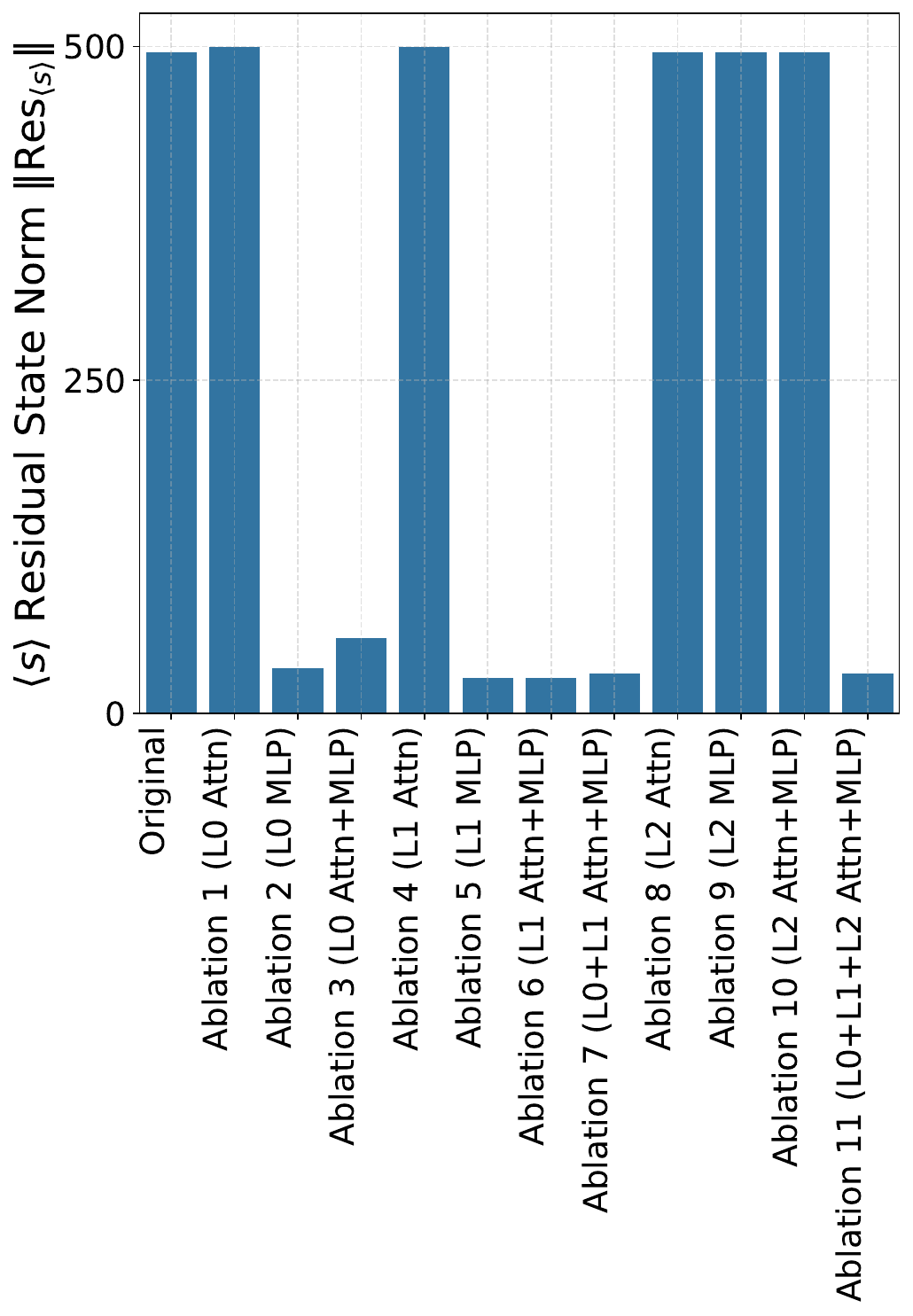}
    \end{subfigure}
    
     \caption{\small \textit{Left (a)}: Value-state drains at Layer 0 Head 31 of Llama 3.1-8B-Base. We observe that the value state associated with \bos{} is already much smaller than every other semantically meaningful token, and still smaller than the delimiter tokens in the same sentence. \textit{Right (b)}: Ablation study on the cause of the residual state peak in Llama 3.1-8B-Base. We perform a series of ablations to understand which components of the network promote the residual state peaks. We find that ablating either the zeroth or first layer's MLP is sufficient to remove the residual state peak phenomenon, while no other layer-level ablation can do it.}
     \label{fig:llama_31_value_states and norms}
 \end{figure}

\paragraph{Value-state drains.} The value states of the \bos{} token at Layer \(0\) Head 31 are already near zero, as demonstrated in \Cref{fig:llama_31_value_states}. While the delimiter tokens, which are less semantically meaningful in the prose context, have smaller value states than the rest, they are not as small as the value state of the \bos{} token which is guaranteed to not have any semantics.


\paragraph{Residual state peaks.} Residual state peaks are caused by the first two layers' MLPs. In particular, we perform several ablations, comparing between the residual state norms in a later layer (\(24\)) of an un-edited forward pass versus forward passes where we force the output of either multiple layers, a single layer, an attention block, or an MLP to be zero (and hence remove its contribution from the residual stream). As shown in Figure~\ref{fig:llama_31_norms_ablation}, ablating \textit{either} Layer 0's or Layer 1's MLP is sufficient to remove the residual state peak. In particular, the second-largest token at Layer 24 in \textit{each} ablation (including the original setup) has norm between \(29\) and \(38\), so the interventions ensure that all tokens have similar size.

\clearpage
\section{Extreme-Token Phenomena Over Many Samples}\label{sec:many_samples}
In this section we show that the extreme-token phenomena, and our predictions from the BB model, exhibit in prompts other than ``Summer is warm. Winter is cold.'' To this end, we use 128 samples from the Wikipedia dataset, each truncated to 8 tokens. \Cref{fig:extreme_tokens_llama_31_many_samples} provides aggregate statistics of extreme-token phenomena in Llama 3.1-8B, which are similar to the fine-grained statistics over a single prompt from \Cref{figure:extreme-token}. \Cref{fig:extreme_tokens_olmo_many_samples} provides aggregate statistics of the development of extreme-token phenomena over the training dynamics of OLMo, which are similar to the fine-grained statistics over a single prompt from \Cref{fig:olmo_predictions_phase0} and \Cref{fig:olmo_predictions_phase1}.

\begin{figure}[h]
    \centering
    \begin{subfigure}{0.31\textwidth}
        \centering 
        \caption{Attention weights (L24).}
        \includegraphics[width=\textwidth]{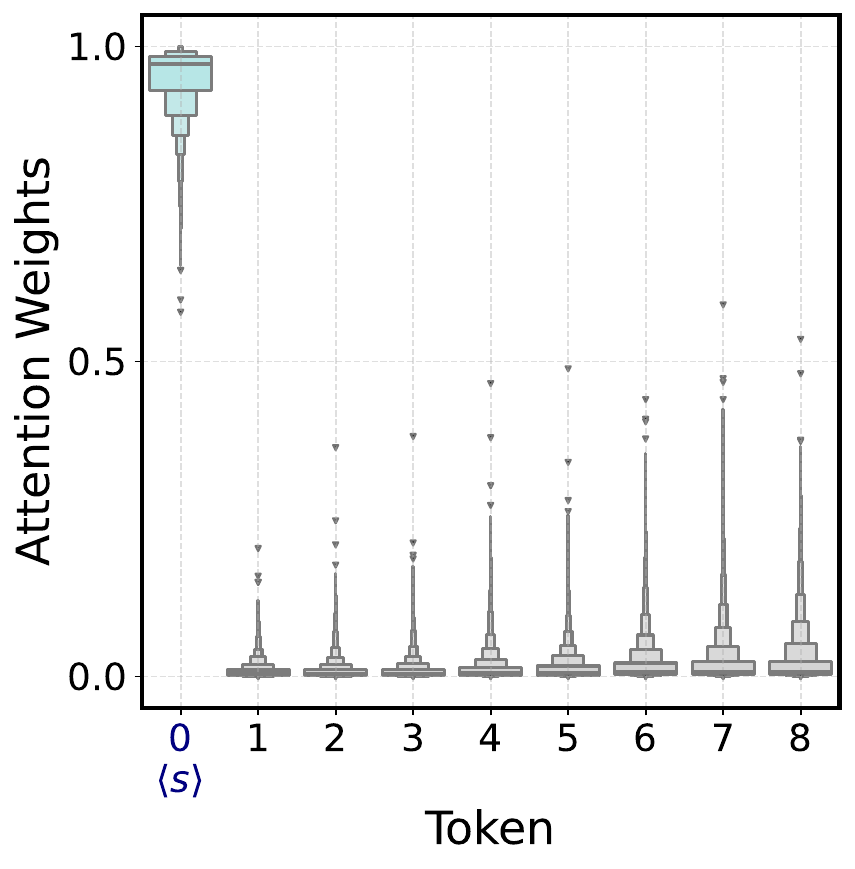}
    \end{subfigure}
    \hfill
    \begin{subfigure}{0.31\textwidth}
    \centering 
    \caption{Value state norms.}
    \includegraphics[width=\textwidth]{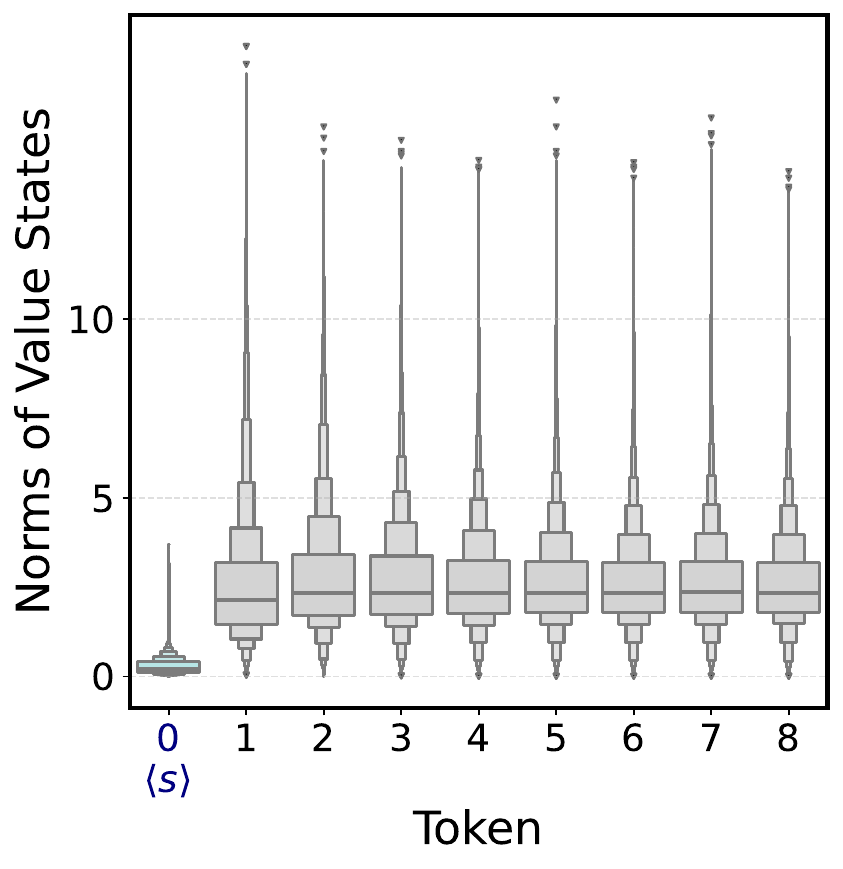}
    \end{subfigure}
        \hfill
    \begin{subfigure}{0.32\textwidth}
        \caption{Residual norms.}
        \includegraphics[width=\textwidth]{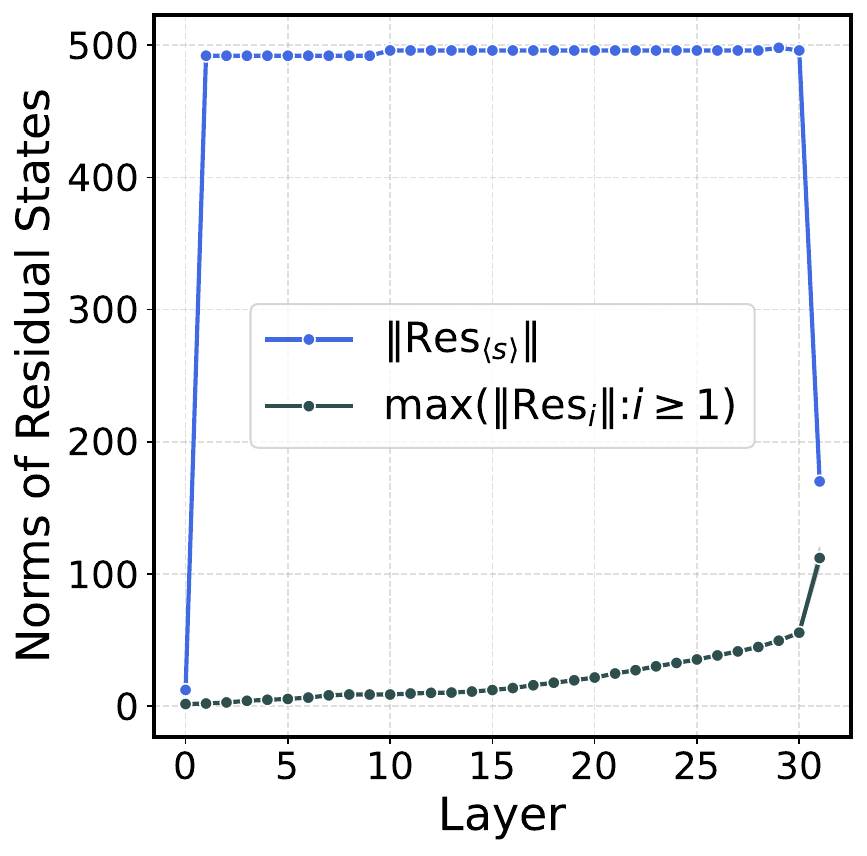}
    \end{subfigure}
    
    \caption{\textbf{Extreme token phenomena over many samples in Llama 3.1-8B-Base.} \textit{Left (a):} Let \(A\) be the attention weight tensor, of shape \((\text{batch size=128, \# heads=32, \# tokens=8, \# tokens=8})\) at Layer 24 of Llama 3.1-8B-Base. We calculate the tensor \(\bar{A}\), of shape \((\text{batch size=128, \# heads=32, \# tokens=8})\), which measures the average attention mass on the key tokens, by the following calculation: \(\bar{A}_{bhj} \doteq \frac{1}{n-j}\sum_{i = j}^{n}A_{bhij}\). We expect, for an attention sink head \(h\) on sample \(b\), that \(\bar{A}_{bh0}\) is large, and \(\bar{A}_{bhj}\) is small for all \(j \geq 1\). We indeed see this by plotting the distribution of \(\bar{A}_{:, :, j}\) for each \(j\), which shows that almost all attention mass is concentrated on the \bos token with high probability, showing the same thing as the individual attention head analysis in \Cref{figure:extreme-token} (a). \textit{Middle (b), Right (c):} We do the same computations as \Cref{figure:extreme-token} (b) and (c), averaged over the \(128\) samples.}
    \label{fig:extreme_tokens_llama_31_many_samples}
\end{figure}

\begin{figure}
    \centering
    \begin{subfigure}{0.45\textwidth}
    \centering 
    \caption{Attention weights (L24).}
    \includegraphics[width=0.7\textwidth]{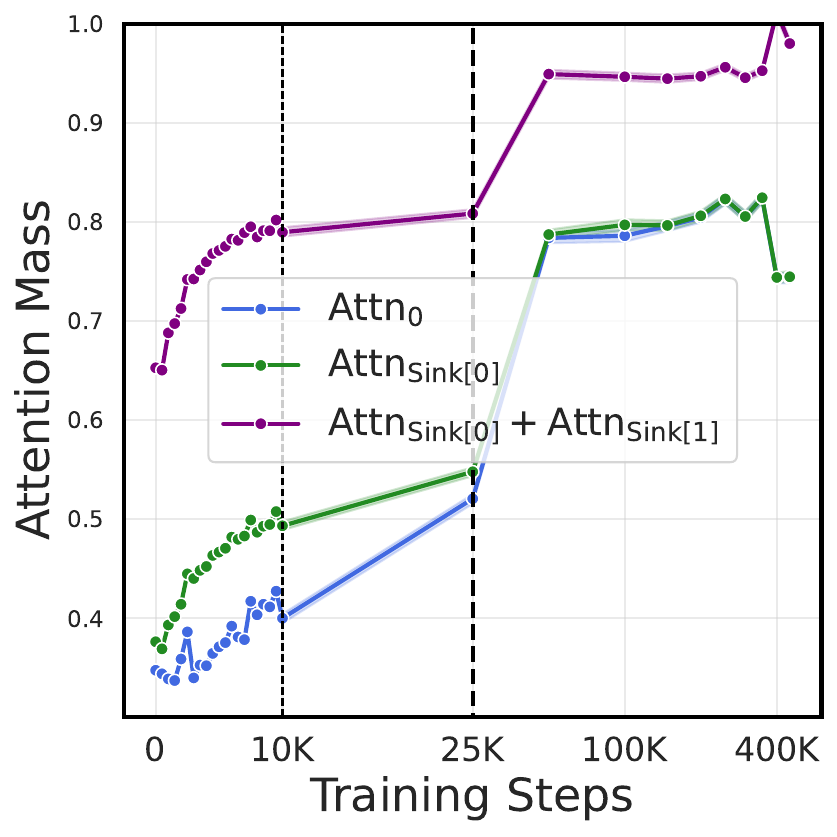}
    \end{subfigure}
    \begin{subfigure}{0.45\textwidth}
    \centering 
    \caption{Attention logits (L24).}
    \includegraphics[width=0.7\textwidth]{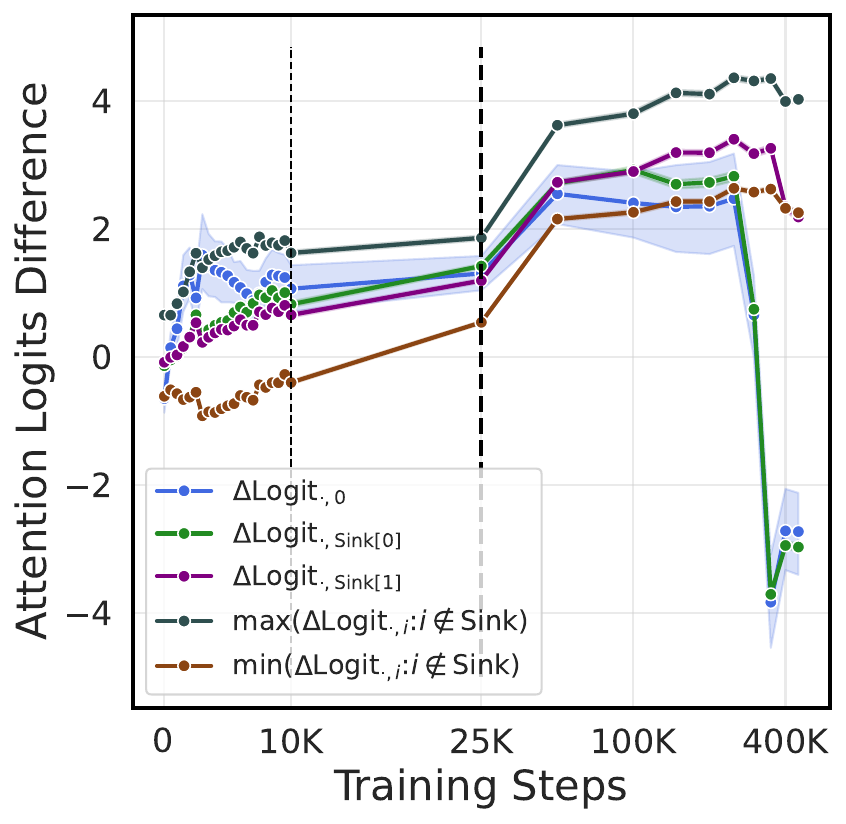}
    \end{subfigure}
    
    \begin{subfigure}{0.45\textwidth}
    \centering 
    \caption{Value state norms (L24).}
    \includegraphics[width=0.7\textwidth]{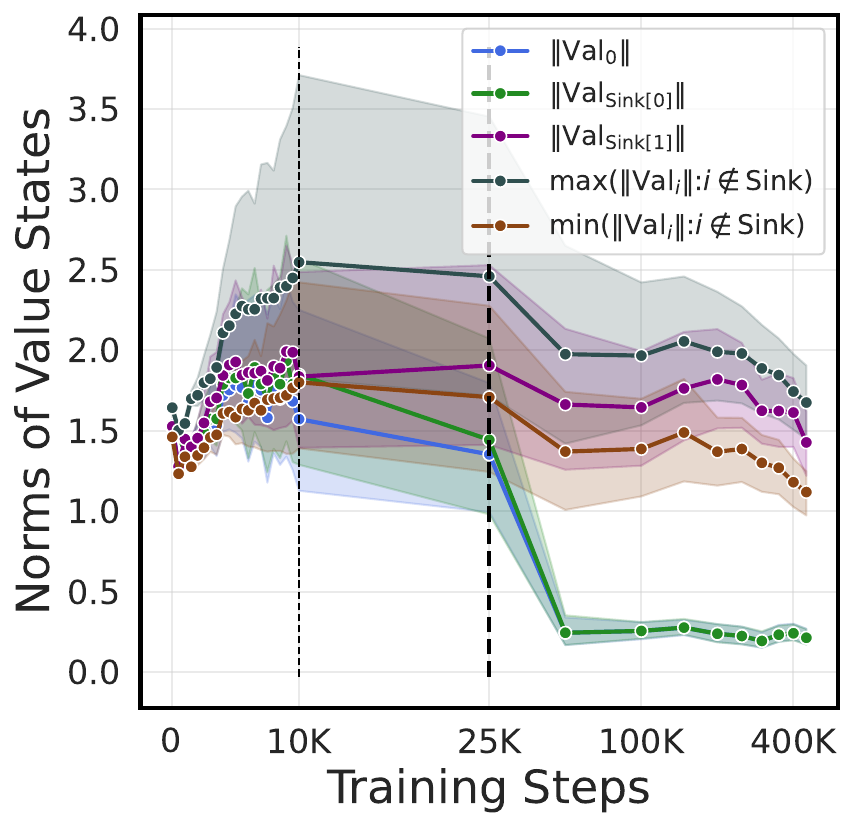}
    \end{subfigure}
    \begin{subfigure}{0.45\textwidth}
    \centering 
    \caption{Residual norms (L24).}
    \includegraphics[width=0.7\textwidth]{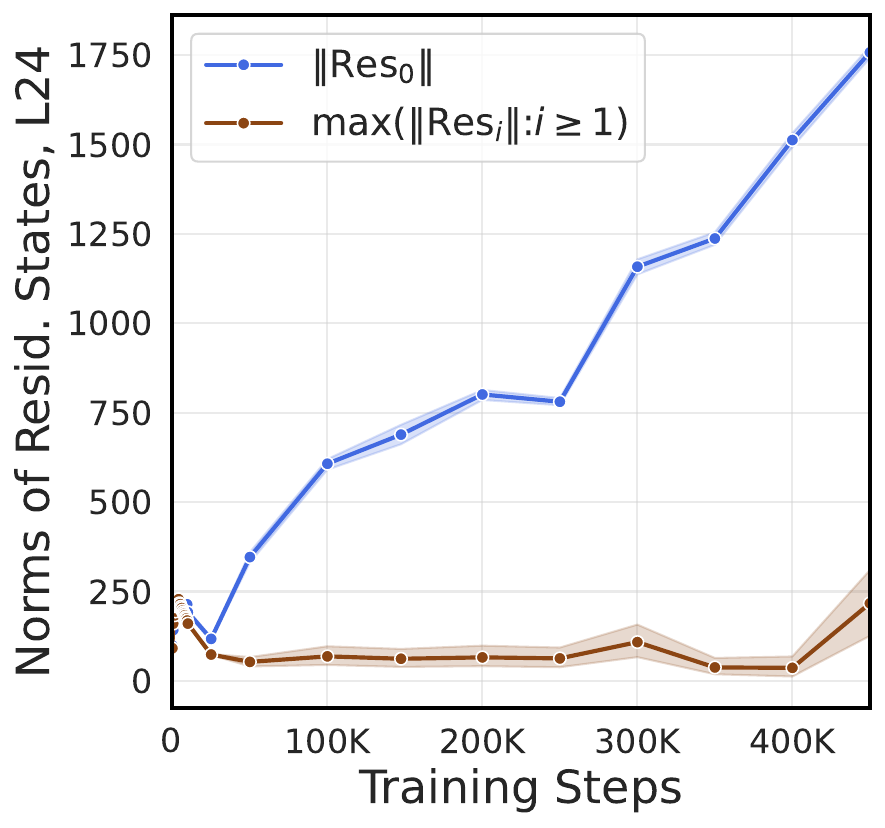}
    \end{subfigure}

    \caption{\small \textbf{Dynamics of extreme-token phenomena in layer 24 over many samples in the training trajectory of OLMo-7B.} For this experiment, as in \Cref{sub:olmo_dynamics}, for each sample and attention head we designate two attention sink tokens as the two tokens with the largest average attention mass \(\bar{A}_{bhj}\) (see \Cref{fig:extreme_tokens_llama_31_many_samples} for definition). We then study the dynamics of sink tokens versus non-sink tokens. In these experiments we observe that token \(0\) is (almost) always a sink token, which we discuss further in \Cref{sub:fixed_bos}. \textit{Top left (a):} The average attention scores \(\bar{A}_{bhj}\) for \(j\) as a sink token versus non-sink tokens. We observe that attention sinks form in nearly all heads and samples: the attention mass on top tokens nearly always sums to \(1\), and moreover the sinks develop relatively early in training. \textit{Top right (b):} We observe that the normalized attention logits of non-sink tokens initially increase until the formation of an attention sink, and then approximately converge to a stable phase with similar logits on token \(0\). \textit{Bottom left (c):} We observe that the value states of all tokens except the first sink token (token \(0\)) rapidly converges to steady state, while the first sink token has a much lower value state norm than all other tokens. \textit{Bottom right (d)}: We observe that the norm of the residual state of token \(0\) increases linearly during pretraining, while all other tokens' residual states do not. Our results mirror and confirm the single-sample detailed analysis conducted in \Cref{sub:olmo_dynamics}.}
    \label{fig:extreme_tokens_olmo_many_samples}
\end{figure}

\clearpage
\section{Assorted Caveats}\label{sec:caveats}

\subsection{Multiple attention sinks vs. one attention sink}\label{sub:multiple_sinks_discussion}

As we have seen, attention heads in the BB task (\Cref{sec:bb_task}), Llama 2-7B-Base (\Cref{sub:active_dormant}), and OLMo (\Cref{sub:olmo_dynamics}) exhibit multiple sink tokens. That is, when heads in these models are dormant, they tend to have two sink tokens. For the LLMs in this group, at least on prose data, the \bos~token as well as the first delimiter token (e.g., representing \texttt{.} or \texttt{;}) are sink tokens. Meanwhile, Llama-3.1-8B-Base (\Cref{sec:llm}) only ever has one attention sink on prose data, and the \bos{} token is always the sink token. Here, we offer a possible explanation of this phenomenon. For the BB task, multiple sink tokens are necessary to solve the task. For LLMs, we believe this distinction may be explained by the relative proportion of coding data, in which delimiters have a greater semantic meaning than prose, within the training set. For instance, OLMo was trained on DOLMA \citep{soldaini2024dolma}, which has around 411B coding tokens. Meanwhile, Llama 2 used at most (2T \(\times\) 0.08 =) 0.16T coding tokens. Finally, Llama 3.1 used around (15.6T \(\times\) 0.17 =) 2.6T coding tokens \citep{dubey2024llama}. On top of the raw count being larger, coding tokens are a larger proportion of the whole pretrtraining dataset for Llama 3.1 compared to other model families. Thus, during training, the presence of delimiters would not be considered unhelpful towards next-token prediction, since such delimiters carry plenty of semantics in a wide variety of cases. Our earlier hypothesis in \Cref{sub:active_dormant} proposes that only tokens which lack semantics in almost all cases are made to be sink tokens. This could be a reason for the distinction.

\subsection{The role of a fixed \bos~ token in the Active-Dormant mechanism}\label{sub:fixed_bos}

Some models, such as OLMo, are not trained with a \bos{} token. Despite this, the first token of the input still frequently develops into a sink token. We can study the effect of positional encoding of the tokens on the attention sink phenomenon by shuffling the tokens before inputting them into the transformer, and observing how and why attention sinks form. If we do this with the phrase ``Summer is warm. Winter is cold.'' with OLMo, we observe that at Layer 24, there are many attention sink heads where the first token and first delimiter token share attention mass, even if the sentence is jumbled up and makes no grammatical sense. This points towards the observation that without a \bos{} token, the attention sink formation uses both positional data and, to a greater degree, the semantic data of each token. We leave studying this effect in greater detail to future work.

\begin{figure}[h]
    \centering
    \includegraphics[width=0.73\textwidth]{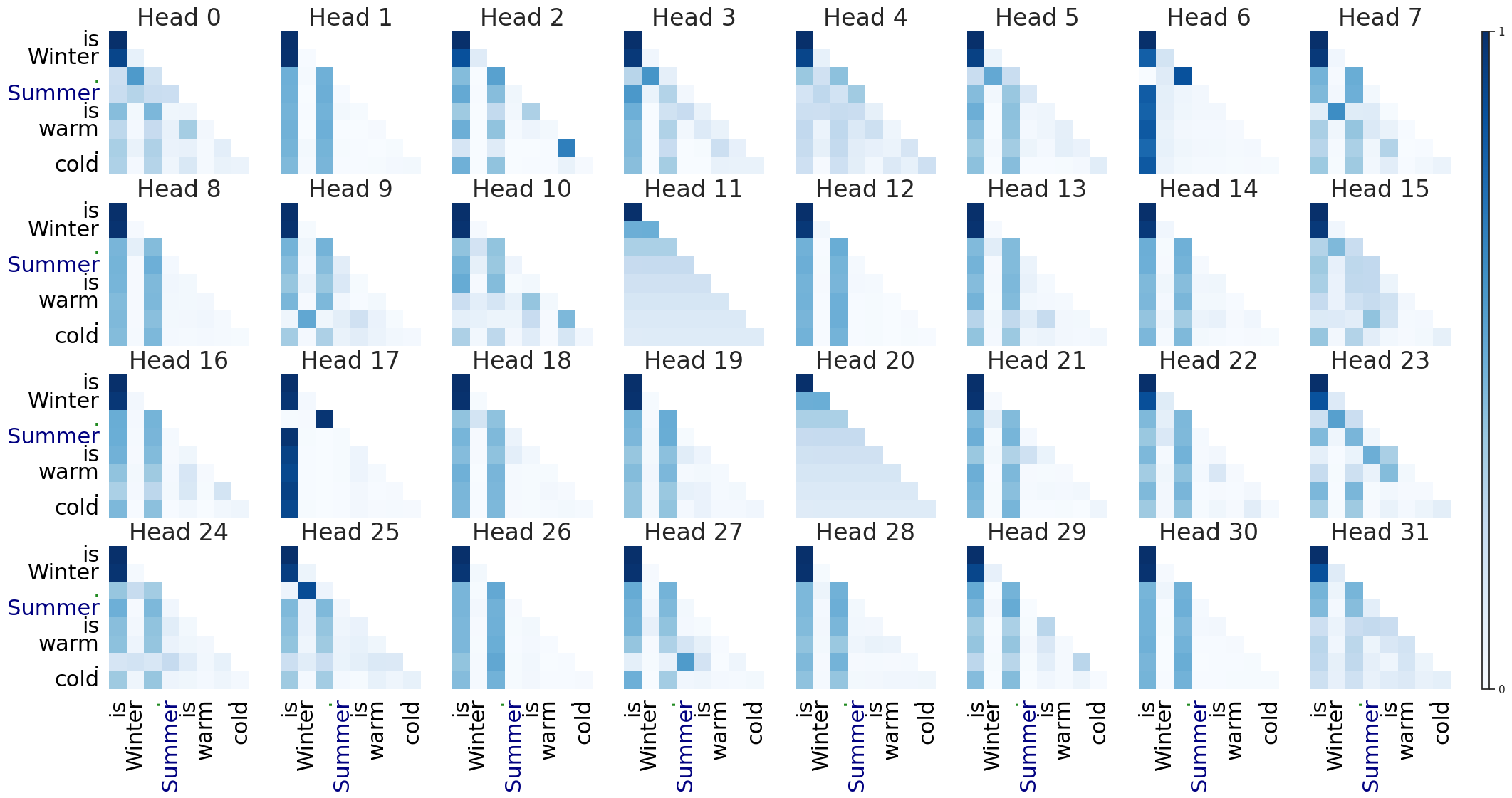}
    \caption{\small \textbf{Attention sinks with shuffled input in Layer 24 of OLMo.} In order to understand the impact of positional encodings when there is no \bos{} token, we shuffle the input of the test string ``Summer is warm. Winter is cold.'' in OLMo. We observe that there is still an attention sink on token \(0\), despite it being a random token that does not usually start sentences or phrases (since it is uncapitalized). This shows that the positional embedding, say via RoPE, has a large impact on the formation of attention sinks --- when the semantics of each token have switched positions, the attention sink still forms on the zeroth token.}
    \label{fig:bos_shuffle}
\end{figure}


\end{document}